%% file: main.tex
\documentclass[runningheads]{llncs}
\usepackage{graphicx}
\usepackage{tikz}
\usepackage{comment}
\usepackage{amsmath,amssymb} % define this before the line numbering.
\usepackage{color}

\usepackage{epsfig}
\usepackage{graphicx}
\usepackage[utf8]{inputenc} % allow utf-8 input
\usepackage[T1]{fontenc}    % use 8-bit T1 fonts
\usepackage{hyperref}       % hyperlinks
\usepackage{url}            % simple URL typesetting
\usepackage{booktabs}       % professional-quality tables
\usepackage{amsfonts}       % blackboard math symbols
\usepackage{nicefrac}       % compact symbols for 1/2, etc.
\usepackage{microtype}      % microtypography
\usepackage{placeins}       % provides FloatBarrier

\usepackage{subcaption}
\newcommand{\citep}[1]{\cite{#1}}
\newcommand{\drop}[1]{}

\usepackage{amstext} % Full math equation support
\usepackage{physics}         % \abs and \norm
\usepackage{siunitx}         % like units, but better

\DeclareMathAlphabet{\mathpzc}{OT1}{pzc}{m}{it}
\newcommand{\opn}[1]{\operatorname{#1}}
\newcommand{\pd}[1]{\frac{\partial}{\partial #1}}       % first order partial derivative operator
\newcommand{\mtx}[1]{\left[ \begin{matrix} #1 \end{matrix} \right]} %matrix
\newcommand{\VEC}[1]{\mathbf{#1}}                                   %vector symbol typeface
\newcommand{\NSYS}[1]{\mathbb{#1}}                                  %number system typeface (R, C, Z, N)
\newcommand{\R}{\NSYS{R}}                                           %number system typeface (R, C, Z, N)
\newcommand{\mbf}[1]{\mathbf{#1}}
\newcommand{\mbi}[1]{\mathbf{#1}}
\newcommand{\boolop}[1]{\textit{#1}}

\begin{document}
% \renewcommand\thelinenumber{\color[rgb]{0.2,0.5,0.8}\normalfont\sffamily\scriptsize\arabic{linenumber}\color[rgb]{0,0,0}}
% \renewcommand\makeLineNumber {\hss\thelinenumber\ \hspace{6mm} \rlap{\hskip\textwidth\ \hspace{6.5mm}\thelinenumber}}
% \linenumbers
\pagestyle{headings}
\mainmatter
\def\ECCVSubNumber{7453}  % Insert your submission number here

\title{LogAvgExp Provides a Principled and Performant Global Pooling Operator} % Replace with your title

% INITIAL SUBMISSION 
\begin{comment}
\titlerunning{ECCV-20 submission ID \ECCVSubNumber} 
\authorrunning{ECCV-20 submission ID \ECCVSubNumber} 
\author{Anonymous ECCV submission}
\institute{Paper ID \ECCVSubNumber}
\end{comment}
%******************

% CAMERA READY SUBMISSION
%\begin{comment}
\titlerunning{LogAvgExp Provides a Principled and Performant Global Pooling Operator}
% If the paper title is too long for the running head, you can set
% an abbreviated paper title here
%
\author{
    Scott C.~Lowe\inst{1,2,*}
    \and
    Thomas Trappenberg\inst{1}
    \and
    Sageev Oore\inst{1,2}
}
\authorrunning{S. Lowe et al.}
% First names are abbreviated in the running head.
% If there are more than two authors, 'et al.' is used.
%
\institute{
    Faculty of Computer Science, Dalhousie University Halifax, Nova Scotia, Canada
    \and
    Vector Institute for Artificial Intelligence, Toronto, Ontario, Canada\\
    $^*$Correspondence: scottclowe@gmail.com
}
%\end{comment}
%******************
\maketitle

% \begin{abstract}
% The abstract should summarize the contents of the paper. LNCS guidelines indicate it should be at least 70 and at most 150 words. It should be set in 9-point font size and should be inset 1.0~cm from the right and left margins. \dots \keywords{We would like to encourage you to list your keywords within the abstract section}
% \end{abstract}

% Feedback from R1 @ ECCV
% (1) Generalized \alpha-pooling 
%   Generalized orderless pooling performs implicit salient matching, ICCV17
% (2) Second-order Pooling 
%   Bilinear CNNs for Fine-grained Visual Recognition, ICCV15
%   Matrix Backpropagation for Deep Networks with Structured Layers, ICCV15
%   Is Second-order Information Helpful for Large-scale Visual Recognition? ICCV17
%   Statistically-motivated Second-order Pooling, ECCV18
%   Deep CNNs Meet Global Covariance Pooling: Better Representation and Generalization, TPAMI 2020
% (3) Probability Distribution Pooling
%   G2DeNet: Global Gaussian Distribution Embedding Network and Its Application to Visual Recognition, CVPR17
%   Deep Global Generalized Gaussian Networks, CVPR19
% (4) Higher-order Pooling
%   Kernel Pooling for Convolutional Neural Networks, CVPR17
%   Therefore, the authors would better make a comprehensive investigation on related works, while making discussions and compassions with these related works.

\begin{abstract}
We seek to improve the pooling operation in neural networks, by applying a more theoretically justified operator.
We demonstrate that LogSumExp provides a natural \boolop{or} operator for logits.
When one corrects for the number of elements inside the pooling operator, this becomes $\operatorname{LogAvgExp} := \log(\operatorname{mean}(\exp(\VEC{x})))$.
By introducing a single temperature parameter, LogAvgExp smoothly transitions from the max of its operands to the mean (found at the limiting cases $t \to 0^+$ and $t \to +\infty$).
We experimentally tested LogAvgExp, both with and without a learnable temperature parameter, in a variety of deep neural network architectures for computer vision.
%We found that this computationally and parameter lightweight change yielded a significant improvement in the performance of the networks.
\end{abstract}

%We seek to improve the pooling operation in neural networks, by applying a more theoretically justified operator.
%We demonstrate that LogSumExp provides a natural "or" operator for logits.
%When one corrects for the number of elements inside the pooling operator, this becomes LogAvgExp := log(mean(exp(x))).
%By introducing a single temperature parameter, LogAvgExp smoothly transitions from the max of its operands to the mean (found at the limiting cases t->0 and t->+infinity.
%We experimentally tested LogAvgExp, both with and without a learnable temperature parameter, in a variety of deep neural network architectures for computer vision.

\section{Introduction}
\label{sec:intro}

Over the past decade, computer vision has been dominated by neural network based methods.
These permit features to be learnt automatically from data with backpropagation and stochastic gradient descent.
Such learnt features out-perform hand-crafted features from preceding ``classical'' computer vision methodologies.

The neural network architectures dominant in computer vision are convolutional neural networks (CNNs), with weight sharing between the kernels applied at each point in space.
This weight sharing exploits the fact that objects, their features, and subfeatures are invariant to translation within the visual field.

The original convolutional neural networks \cite{cnn} contained interleaved convolutional linear layers, an activation function (either tanh or ReLU), and max-pooling layers \cite{alexnet,vgg}.
In order to reduce the dimensionality of the latent space, the representations were downsampled at various points within the network using max-pooling layers with a stride of 2 (typically with 2x2 kernels).
This has the effect of downsampling each spatial dimension by a factor of two --- reducing the representation of a 2d grid to a quarter as many datapoints.
However, despite halving the spatial dimensions multiple times, spatial dimensions were not fully collapsed, thus preserving information about where in the visual field high-level features occurred.
After the final convolutional layer, networks featured a fully-connected layer (FC), in which the majority of the network's overall parameters resided, even though it corresponded to a small fraction of the computational operations.

In contrast to this, modern architectures perform a global pooling step after the final convolutional layer \cite{resnext,densenet,squeeze-and-excitation,pyramidnet,mobilenetv1,mobilenetv2}.
With this operation, features are averaged across all of the spatial dimensions.
Here, global pooling serves two key purposes.
Firstly, collapsing down the spatial components reduces the latent space from 8x8 (for instance) to 1x1, reducing the number of inputs to the subsequent fully-connected layer (and hence the number of parameters) by a factor of 64.
In particular, with this architecture we can do as much of the computation as possible with the much more parameter-efficient convolutional layers instead of fully connected layers.
In some cases, there is no fully connected layer at all, with the final convolutional layer containing one channel per class, and a global pooling layer on top of this sufficient to output logits indicating the overall image label. % citation please! I can't find one.
Secondly, as we are collapsing down all the spatial domain into a vector, the same network architecture can be used for inputs of differing sizes (the pooling operation is adaptive and the kernel expands to fill the space).
This methodology was popularized by the Inception network of \cite{inception}.

While a global pooling layer immediately before the final output of the network is prevalent among state of the art networks, other pooling layers (in the middle of the network) are finding diminishing application \cite{resnet,resnext,squeeze-and-excitation,pyramidnet,shakedrop,autoaugment}.
Instead, mid-network spatial downsampling is accomplished by doubling the stride.
This results in a subsampling of the spatial field --- the convolutional filter is applied to only 25\% of the possible spatial tilings.
Intuitively, one might think that subsampling space in this way would have a large negative impact on performance, since it is equivalent to throwing away 75\% of the outputs of the same convolutional layer with a stride of 1.
However, the activations of nodes corresponding to adjacent points in space are highly correlated.
Consequently, this step does not result in much loss of information, nor does the removal of max-pooling layers appear to hinder the network's ability to generalize.

We anticipate there is little utility to be gained in improving the mid-network max-pooling operation for the aforementioned reasons.
But the global pooling, increasingly prevalent in modern architectures, is a potential point for further improvement.
When we perform global pooling, the kernel is larger and hence its contents are more diverse.
How we pool over space matters both for the forward step (we want to integrate spatial information optimally), and for the backward step (we want to assign credit appropriately to update operands efficiently).

Previous work has looked at improving the pooling function, but not with a focus on global pooling.
In particular, \cite{generalizing-pooling} considered interpolating between max pooling and average pooling with a trainable parameter, $\alpha$, such that $\opn{MixedPool}(\VEC{z}) = \alpha \max(\VEC{z}) + (1-\alpha) \opn{avg}(\VEC{z})$.
They also introduced a Gated pooling, in which the mixing parameter is given by a linear mapping of the pooling kernel passed through a sigmoid, $\alpha = \sigma(\VEC{w}^T\VEC{z})$. 

Other methods that have been proposed to generalize max and average pooling include \cite{kolesnikov2016seed,radenovic2018fine,fractional-maxpool}.
\cite{saeedan2018detail} proposed a detail-preserving pooling (DPP) function based on inverse bilateral filters, with learnable parameters that control the amount of detail that is preserved.
Consequently DPP is very capable technique for downscaling space, but is not suitable for global pooling as it collapses space entirely and all spatial information is lost.

In the present work, we derive a global pooling operator from first principles and converge upon a LogSumExp-based pooling operation.
LogSumExp-pooling has been introduced previously for pixel-level semantic segmentation \cite{Pinheiro2015}, and has seen a small number of related applications \cite{Wang2017,Wang2018}.
LogAvgExp pooling appears to have been introduced for multiple-instance learning by \cite{ramon2000multi}, and then analysed as part of a general set of pooling functions by \cite{boureau2010theoretical}, where it appeared to be inferior to other methods.
\cite{cohen2016deep} use LAE pooling as one component of an interesting architecture called \textit{Simnets} that generalizes convolutional neural networks.
This architecture is more expressive with fewer parameters, but computationally expensive, hence it is well suited for tasks with small datasets.

In this work, we illustrate the theoretical underpinnings of LSE and LAE for global pooling, demonstrate its capabilities on CNNs, and extend LAE through the addition of a trainable temperature parameter.
In so doing, we hope to vitalize the usage of this more powerful pooling operator.

\section{Theoretical motivation}

In this section, we derive a more principled pooling operation.
In order to do so, we must first make some well-principled assumptions about what it is that neurons encode (the inputs to the pooling operation) and what we aim to accomplish as we perform the pooling (the output).

Let us first consider the case where our neural network terminates with Conv-Pool-Softargmax.%
\footnote{Our arguments also hold if the network is not convolutional.}
In this case, there is no terminal fully connected layer.
The softargmax (often referred to as softmax),
\begin{align}
\opn{softargmax}(\VEC{z})_i &= \frac{\exp(z_i)}{\sum_{\forall j} \exp(z_j)} \label{eq:softargmax}
,\end{align}
rescales a vector of unnormalized scores or logits, $\VEC{z}$, into probabilities, $\VEC{p}$; as such it is a soft approximation of argmax.
We will assume the Pool step operates independently for each channel with no inter-channel interactions.
Consequently, the final convolution must output the same number of features/channels as the number of target classes in the dataset; each of its output channels will correspond to precisely one class label.

The activations passed from the Conv layer to Pool each correspond to a large subregion of space, their receptive fields overlapping one-another.
Across the spatial domain, each activation indicates a score for the presence of one of the classes for the corresponding region of space.
We could perform the softargmax before the pooling and extract the probability of each class being at each corresponding subregion of space.
The goal of this network is to output the label of the entire image, not a subregion of it, and so we would need to integrate these probabilities across space to complete the objective.
To complete the task, we don't care \textit{where} in the image an object is located.
The image is one of a cat if there is a cat situated in the top-left region%
\footnote{Assuming the cat dominates the image.}, or the top-right region, etc, of the image.
As any one of the options is sufficient, we need an \boolop{or}-like operator.

Tangentially, consider a classification problem with 5 classes \{cat, dog, car, truck, tree\}, denoted $u_i, \, i \in \{1, \ldots, 5\}$ respectively, and a model $\mathpzc{M}$ whose final layer provides logits $z_i$ corresponding to each class.
Applying softargmax to $\VEC{z}$ provides us with estimated probabilities for each class, $\VEC{p} = \opn{softargmax}(\VEC{z})$.
Let us now suppose we wish to group our outputs together into coarser labels, \{animal, vehicle, plant\} denoted $u'_j, \, j \in \{1, 2, 3\}$ (with the intuitive hierarchical class composition $u'_1 = \{u_1, u_2\}, u'_2 = \{u_3, u_4\}, u'_3 = u_5$).
We can estimate the probabilities of each superclass $u'_j$ by summing the probabilities of each of its members, for instance $p'_1 = p_1 + p_2$.

Alternatively, we could use the fine-grained logits $z_i$ to create coarse-grained logits $z'_j$, such that they will yield these very same $p'_j$ values when a softargmax is applied over all $z'_j$.
From \autoref{eq:softargmax}, we can see that $p'_1 = p_1 + p_2 = (\exp(z_1) + \exp(z_2)) / \sum_i \exp(z_i)$.
If we choose to define coarse grain logits of $z'_1 = \log(\exp(z_1) + \exp(z_2))$, and so on for other $z'_j$ terms, applying softmax to these yields the target probabilities, $\opn{softargmax}(z')_j = p'_j$.
The operation we have discovered is none other than LogSumExp (LSE), defined as
\begin{equation}
\operatorname{LogSumExp}(\VEC{z}) = \opn{LSE}(\VEC{z}) = \log\left( \sum_{i=1}^n \exp(z_i) \right)
.\end{equation}
From this we can conclude that LSE is an ``or-operator'' for unnormalized logits.
%\footnote{Provided that normalization is consistent.}
And, in the same vein, we can create an unnormalized logit indicating whether a class-level object occurs anywhere in space by pooling with the LSE operator.

Now let us consider the case where our global pooling layer does not operate on class-level features.
Provided we do not have an activation function between the linear layer and the pooling, we can consider these input activations to be the \textit{logits} of features occurring within space.
Such an interpretation has long standing, as early neural networks used sigmoid activations to convert intermediate logits into the probability of presence of a feature.
Moreover, modern architectures use batch normalization for their activations, the output of which is a $Z$-score indicating the \textit{significance} of the activation with respect to the average activation seen across the (history of) batch(es) and not the \textit{absolute} preponderance of a feature.

Assuming that the input to our pooling layer is a logit (having not passed through an activation function), we consider how to integrate these feature logits across space.
An intuitive option is to convert the spatial logits --- the log-odds that a feature is locally present --- into a logit indicating the log-odds that a feature is globally present.
Again, this is accomplished with LSE pooling.

Note though, that the output of LSE is guaranteed to be larger than any of the logits within its pooling window.
Intuitively, this corresponds to the fact that increasing the number of options can only increase the probability that one of the options is true.
To correct against this, we can instead use
\begin{align}
\opn{LogAvgExp}(\VEC{z})
    = \opn{LAE}(\VEC{z})
    :=& \log\left( \frac{1}{n} \sum_{i=1}^n \exp(z_{i}) \right) \\
     =& \opn{LSE}(\VEC{z}) - \log(n), \label{eq:lae}
\end{align}
where $n$ is the number of elements over which we are pooling. By subtracting $\log(n)$, we introduce a bias that corrects for the size of the pooling kernel.
This prevents the output from growing without bound as our global pooling kernel changes in size (recall that we want the pooling to be adaptive, and so return consistent outputs no matter the size of the kernel it is applied across).
When we perform spatial pooling across all the channels individually, using LAE instead of LSE does not change our interpretation of the procedure as an \boolop{or}-operation for logits, because the kernel, and hence $\log(n)$, is the same for all channels; these terms cancel out in the normalization step of softargmax.

Note that both LSE and LAE are soft approximations to the maximum operator, but LSE is bounded below by the max and LAE is bounded above.
For example, given a vector containing a repeated single value, $\VEC{a} = [a, a, \ldots, a]$, we find that $\opn{LAE}(\VEC{a}) = a$.
Meanwhile, $\opn{LSE}(\VEC{a}) = a + \log(\opn{len}(\VEC{a}))$.

Another advantage of LSE and LAE over max-pooling is that gradients flow back to more than one input.
For instance, if two operands are at or near the maximum value, the gradient of max-pooling will pass back to only one of them.
In contrast, the derivative of LSE is similar for both operands.
As exemplified in \autoref{fig:derivative-examples}, the derivatives of max and average pooling functions are indifferent to changes in the values within the pooling kernel, except for the (undesirable) discontinuity for max when the maximum value jumps from one element to another.

\begin{figure*}
  \centering
  \footnotesize

%\fbox{ %\rule{0pt}{2in} 
       % \rule{.9\linewidth}{0pt}
%%%%%%% MATRICES  COLUMN 1
  \begin{subfigure}[c]{.15\linewidth}
    \centering
    \hfill
  \end{subfigure}
%  \begin{subfigure}[c]{.14\linewidth}
%    \centering
%    $X = \mtx{-1 & 0 \\ 1 & 2}$
%  \end{subfigure}
  \begin{subfigure}[c]{.14\linewidth}
    \centering
    \scalebox{0.8}{$X = \mtx{-1 & 0 \\ 1.4 & 1.6}$}
  \end{subfigure}
  \begin{subfigure}[c]{.14\linewidth}
    \centering
    \scalebox{0.8}{$X = \mtx{-1 & 0 \\ 1.6 & 1.4}$}
  \end{subfigure}
  \begin{subfigure}[c]{.06\linewidth}
    \centering
    \hfill
  \end{subfigure}
  \begin{subfigure}[c]{.17\linewidth}
    \centering
    \hfill
  \end{subfigure}
%  \begin{subfigure}[c]{.14\linewidth}
%    \centering
%    \scalebox{0.8}{$X = \mtx{-1 & 0 \\ 1 & 2}$}
%  \end{subfigure}
  \begin{subfigure}[c]{.14\linewidth}
    \centering
    \scalebox{0.8}{$X = \mtx{-1 & 0 \\ 1.4 & 1.6}$}
  \end{subfigure}
  \begin{subfigure}[c]{.14\linewidth}
    \centering
    \scalebox{0.8}{$X = \mtx{-1 & 0 \\ 1.6 & 1.4}$}
  \end{subfigure}

  %\vspace*{.10cm}

%%%%%% MAX X  
  \begin{subfigure}[c]{.14\linewidth}
    \centering
    \scalebox{0.9}{$\opn{max}(X)$}
  \end{subfigure}
%  \begin{subfigure}[c]{.14\linewidth}
%    \centering\includegraphics[width=.95\linewidth]{figs/derivative_A_max_.eps}
%    \caption{\label{fig:max_A}$\opn{max}(X)=2$}
%  \end{subfigure}
  \begin{subfigure}[c]{.14\linewidth}
    \centering\includegraphics[width=.95\linewidth]{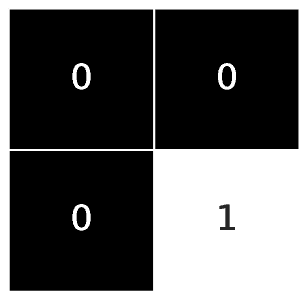}
    \caption{\label{fig:max_B}$1.6$}
  \end{subfigure}
  \begin{subfigure}[c]{.14\linewidth}
    \centering\includegraphics[width=.95\linewidth]{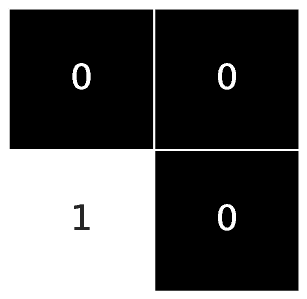}
    \caption{\label{fig:max_D}$1.6$}
  \end{subfigure}
  \begin{subfigure}[c]{.06\linewidth}
    \centering
    \hfill
  \end{subfigure}
%%%%%% LAE X t=0.5
  \begin{subfigure}[c]{.17\linewidth}
    \centering
    \scalebox{0.9}{$\opn{LAE}(X;t=\nicefrac{1}{2})$}
  \end{subfigure}
%  \begin{subfigure}[c]{.14\linewidth}
%    \centering\includegraphics[width=.95\linewidth]{figs/derivative_A_lae_t05_.eps}
%    \caption{\label{fig:lae_0.5_A}$\opn{LAE}=1.38$}
%  \end{subfigure}
  \begin{subfigure}[c]{.14\linewidth}
    \centering\includegraphics[width=.95\linewidth]{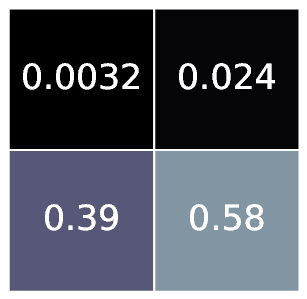}
    \caption{\label{fig:lae_0.5_B}$1.18$}
  \end{subfigure}
  \begin{subfigure}[c]{.14\linewidth}
    \centering\includegraphics[width=.95\linewidth]{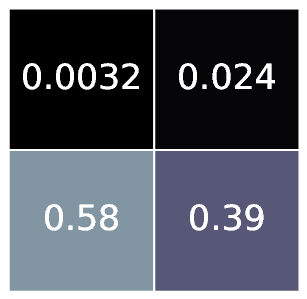}
    \caption{\label{fig:lae_0.5_D}$1.18$}
  \end{subfigure}

%%%%%% AVG X
\vspace*{-0.10cm}

  \begin{subfigure}[c]{.14\linewidth}
    \centering
    \scalebox{0.9}{$\opn{avg}(X)$}
  \end{subfigure}
%  \begin{subfigure}[c]{.14\linewidth}
%    \centering\includegraphics[width=.95\linewidth]{figs/derivative_A_avg_.eps}
%    \caption{\label{fig:avg_A}$\opn{avg}(X)=0.5$}
%  \end{subfigure}
  \begin{subfigure}[c]{.14\linewidth}
    \centering\includegraphics[width=.95\linewidth]{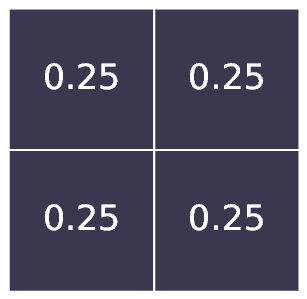}
    \caption{\label{fig:avg_B}$0.5$}
  \end{subfigure}
  \begin{subfigure}[c]{.14\linewidth}
    \centering\includegraphics[width=.95\linewidth]{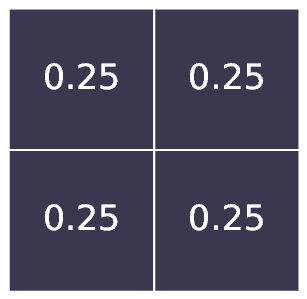}
    \caption{\label{fig:avg_D}$0.5$}
  \end{subfigure}
  \begin{subfigure}[c]{.06\linewidth}
    \centering
    \hfill
  \end{subfigure}
%%%%%%%%%%%%% LAE X t=1    COLUMN 2
 \begin{subfigure}[c]{.17\linewidth}
    \centering
    \scalebox{0.9}{$\opn{LAE}(X;t=1)$}
    %LAE $t=1$
  \end{subfigure}
%  \begin{subfigure}[c]{.14\linewidth}
%    \centering\includegraphics[width=.95\linewidth]{figs/derivative_A_lae_t1_.eps}
%    \caption{\label{fig:lae_A}$\opn{LAE}=1.05$}
%  \end{subfigure}
  \begin{subfigure}[c]{.14\linewidth}
    \centering\includegraphics[width=.95\linewidth]{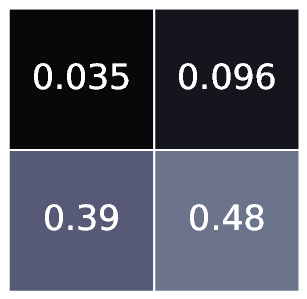}
    \caption{\label{fig:lae_B}$0.95$}
  \end{subfigure}
  \begin{subfigure}[c]{.14\linewidth}
    \centering\includegraphics[width=.95\linewidth]{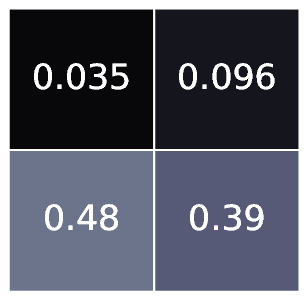}
    \caption{\label{fig:lae_D}$0.95$}
  \end{subfigure}
%%%%%% (MAX + AVG)/2 X
\vspace*{-0.10cm}

  \begin{subfigure}[c]{.13\linewidth}
    \centering
    %$\frac{\opn{max}(X) + \opn{avg}(X)}{2}$
    $\frac{\opn{max} + \opn{avg}}{2}$
  \end{subfigure}
%  \begin{subfigure}[c]{.14\linewidth}
%    \centering\includegraphics[width=.95\linewidth]{figs/derivative_A_mix_.eps}
%    \caption{\label{fig:mix_A}$\opn{pool}(X)=1.25$}
%  \end{subfigure}
  \begin{subfigure}[c]{.14\linewidth}
    \centering\includegraphics[width=.95\linewidth]{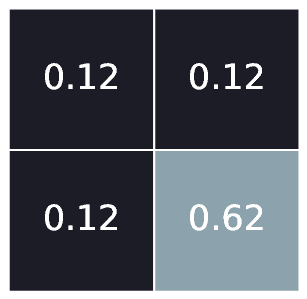}
    \caption{\label{fig:mix_B}$1.05$}
  \end{subfigure}
  \begin{subfigure}[c]{.14\linewidth}
    \centering\includegraphics[width=.95\linewidth]{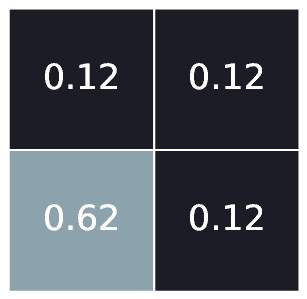}
    \caption{\label{fig:mix_D}$1.05$}
  \end{subfigure}
  \begin{subfigure}[c]{.06\linewidth}
    \centering
    \hfill
  \end{subfigure}
  %%%%%%%%%%%%%% LAE t=2 COLUMN 2
  \begin{subfigure}[c]{.17\linewidth}
    \centering
    %LAE $t=2$
    \scalebox{0.9}{$\opn{LAE}(X;t=2)$}
  \end{subfigure}
%  \begin{subfigure}[c]{.14\linewidth}
%    \centering\includegraphics[width=.95\linewidth]{figs/derivative_A_lae_t2_.eps}
%    \caption{\label{fig:lae_2_A}$\opn{LAE}=0.80$}
%  \end{subfigure}
  \begin{subfigure}[c]{.14\linewidth}
    \centering\includegraphics[width=.95\linewidth]{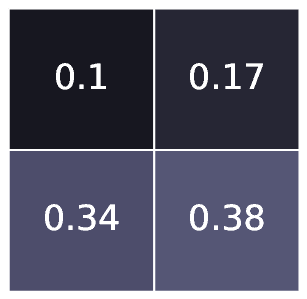}
    \caption{\label{fig:lae_2_B}$0.76$}
  \end{subfigure}
  \begin{subfigure}[c]{.14\linewidth}
    \centering\includegraphics[width=.95\linewidth]{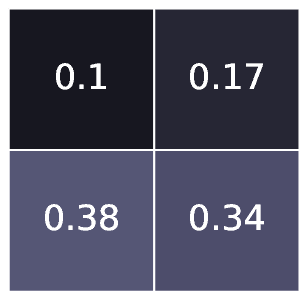}
    \caption{\label{fig:lae_2_D}$0.76$}
  \end{subfigure}

\vspace*{-0.1cm}
\caption{
The top row of the figure shows two distinct values of matrix $X$. Below each matrix $X$ are three derivatives, each a 2x2 matrix, and each corresponding to a different pooling operator $y = \opn{pool}(X)$. For example, subfigures \protect\subref{fig:max_B} and \protect\subref{fig:max_D} show that the derivatives of the max pool are completely different for the two matrices due to the swapped elements in its lower row (1.4 and 1.6), while subfigures \protect\subref{fig:avg_B} and \protect\subref{fig:avg_D} show that the derivatives of the average pool operator are identical. In contrast, the derivatives for all three of the LAE operators exemplify the continuous and distributed behaviour that we are aiming for.
The caption to each subfigure indicates the value returned by $\opn{pool}(X)$.
% %Note that the derivative of $\opn{LSE}(X)$ (not shown) is equal to that of $\opn{LAE}(X; t=1)$.
}
\label{fig:derivative-examples}

\end{figure*}

We may also add a temperature parameter to LAE, which acts to rescale the logits before applying the regular LAE operation.
This temperature is equivalent to the temperature used when sampling from a generative model (such as an LSTM); in such a case the logits returned by the network are rescaled by the temperature before performing the softargmax operation. %In both cases, a temperature $t>1$ decreases the confidence and moves the output closer to the
We define the temperature controlled variant of $\opn{LogAvgExp}$ as
\begin{align}
\opn{LogAvgExp}(\VEC{z}; t)
    =& \, t \cdot \opn{LogAvgExp}\left(\frac{\VEC{z}}{t}\right) \\
    =& \, t \, \left( \opn{LSE}\left(\frac{\VEC{z}}{t}\right) - \log(n) \right) \label{eq:lae-temp2}
.\end{align}
Henceforth, when we refer to $\opn{LogAvgExp}$ without a temperature parameter it can be assumed to be the temperature-free variant defined in \autoref{eq:lae}, which is equivalent to letting $t=1$ in \autoref{eq:lae-temp2}.

We note that the limiting cases for the temperature, $t \to 0^+$ and $t \to +\infty$,  are
\begin{align}
\lim_{t \to 0^+} \opn{LogAvgExp}(\VEC{z}; t) =& \max(\VEC{z}) = \max_{\forall i} z_i, \\
\lim_{t \to +\infty} \opn{LogAvgExp}(\VEC{z}; t) =& \opn{mean}(\VEC{z}) = \frac{1}{n} \sum_{i=1}^n z_i
.\end{align}
For a proof of these limits, see our supplementary materials.
This means that the temperature allows us to smoothly interpolate between max pooling and average pooling.

The temperature could be a predefined, fixed hyperparameter; but this does not necessarily have to be the case; it can also be a learnable parameter trained with backpropagation.%
\footnote{Derivation of $\frac{\partial}{\partial t}\opn{LogAvgExp}(\VEC{z}; t)$ is provided in the supplementary material.}
In this paper, we consider three temperature parameter variations:
(1) Omitted (fixed at $t=1$).
(2) One trainable temperature parameter per pooling layer.
(3) One trainable temperature parameter per channel.
% (4) Dynamically determined temperature as a learnable log-linear function of its operands, $\log(t) = \sigma(\VEC{w}^T\VEC{z})$, which we call Context-Aware temperature --- this variant is analogous to Gated Max-Avg pooling~\citep{generalizing-pooling}.
% (5) Dynamically determined temperature as a log-linear function of its operands, with a unique, learnable transformation per channel.

\section{Implementation}

We note that temperature is bounded below, $t \in (0, \infty)$. As temperature is multiplicative, a null hypothesis is $t=1$; \textit{a priori}, we expect temperatures of $\hat{t}$ and $\nicefrac{1}{\hat{t}}$ to be equally probable.
As our prior distribution for the temperature parameter is log-normally distributed around 1, we work with the log-temperature instead.
Our prior for $\log(t)$ is normally distributed around 0, and $\log(t)$ can take any value in the range $(-\infty, +\infty)$, making $\log(t)$ a much more well behaved parameter than $t$ during training.

We implemented $\opn{LAE}$ pooling in PyTorch v0.4.1 \citep{pytorch-paper}, with temperature parameter options as above, and implemented Mixed and Gated Max-Avg pooling \citep{generalizing-pooling} as additional benchmarks.
%Note that we did not consider the Tree pooling method from \citep{generalizing-pooling}, as this is only defined for pooling with a fixed kernel

%Our code is freely available online at REDACTED.

\section{Experimental Results}

We experimentally verified the utility of LAE pooling on the CIFAR-10 and CIFAR-100 datasets%
\footnote{Available from \url{https://www.cs.toronto.edu/~kriz/cifar.html}.} \citep{cifar-report},
and on version 1 of the Imagenette and Imagewoof datasets\footnote{Available from \url{https://github.com/fastai/imagenette}.} \citep{imagenette}.
Experiments were performed on NVIDIA GPUs: a Titan V and a number of Telsa P100s, with CUDA 9.0.

\subsubsection{Initial Results}

Initially, we trained state-of-the-art PyramidNet(depth=110, alpha=200) networks with ShakeDrop, with pre-trained Autoaugment augmentation policies, according to the training paradigm described in \citep{pyramidnet,shakedrop}: 300 epochs, SGD, batch size 128, momentum 0.9, weight decay 1e-4, initial learning rate 0.1 falling by a factor of 10 at 150 and 225 epochs.

\begin{figure}[htbp]
\centering
\begin{subfigure}[b]{0.49\textwidth}
    \centering
    \includegraphics[width=\textwidth]{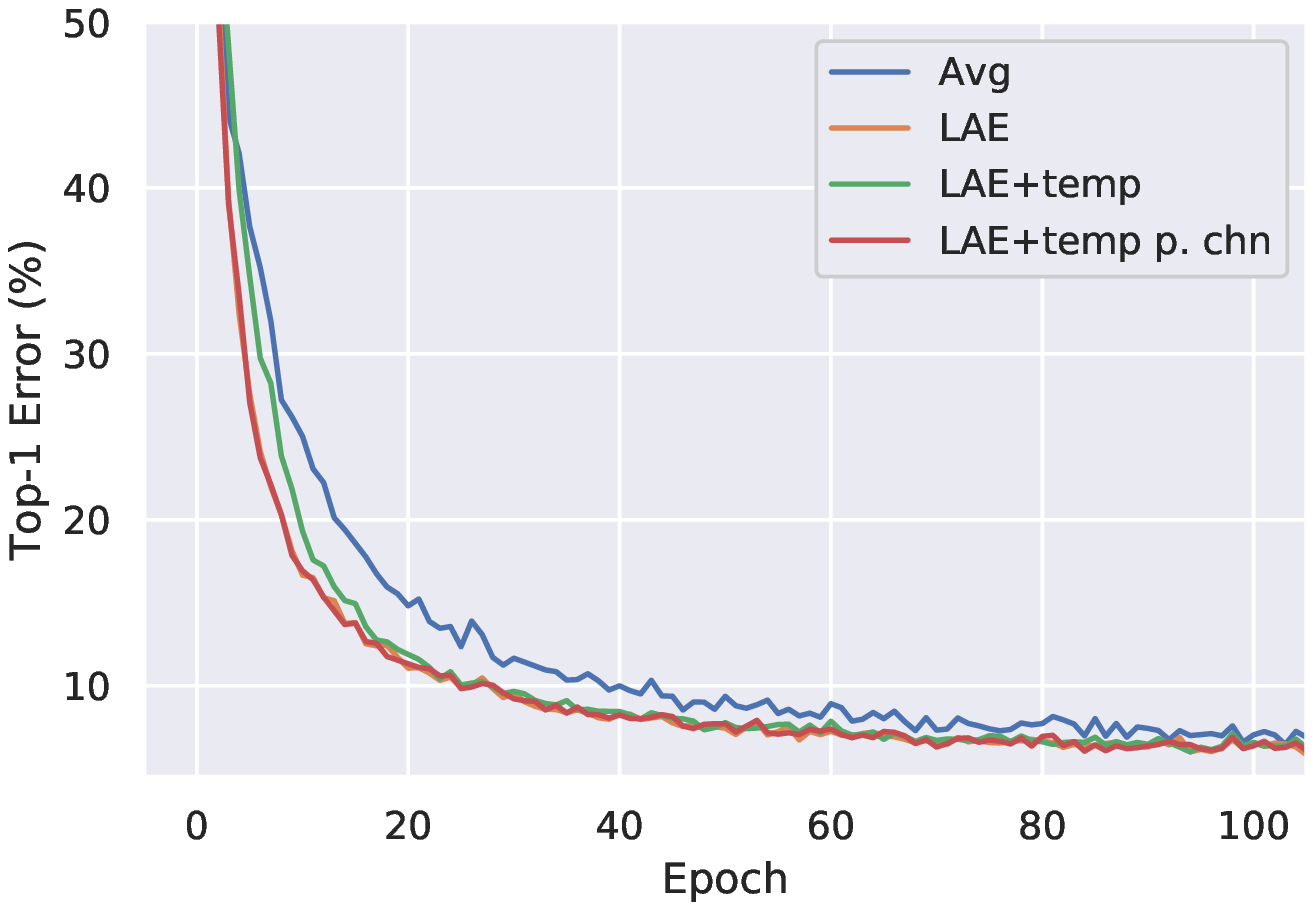}
    \caption{\label{fig:training-cifar10}CIFAR-10}
\end{subfigure}
\begin{subfigure}[b]{0.49\textwidth}
    \centering
    \includegraphics[width=\textwidth]{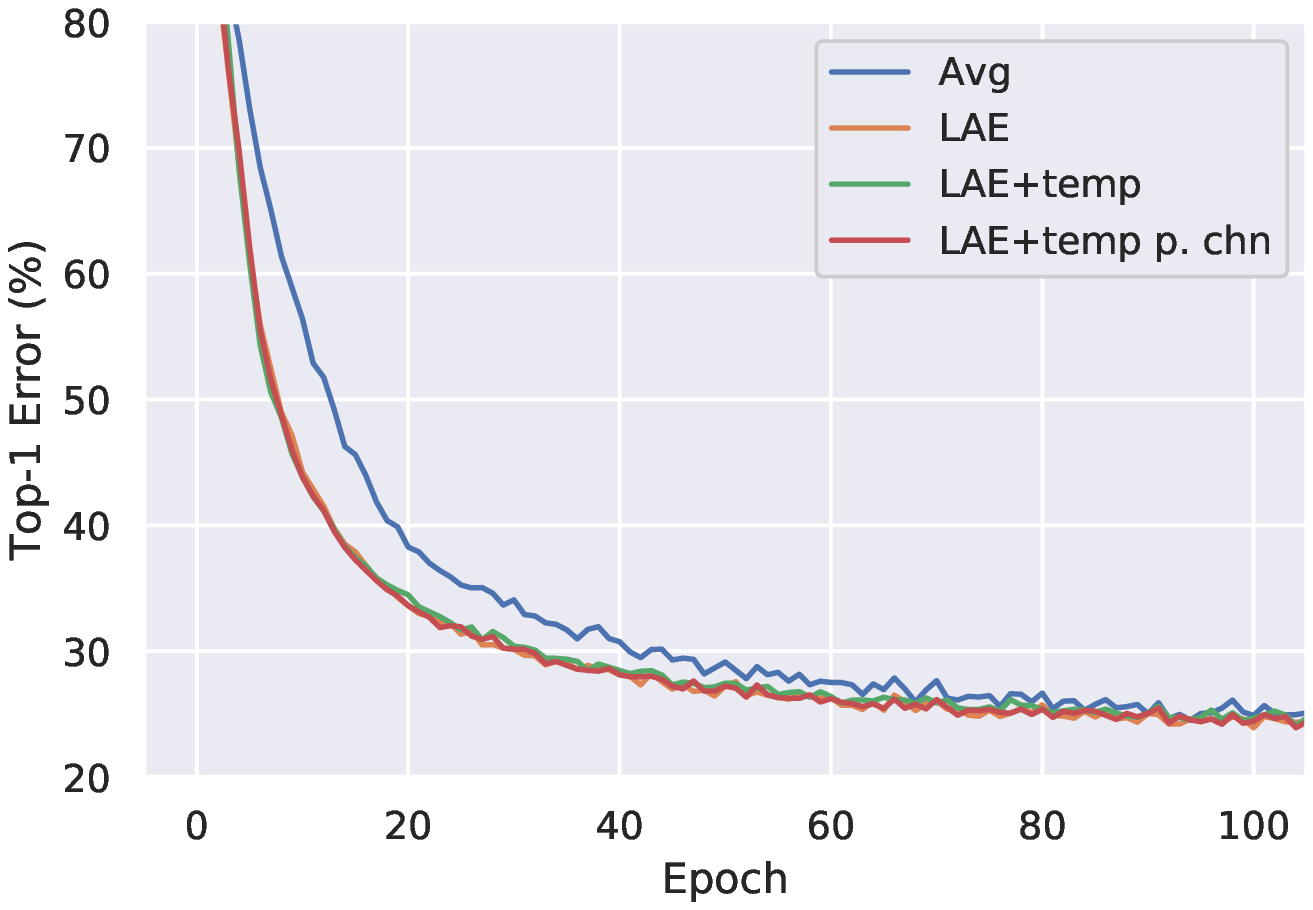}
    \caption{\label{fig:training-cifar100}CIFAR-100}
\end{subfigure}
\caption{Validation error during training of PyramidNet+ShakeDrop with initial LR 0.1. Average training curve over $n=5$ repetitions. The potential benefit of the LAE error signal over the average pooling is visible in all of the training curves, a behavior typical of all our experiments.}
\label{fig:training-curve}
\end{figure}

Despite the learning rate being the same, the effective learning rate was higher for networks using LAE for global pooling.
As shown in \autoref{fig:training-curve}, we found networks using LogAvgExp pooling learn much faster at the start of training.

Although for CIFAR-10, the top-1 accuracy was immediately higher using LAE than average pooling, for CIFAR-100 we had to re-tune the hyperparameters (away from those previously optimized for average pooling) in order to get results for LAE which were equivalent to our baseline with average pooling.

Training such a large model for so many epochs was very computationally demanding, and it was not feasible for us to do a large scale hyperparameter search on it.
Consequently, we progressed to considering smaller networks, trained for fewer epochs, with which we could more extensively explore the hyperparameter space.

%------------------------- This section *might* be dropped depending on readability
%%% dropped content
\drop{
\subsection{PyramidNet+ShakeDrop on CIFAR-10/100.}

Building on one of the leading model architectures for CIFAR-10/100, we trained a PyramidNet+ShakeDrop model \citep{pyramidnet,shakedrop}, adapting the implementation of \citep{shakedrop-code}. % implementation Masaki Yano \url{https://github.com/owruby/shake-drop_pytorch}

%For a pre-activation ResNet (PyramidNet among them), the order of operations at the end of the network is Conv+Residual-BatchNorm-ReLU-Pool-FC (we denote this ordering BRP).
%Applying the ReLU activation immediately before LAE pooling will truncate any negative elements away, changing the output of the pooling operation, and changing the derivative.

%Since PyramidNet is a pre-activation ResNet, its final operations not necessarily in the best order for pooling.
%We changed the BatchNorm-ReLU-Pool order to BatchNorm-Pool-BatchNorm-ReLU for these experiments.

\begin{table*}[htb]
\caption{
Top-1 error rates (\%) for PyramidNet+ShakeDrop.
``p. chn'' denotes a parameter or gate \textit{per channel}.
The smallest value in each column, and those deemed not significantly larger than the smallest value, are shown in bold (Welch's $t$-test, one-sided; $p<0.05$).
Each experiment was repeated with $n=5$ RNG seeds for the initialization state, and we indicate both the mean and standard deviation.
All experiments were run with the augmentation policies learnt by Autoaugment.
}
\label{tab:shakepyramid}
\centering
% \begin{tabular}{lrrrrrr}
%% sem
%\toprule
%Global Pooling Operation    & CIFAR-10              & CIFAR-100             \\
%\midrule
%Average                     & $\mbi{ 3.03}\pm 0.04$ & $\mbi{16.29}\pm 0.07$ \\
%LAE                         & $    {10.45}\pm 6.23$ & $    {17.96}\pm 0.09$ \\
%LAE + train $t$            & $    { 3.25}\pm 0.06$ & $    {19.00}\pm 0.13$ \\
%LAE + train $t$ p. chn     & $    { 3.34}\pm 0.07$ & $    {17.67}\pm 0.13$ \\
%LAE + Context temp          & $    {90.00}\pm 0.00$ & $    {92.98}\pm 2.25$ \\
%\bottomrule
%% stdev
% \begin{tabular}{lrrrrrr}
% \toprule
%                             & \multicolumn{2}{c}{CIFAR-10}                  & \multicolumn{2}{c}{CIFAR-100}                 \\
% \cmidrule(r){2-3} \cmidrule{4-5}
% Global Pooling Operation    & Initial LR=0.5        & Initial LR=0.1        & Initial LR=0.5        & Initial LR=0.1        \\
% \midrule
% Average                     & $    { 3.03}\pm 0.09$ & $    { 2.59}\pm 0.06$ & $    {16.29}\pm 0.15$ & $\mbf{14.83}\pm 0.14$ \\
% LAE                         & $    {10.45}\pm13.92$ & $\mbf{ 2.49}\pm 0.03$ & $    {17.96}\pm 0.20$ & $    {15.38}\pm 0.17$ \\
% LAE + train $t$            & $    { 3.25}\pm 0.14$ & $\mbf{ 2.56}\pm 0.12$ & $    {19.00}\pm 0.30$ & $    {15.51}\pm 0.14$ \\
% LAE + train $t$ p. chn     & $    { 3.34}\pm 0.15$ & $\mbf{ 2.39}\pm 0.11$ & $    {17.67}\pm 0.28$ & $    {15.42}\pm 0.18$ \\
% LAE + Context temp          & $    {90.00}\pm 0.00$ & $    {20.05}\pm34.98$ & $    {92.98}\pm 5.04$ & $    {15.24}\pm 0.29$ \\
\begin{tabular}{lrrr}
\toprule
                            & {CIFAR-10}            & \multicolumn{2}{c}{CIFAR-100}                 \\
\cmidrule(r){2-2} \cmidrule{3-4}
Global Pooling Operation    & Initial               & Initial               & Hyper optim.          \\
\midrule
Average                     & $    { 2.59}\pm 0.06$ & $    {14.83}\pm 0.14$ & $\mbf{14.42}\pm 0.10$ \\
LAE                         & $\mbf{ 2.49}\pm 0.03$ & $    {15.38}\pm 0.17$ & $\mbf{14.28}\pm 0.13$ \\
LAE + train $t$            & $\mbf{ 2.56}\pm 0.12$ & $    {15.51}\pm 0.14$ & \\
LAE + train $t$ p. chn     & $\mbf{ 2.39}\pm 0.11$ & $    {15.42}\pm 0.18$ & $\mbf{14.30}\pm 0.14$ \\
% LAE + Context temp          & $    {20.05}\pm34.98$ & $    {15.24}\pm 0.29$ \\
\bottomrule
\end{tabular}
\end{table*}

\drop{
\begin{figure}[htbp]
\centering
\begin{subfigure}[b]{0.49\textwidth}
    \centering
    \includegraphics[width=\textwidth]{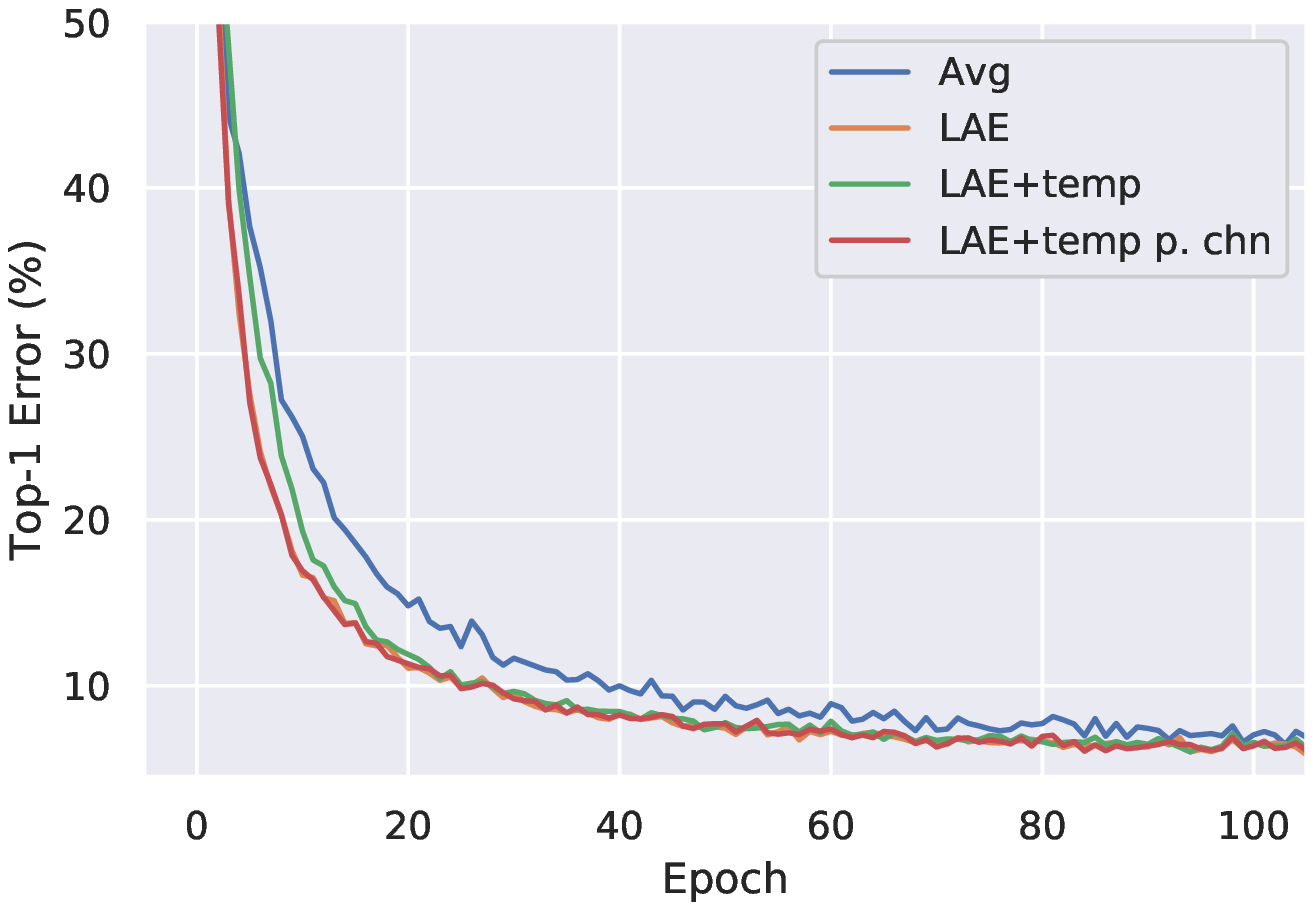}
    \caption{\label{fig:training-cifar10}CIFAR-10}
\end{subfigure}
\begin{subfigure}[b]{0.49\textwidth}
    \centering
    \includegraphics[width=\textwidth]{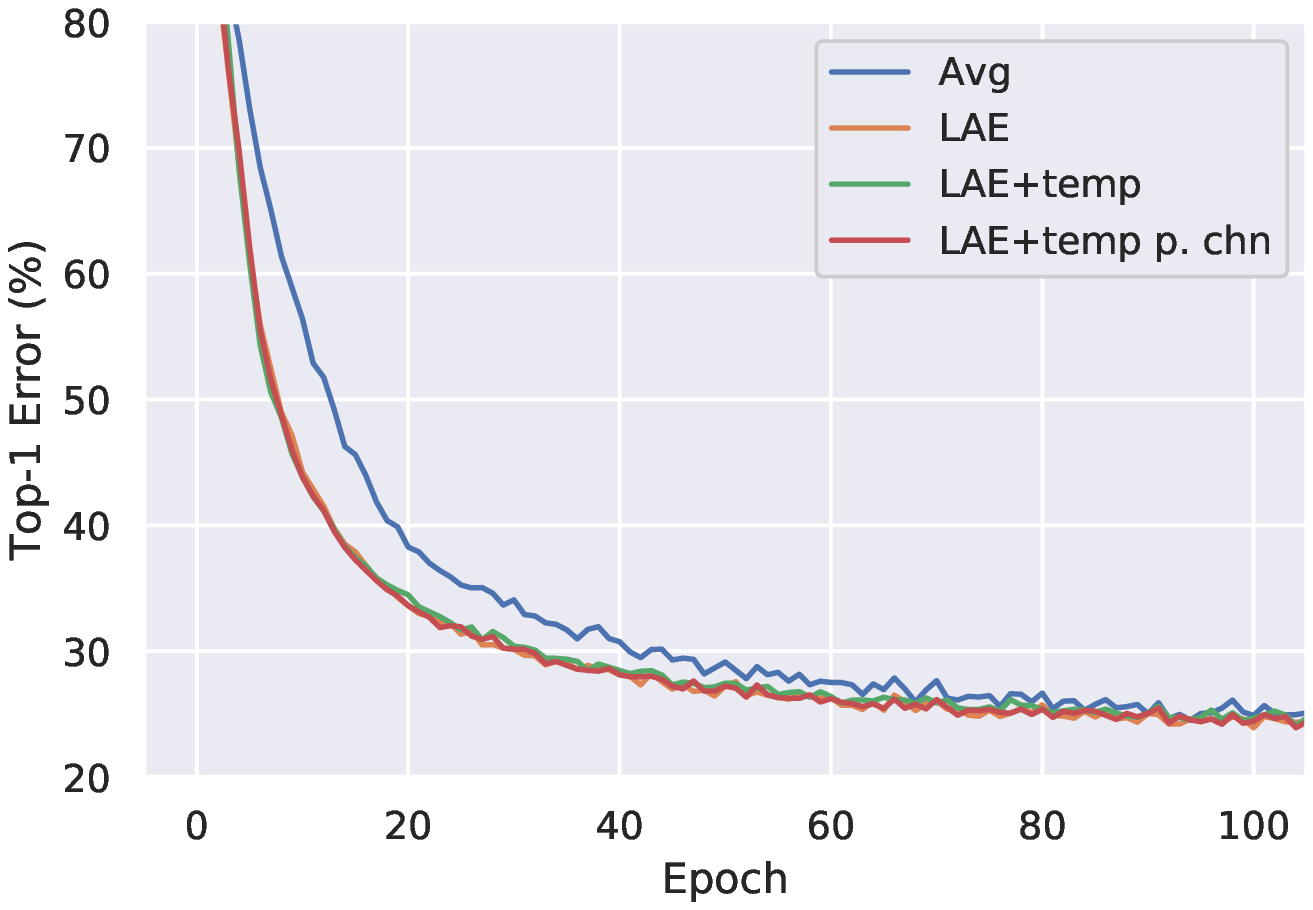}
    \caption{\label{fig:training-cifar100}CIFAR-100}
\end{subfigure}
\caption{Validation error during training of PyramidNet+ShakeDrop with initial LR 0.1. Average training curve over $n=5$ repetitions. The potential benefit of the LAE error signal over the average pooling is visible in all of the training curves, a behavior typical of all our experiments.}
\label{fig:training-curve}
\end{figure}
}

We trained a PyramidNet(depth=110, alpha=200) with ShakeDrop, with pre-trained Autoaugment augmentation policies, according to the training paradigm described in \citep{shakedrop}: 300 epochs, SGD, batch size 128, momentum 0.9, weight decay 1e-4, initial learning rate 0.5 falling by a factor of 10 at 150 and 225 epochs.
We found this worked for the average pooling, but the learning rate was too high for LAE pooling, with many runs diverging.
Consequently, we repeated the experiment with reduced initial learning rate of 0.1 (also previously used in \citep{pyramidnet}).
The reduced learning rate was better across the board, yielding an increase in performance for all models on both datasets.

For CIFAR-10, we achieved a top-1 error of $2.39\pm 0.11$\% with the \textit{LAE + temp per chn} pooling.
% However, the reduced learning rate was still too high to guarantee convergence of \textit{LAE + context-aware temp}; further improvements could be made with a hyperparameter search.
For CIFAR-100, lower top-1 error was attained with average pooling than with LAE pooling.

The top-1 validation error for the first 100 epochs of training on CIFAR-10/100 is shown in \autoref{fig:training-curve}.
This illustrates a finding common across all our experiments: when using any variant of LAE pooling, the network learns faster, for the same learning rate, than when using average pooling.

Since training networks with LAE pooling appeared to induce a higher effective learning rate, and the hyperparameters we were using were optimized for a network with average global pooling, we performed a small amount of manual fine-tuning of the hyperparameters: learning rate, learning rate decay factors, momentum, and weight decay.
As LAE is itself non-linear, the ReLU at the tail of the network may not be necessary for networks with LAE pooling, so we also replaced it with a Parameterized-ReLU (PReLU) \cite{prelu}, with which the network can learn to apply ReLU or the identity function if appropriate.
Changing the tail activation function to PReLU improved performance of average pooling as well as LAE pooling.
After performing these small optimizations to the network and training paradigm, we observed that both LAE pooling models we optimized now beat average pooling by a small margin (not significant).
(For fair comparison, we used the same techniques to try to improve on the baseline average pooling model.)

%During the course of this hyperparameter optimization, the global pooling model benefited from halving the weight decay, and LAE benefited from increasing the weight decay by a factor of 8.

%The 16x difference in optimal weight decay means that {\bf{LAE models have a smaller generalization margin} --- the gap between training and validation accuracy is {reduced by 22\%}}.

Training such a large model for so many epochs was very computationally demanding, and it was not feasible for us to do a large scale hyperparameter search on it.
Consequently, we progressed to considering smaller networks, trained for fewer epochs, with which we could more extensively explore the hyperparameter space.
}
%%% end dropped content

\subsection{WRN-18-6 on CIFAR-10/100.}
\label{sec:resnet}

We trained an 18-layer Wide ResNet, WRN-18-6 \citep{resnet,wrn} network, whose final ReLU activation function at the end of each block was replaced with a Parametrised-ReLU on CIFAR-10 and CIFAR-100.
Our implementation was based on that of \citep{hoffer-code,hoffer-fix}.
We substituted the global pooling operator from average pooling (original) to our LAE methods as described above.

\begin{table*}[htb]
\caption{
Top-1 error rates (\%) for WRN-18-6.
``p. chn'' denotes a parameter or gate \textit{per channel}.
The smallest value in each column, and those deemed not significantly larger than the smallest value, are shown in bold (Mann-Whitney rank $U$ test, two-sided; $p<0.05$).
Each experiment was repeated with $n=30$ seeds for the initialization state, and we indicate both the mean and standard deviation.
%All experiments were run with the augmentation policies learnt by Autoaugment.
}
\label{tab:wrn-18-6}
\centering
%
% n=10
%\begin{tabular}{lrr}
%\toprule
%Global Pooling Operation            & {CIFAR-10}            & CIFAR-100             \\
%\midrule
%Average                             & $    { 5.26}\pm 0.19$ & $    {22.19}\pm 0.25$ \\
%Mixed + trainable $\alpha$ p. chn   & $    { 5.19}\pm 0.13$ & $    {22.77}\pm 0.24$ \\
%Mixed + Gated $\alpha$ p. chn       & $\mbf{ 5.05}\pm 0.24$ & $    {22.07}\pm 0.24$ \\
%LAE                                 & $\mbf{ 5.08}\pm 0.12$ & $    {22.09}\pm 0.21$ \\
%LAE + train $t$                    & $\mbf{ 5.01}\pm 0.11$ & $\mbf{21.48}\pm 0.26$ \\
%LAE + train $t$ p. chn             & $\mbf{ 5.08}\pm 0.13$ & $    {22.56}\pm 0.32$ \\
%\bottomrule
%\end{tabular}
%
% n=30
\begin{tabular}{lrr}
\toprule
Global Pooling Operation            & {CIFAR-10}            & CIFAR-100             \\
\midrule
Average                             &\, $    { 5.26}\pm 0.15$ &\, $    {22.16}\pm 0.22$ \\
Mixed + trainable $\alpha$ p. chn   &\, $    { 5.15}\pm 0.11$ &\, $    {22.71}\pm 0.28$ \\
Mixed + Gated $\alpha$ p. chn       &\, $\mbi{ 5.00}\pm 0.16$ &\, $    {22.12}\pm 0.26$ \\
LAE                                 &\, $    { 5.06}\pm 0.16$ &\, $    {22.31}\pm 0.21$ \\
LAE + train $t$                    &\, $\mbf{ 4.98}\pm 0.12$ &\, $\mbf{21.55}\pm 0.26$ \\
LAE + train $t$ p. chn             &\, $    { 5.06}\pm 0.11$ &\, $    {22.54}\pm 0.29$ \\
\bottomrule
\end{tabular}
\end{table*}

To optimize the training hyperparameters of learning rate $\rho$, weight decay $\lambda$, and momentum $\mu$, we performed a hyperparameter search using a methodology similar to that described in \citep{howtotrainresnet}.
We performed our hyperparameter optimization routine with resolution $r=1.6$.
The process consisted of multiple rounds in which we change either:
(1) learning rate $\rho$ scaled up/down by factor $r$;
(2) weight decay scaled up/down by $r$ and learning rate simultaneously scaled inversely;
(3) momentum changed such that $(1 - \mu)$ is scaled up/down by $r$ and learning rate is simultaneously scaled similarly.
This process assumes that $[\rho,\, \rho\lambda,\, \nicefrac{\rho}{1-\mu}]$ form an independent basis along which the hyperparameters can be explored.
We initialised the hyperparameters with $\rho=0.4$, $\lambda=0.0004$, $\mu=0.9$, $r=1.6$, and chose initial search directions (increase/decrease) at random.
For each step, we compared the results of $n=5$ random cross-validation folds of the training data (80:20 split) for the current and candidate hyperparameters.

During cross-validation for hyperparameter optimization, we trained for 60 epochs on 80\% of the training data and evaluated on the remaining 20\%.
For final model evaluation, we trained the network for 48 epochs on the entire training data partition data, with a mini-batch size of 256.
The learning rate schedule was a linear ramp up from 0 to $\rho$ for 6 epochs, followed by a linear ramp down to 0 over the subsequent 42 epochs.
The data was augmented during training using the pre-trained augmentation policies learned by Autoaugment \citep{autoaugment}.

As shown in \autoref{tab:wrn-18-6}, we found that LAE with a single trainable temperature parameter was consistently the best global pooling method on both CIFAR-10 and 100, out-performing average pooling by a statistically significant margin.

\subsection{XResNet on Imagenette and Imagewoof}
\label{sec:imagenette}

We ran further experiments on Imagenette and Imagewoof\footnote{Available from \url{https://github.com/fastai/imagenette}.} \citep{imagenette}, using version 1 of the train/val partitions.
These two datasets are each a subset of the Imagenet dataset \citep{imagenet}, comprising 10 of its classes: Imagenette contains 10 dissimilar classes, whereas Imagewoof contains 10 breeds of dog.
They are intended to facilitate rapid development without reducing the complexity of the classification task compared with Imagenet (as they are only 1\% of the size of Imagenet, but full-scale input images).
Hence, we choose to run our experiments on Imagenette and Imagewoof so we could robustly optimize the training hyperparameters, which would not have been possible if training on the full Imagenet.

\begin{table*}[htb]
\caption{
Top-1 accuracy rates (\%) for MXResNet trained on imagenette and imagewoof.
``p. chn'' denotes a parameter or gate \textit{per channel}.
The largest value in each column, and those deemed not significantly smaller than the largest value, are shown in bold (Mann-Whitney rank $U$ test, two-sided; $p<0.05$).
}
\label{tab:imagenette}
\centering
\begin{tabular}{lrrrr}

%\toprule
%                            & \multicolumn{2}{c}{Imagenette}                 & \multicolumn{2}{c}{Imagewoof}                \\
%\cmidrule(r){2-3} \cmidrule{4-5}
%Global Pooling Operator     & 128\,px, 5 ep      & 256\,px, 5 ep    & 128\,px, 5 ep     & 256\,px, 5 ep     \\
%% Global Pool               & imagenette128         & imagenette256         & imagewoof128          & imagewoof256          \\
%\midrule
%% RangerMish \cite{rangermish} (SOTA) & $ {89.48}\qquad $ & $  {90.36}\qquad $ & $ 	 74.97\qquad   $ & $     74.92 \qquad  $ \\
%Average                      & $\mbf{89.64}\pm 0.82$ & $\mbf{90.86}\pm 0.73$ & $\mbf{75.20}\pm 1.08$ & $    {74.58}\pm 1.32$ \\
%Mixed + train $\alpha$ p. chn& $    {89.06}\pm 0.52$ & $    {90.22}\pm 0.59$ & $    {72.08}\pm 0.82$ & $    {73.80}\pm 1.63$ \\
%LAE                          & $\mbf{89.84}\pm 0.32$ & $\mbf{90.68}\pm 0.95$ & $\mbf{74.12}\pm 1.37$ & $\mbf{76.26}\pm 1.46$ \\
%LAE + train $t$             & $\mbf{89.30}\pm 0.94$ & $\mbf{90.64}\pm 0.81$ & $\mbf{74.34}\pm 1.08$ & $\mbf{76.02}\pm 0.86$ \\
%LAE + temp p. chn            & $\mbf{89.48}\pm 0.72$ & $\mbf{91.08}\pm 0.89$ & $\mbf{74.14}\pm 1.08$ & $\mbf{76.18}\pm 1.61$ \\
%LAE($t_0=8$) + temp p. chn   & $\mbf{89.88}\pm 0.64$ & $\mbf{91.10}\pm 0.60$ & $\mbf{75.24}\pm 0.80$ & $\mbf{75.97}\pm 0.66$ \\
%% Note: I accidentally ran the Average Imagenette 128 experiments with suboptimal hyperparameters. Results shouldn't change much, as they were only the second best instead of the best hyperparameters.
%\bottomrule
%
\toprule
                            & \multicolumn{2}{c}{Imagenette}                 & \multicolumn{2}{c}{Imagewoof}                \\
\cmidrule(r){2-3} \cmidrule{4-5}
Global Pooling Operator     & 128\,px, 5 ep      & 256\,px, 5 ep    & 128\,px, 5 ep     & 256\,px, 5 ep     \\
\midrule
Avg                             &\,$\mbf{89.45}\pm 0.84$&\,$    {90.44}\pm 0.71$&\,$\mbf{75.20}\pm 1.02$&\,$    {74.90}\pm 1.10$ \\
Mixed + train $\alpha$ p. chn   &\,$    {89.07}\pm 0.63$&\,$    {90.04}\pm 0.80$&\,$    {72.34}\pm 1.22$&\,$    {73.54}\pm 1.47$ \\
LAE($t_0=1$)                    &\,$\mbf{89.66}\pm 0.70$&\,$\mbf{90.83}\pm 0.69$&\,$    {73.87}\pm 1.44$&\,$\mbf{76.14}\pm 1.25$ \\
LAE($t_0=4$)                    &\,$\mbf{89.68}\pm 0.98$&\,$\mbf{90.62}\pm 0.69$&\,$\mbf{74.79}\pm 1.07$&\,$\mbf{76.27}\pm 1.27$ \\
LAE($t_0=1$) + train $t$        &\,$\mbf{89.61}\pm 0.90$&\,$\mbf{90.85}\pm 0.72$&\,$    {74.08}\pm 1.43$&\,$\mbf{76.06}\pm 1.32$ \\
LAE($t_0=4$) + train $t$        &\,$\mbf{89.62}\pm 0.64$&\,$\mbf{90.71}\pm 0.83$&\,$\mbf{74.97}\pm 1.15$&\,$    {75.78}\pm 1.11$ \\
LAE($t_0=1$) + train $t$ p. chn &\,$    {89.21}\pm 0.89$&\,$\mbf{90.71}\pm 0.78$&\,$    {73.96}\pm 1.41$&\,$\mbf{76.19}\pm 1.28$ \\
LAE($t_0=4$) + train $t$ p. chn &\,$\mbf{89.38}\pm 0.80$&\,$    {90.27}\pm 0.74$&\,$\mbf{75.27}\pm 1.18$&\,$\mbf{75.83}\pm 1.23$ \\
\bottomrule
\end{tabular}
\end{table*}

We tested LAE by adding it to the current state-of-the-art architecture on both datasets, \citep{rangermish}.
Following \citep{rangermish}, we used an XResNet network \citep{xresnet,bag-of-tricks} with Mish activation function \citep{mish}.
The network also had one layer of self-attention \citep{selfattention}, located at the start of the last residual block in the first group.
Using this base network, we compared the effect of changing the global pooling operator.

During training, the datasets were augmented with standard Imagenet image reflection/resizing/cropping.
The network was trained using the Ranger optimizer \citep{rangermish}, which combines RAdam \citep{RAdam} with LookAhead \citep{lookahead-opt}, using a mini-batch size of 64.
The learning rate was held constant at its initial value, $lr$, for some fraction, $a_0$, of the training epochs, after which the learning rate was annealed to zero using a cosine schedule.
The hyperparameters $lr$ and $a_0$ were optimized through random search, along with the weight decay coefficient $\lambda$, and the optimizer hyperparameters momentum $\mu$, alpha $\alpha$, and epsilon $\epsilon$.
The log-temperature parameter and mixing factor were excluded from the weight decay process.

We performed independent hyperparameter searches for Imagenette and Imagewoof, both at 128\,px and 256\,px input size.
To prevent overfitting to the test set, we partitioned the training set 80:20 to create a validation set for the hyperparameter search.
Our search was initially centered at $lr=0.004$, $mom=0.95$, $\alpha=0.99$, $\epsilon=10^{-6}$, $wd=0.02$, $a_0=0.72$, with $lr$, $(1-mom)$, $(1-\alpha)$, $\epsilon$, $\lambda$ sampled logarithmically and varying by x10 in each direction, and $a_0$ sampled linearly $0.5 < a_0 < 0.85$.
For the hyperpameter optimization routine, we trained the network for 6 epochs of its 80\% subset of the training data, so the total number of optimization steps was held approximately constant.
After some number of random samples (imagewoof128: 200, imagenette128: 100, 256\,px: 40), we refined our search to the span of the top-k hyperparameters by validation accuracy (128\,px: $k=10$, 256\,px: $k=5$), and resampled our train/val split.
After another set of random samples (128\,px: $n=100$, 256\,px: $n=50$), we refined our search once more, and generated another set of random samples (128\,px: $n=100$, 256\,px: $n=50$).
We selected the top-5 hyperparameter samples by accuracy, plus the (geometric) mean of the hyperparameters the top-5 and top-10 samples, and measured the performance of each with $n=5$ cross-validation folds (80:20) of the training set.
Finally, for each global pooling operator we selected the hyperparmeters with the highest cross-validation accuracy, and measured their performances on the test set. % for $n=40$ random initial states.
% Too detailed? Can we just say "random search" and not go into the step by step?

As shown in \autoref{tab:imagenette}, we found LAE pooling methods significantly outperformed average pooling on Imagewoof at 128\,px and 256\,px, and on Imagenette 128\,px.
There was no statistical difference between average pooling and LAE pooling options on Imagenette 256\,px.
We found there was generally no significant difference in performance between the LAE pooling options we tried.
% Our results set a new state-of-the-art across all four benchmark categories we attempted.

\subsubsection{Sensitivity to Input Resolution}

One of the advantages of using global pooling in a convolutional network is the same architecture can be used for a variety of different input sizes.
This is contrary to CNN architectures which use a fully-connected layer without global pooling (such as LeNet), which can only accept inputs of a fixed size.

For CNNs using global pooling, a change in the input size results in a change in the size of the latent space upon which the global pooling operator is applied.
Since the global pooling operator accepts an input of any size, and its output is the same spatial size, this can be handled by the network.
However, the choice of global pooling operator will change how the network behaves when the input changes size.

Consequently, we explored performance on a few size-related distortions of the input.
% Of course, the performance of the network will degrade as the input diverges from the training distribution.
Specifically we compared robustness of the pooling operators on zooming and cropping (each of which occurs to some degree in the training data augmentation), and zero-padding (which is not seen at all in training).

\begin{figure}[htbp]
\centering
\begin{subfigure}{0.45\textwidth}
    \includegraphics[width=\textwidth]{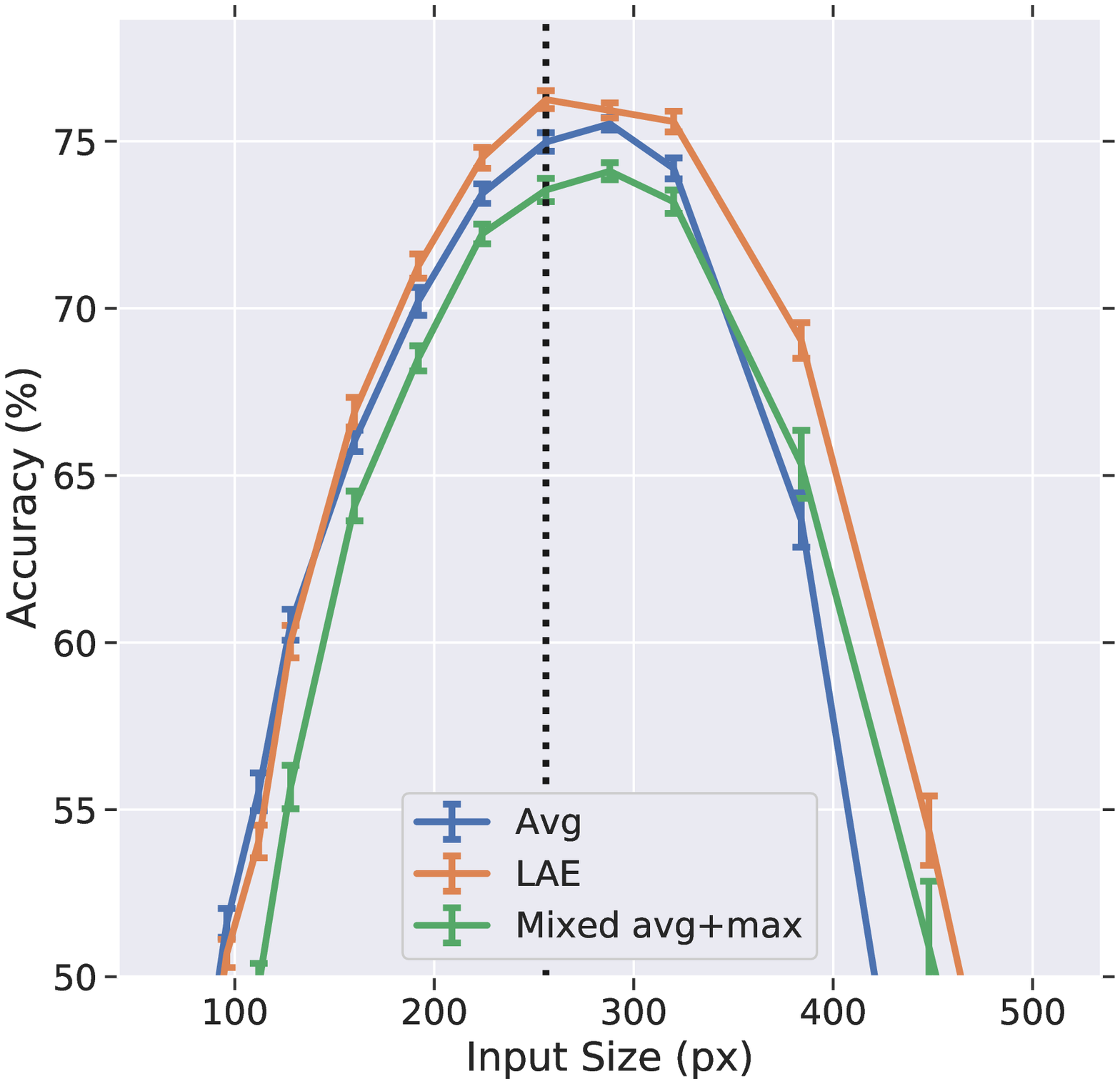}
    \caption{\label{fig:woof256-zoom}Stretch input to size}
\end{subfigure}
\begin{subfigure}{0.45\textwidth}
    \includegraphics[width=\textwidth]{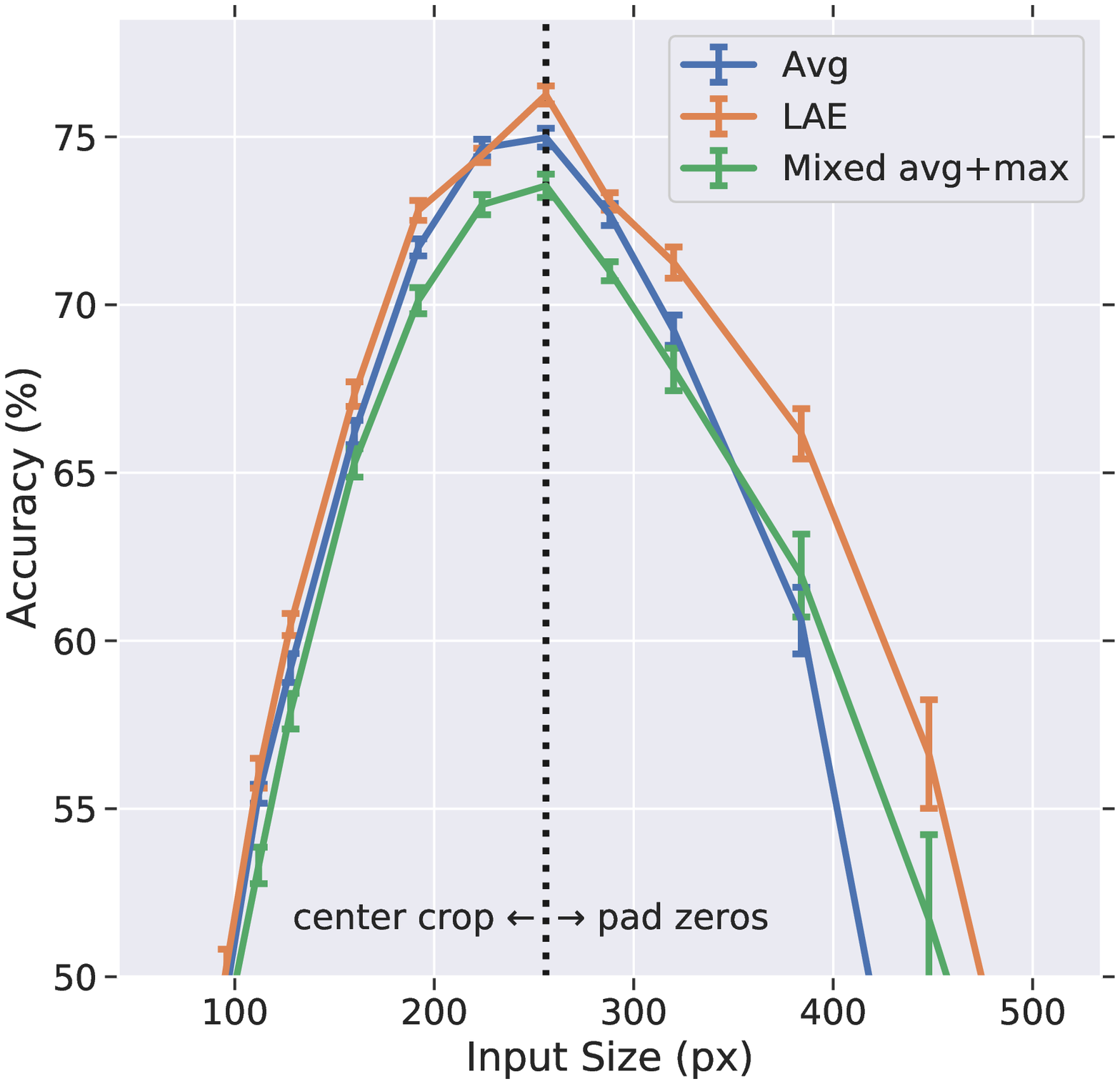}
    \caption{\label{fig:woof256-crop-pad}Crop or pad input to size}
\end{subfigure}
\caption{\label{fig:distort}The effect of changing the input resolution on the validation performance, with different global pooling operators.
Experiments were performed on Imagewoof, with networks trained on 256\,px inputs (black dotted line).
}
\end{figure}

In \autoref{fig:woof256-zoom} we change the size of validation images by zooming in and out on the Imagewoof 256\,px dataset, and evaluate the effect on the performance of the network.
We see that LAE pooling allows the network to generalise better to different input resolutions.

In \autoref{fig:woof256-crop-pad}, we use center cropping to reduce image size, and zero-padding to increase image size, and again observe a similar effect.

Similar behaviour was also observed in experiments on Imagenette, and for networks trained on 128\,px input sizes (see supplementary materials).

\subsubsection{Initial vs Final Temperature}

We inspected the distribution of LAE temperature values learned by the model, when using either a single common temperature across all channels, or an independent temperature per channel.
This was performed for a range of different initial temperatures.

\begin{figure}[htbp]
\centering
\includegraphics[width=0.95\textwidth]{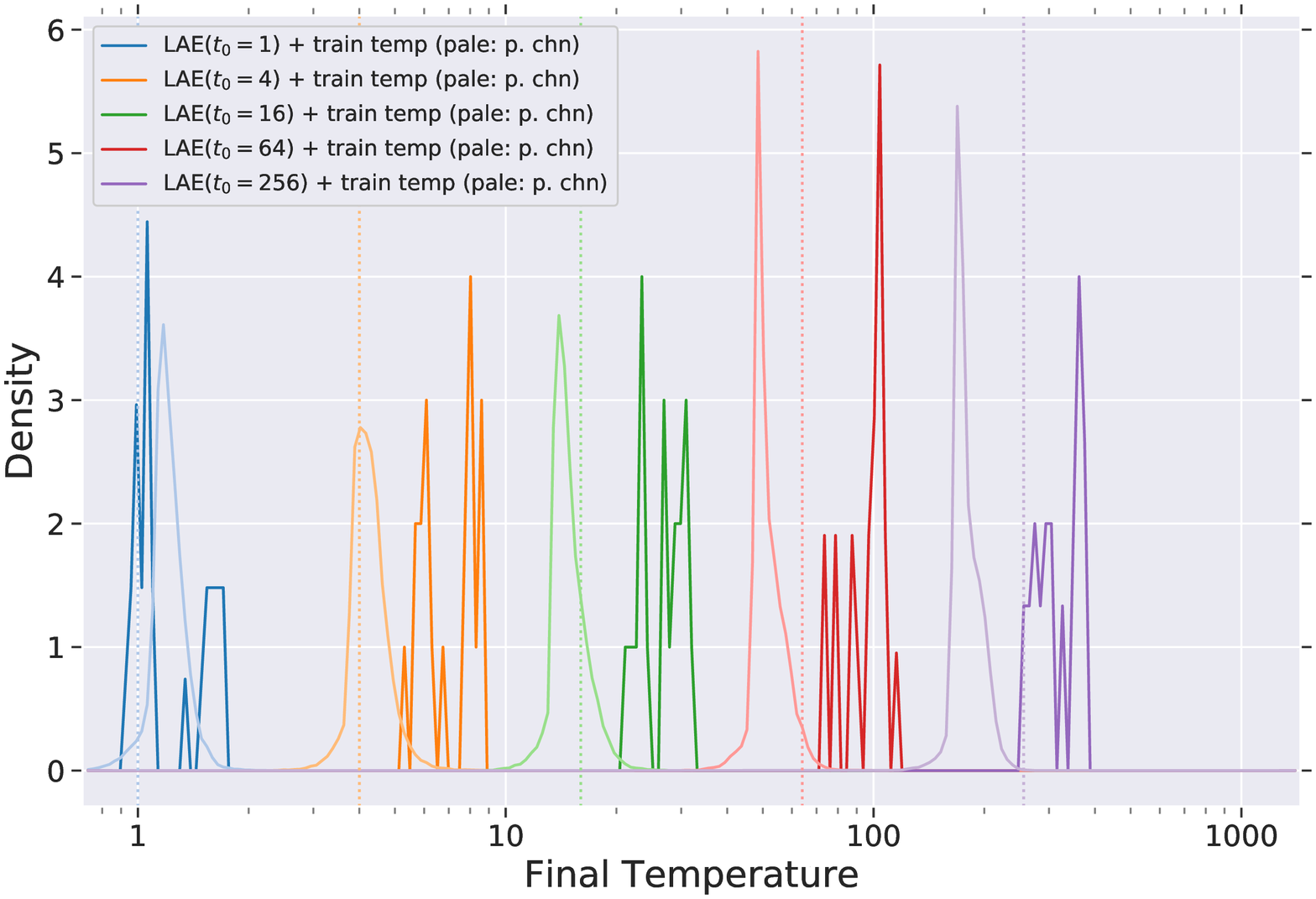}
\caption{\label{fig:final-temp}Distribution of final temperature after training with different initial temperatures.
In dark shades, LAE networks had a single trainable temperature parameter; in light shades a trainable temperature per channel.
Dotted lines indicate the initial temperature values.
Experiments were performed 10 times on each of imagenette and imagewoof, 128px, with networks trained for 25 epochs of random 80\% subsets of the training partition.
}
\end{figure}

We found that the temperatures did not converge to a value common to all initial temperatures.
This indicates the temperature parameter does not change very rapidly, and one may benefit from using a larger learning rate for the temperature parameter than for the network weights.
Additionally, this means there is some importance to the choice of initial temperature, even when the temperature is a trainable parameter.

When using a single temperature parameter, the temperature tended to increase, though not always when $t_0=1$ and not by much when $t_0=256$.
When using a temperature parameter per channel, the temperature tended to increase when $t_0\le4$ and decrease when $t_0\ge8$ (not shown), suggesting the optimal temperature lies in the range $4<t<8$.
Consequently, we ran experiments with $t_0=4$, as well as with $t_0=1$, the results of which are included in \autoref{tab:imagenette}.

\subsubsection{Floating Point Precision}

The LAE operation requires both logarithmic and exponentiation steps, which can be problematic due to the possibility of numeric instability with either overflow or underflow.
We used the log-sum-exp trick to compute LAE in a stable manner, limiting the impact of this problem.
To ensure that numeric instability was not still impinging on the performance of the network, we tried training LAE networks with half (FP16), single (FP32), and double (FP64) floating point precision for the LAE operation, with a range of temperature values.
Because the logits are divided by the temperature before computing LAE and multiplied by it afterwards, a large temperature value can exasperate any problems with stability.

We found there was discernable difference in performance between LAE with single or double precision across the whole range of temperatures we considered ($t\le1024$).
Using half precision did not negatively impact performance for $t<32$, but for $t>32$ performance was hindered due to numeric underflow in the derivatives.
The effect is illustrated in our supplementary materials.

\section{Squeeze and Excitation Global Pooling}
\label{sec:squeeze}

Another application where global pooling is commonly utilized is in Squeeze and Excitation (SE) blocks \citep{squeeze-and-excitation}.
With SE, feature activations are spatially integrated to create a bank of global activation features, which is then used to create an attention-like vector to modulate channels across all space.
In the original paper, the global squeezing of features is performed with an average pooling operation, though the authors speculated that this basic operation could be improved.

\begin{table*}[htb]
\caption{
Top-1 error rates (\%) for WRN-18-6 with Squeeze and Excitation.
``p. chn'' denotes a parameter or gate \textit{per channel}.
The smallest value in each column, and those deemed not significantly larger than the smallest value, are shown in bold (Mann-Whitney rank $U$ test, two-sided; $p<0.05$).
Each experiment was repeated with $n=30$ seeds for the initialization state, and we indicate both the mean and standard deviation.
% All experiments were run with the augmentation policies learnt by Autoaugment.
}
\label{tab:wrn-18-6_se}
\label{tab:se}
\centering
\begin{tabular}{llrr}
\toprule
Global Pooling Operation            &\, SE Pooling Operation              &\, {CIFAR-10}            &\, CIFAR-100             \\
\midrule
Average                             &\, Average                           &\, $    { 5.07}\pm 0.12$ &\, $    {21.92}\pm 0.20$ \\
Average                             &\, LAE                               &\, $\mbf{ 4.91}\pm 0.14$ &\, $    {21.94}\pm 0.30$ \\
Average                             &\, LAE + train $t$                  &\, $\mbf{ 4.98}\pm 0.15$ &\, $\mbf{21.71}\pm 0.28$ \\
Average                             &\, LAE + train $t$ p. chn           &\, $\mbf{ 4.89}\pm 0.19$ &\, $\mbf{21.79}\pm 0.31$ \\
LAE                                 &\, LAE                               &\, $    { 5.01}\pm 0.12$ &\, $    {21.99}\pm 0.24$ \\
LAE + train $t$                    &\, LAE + train $t$                  &\, $\mbf{ 4.95}\pm 0.18$ &\, $    {22.32}\pm 0.28$ \\
LAE + train $t$ p. chn             &\, LAE + train $t$ p. chn           &\, $    { 5.04}\pm 0.15$ &\, $    {22.63}\pm 0.22$ \\
\bottomrule
\end{tabular}
\end{table*}

We ran experiments with SE blocks added to the WRN-18-6 network described in \autoref{sec:resnet}, and with hyperparameters also optimized in the same manner.
As detailed in \autoref{tab:se}, we find that using LAE for SE pooling gives better performance than using average pooling --- provided one is using average pooling for the global layer pooling at the penultimate layer of the network.
When the global pooling is LAE, we typically do not find any gain in adding SE (comparing to \autoref{tab:wrn-18-6}).

\section{Discussion}

We described a new pooling operation, LogAvgExp (LAE), which is theoretically grounded as an \boolop{or}-operator for logits, including a correction for the number of operands.
By introducing a temperature parameter, LAE can smoothly interpolate between max and mean pooling.
This temperature can be a static model parameter, or a trainable component of the network.
If it is trainable, we can consider a single temperature parameter across all channels, a temperature parameter per channel, or a context-aware temperature determined from the activations within the pooling kernel.

Having tested against a range of models, we found that LAE pooling generally outperforms simpler pooling methods, such as average pooling. %For example, as visible in \autoref{tab:shakepyramid}, the results on CIFAR-10 are essentially SOTA relative to other networks of comparable size...
We expect the better credit assignment of LAE pooling (\autoref{fig:derivative-examples}) over both average and maximum pooling to lead to a stronger learning signal.
Indeed, we find that LAE pooling consistently learns faster at the early stages of training (see~\autoref{fig:training-curve}).
However, since the effective learning rate is higher when using LAE pooling compared with average pooling, the optimal hyperparameters will be different for models with the different global pooling operators.
If you have already optimized the hyperparameters for a network using average pooling, you are unlikely to see a benefit simply by swapping out the pooling operator, certainly for LAE with the default temperature of 1.
But if you choose LAE for the global pooling before optimizing the hyperparameters of the network, you are likely to benefit from doing so.
% And even without this, you may still be able to see benefits if you use LAE with a higher initial temperature of at least 8 or 16.
We recommend using LAE with an initial temperature of $t_0=4$, and letting the temperature be a trainable model parameter (either a single temperature value common to all channels, or a temperature parameter per channel).
% Furthermore, this suggests that LAE has potential to benefit significantly from better per-architecture tuning of hyperparameters.
% For fair comparison we left these untouched so far --- corresponding to conservative estimates --- but initial results with smaller learning rates for LAE do indicate even better performance. 
%As we have not optimized the hyperparameters such as learning rate on a per architecture basis, it is likely that most of the difference in performance between networks with LAE pooling versus average pooling comes down to the difference in effective learning rate between them.
%During training, we found that while final convergence times are comparable, LAE pooling initially learnt faster than average pooling.

Additionally, we investigated the utility of LAE as a pooling operator for Squeeze and Excitation blocks.
We found it performed better than average pooling, though the benefits did not appear to stack when using LAE pooling for both the network's penultimate global pooling operation and the SE pooling.
We recommend using LAE for SE blocks.

\subsubsection*{Acknowledgments}

Resources used in preparing this research were provided, in part, by the Province of Ontario, the Government of Canada through CIFAR, and companies sponsoring the Vector Institute \url{https://vectorinstitute.ai/partners/}, and in part by DeepSense \url{https://deepsense.ca/}.
Additionally, we gratefully acknowledge the support of NVIDIA Corporation with the donation of the Titan Xp GPU used for this research.

\FloatBarrier

\bibliographystyle{splncs04}
\bibliography{main}

\input{./supplementary_inner.tex}

\end{document}

%% file: supplementary_inner.tex
In this supplementary material, we provide the following:
\begin{itemize}
\item validation results for LogAvgExp with varying initial temperature;
\item effect on performance of using different floating point precision for the LogAvgExp operation;
\item results with input resolutions different to that of training, for Imagenette and Imagewoof datasets;
\item the final hyperparameters used in the experiments, as discovered by our hyperparameter search;
\item proofs for bounds and limits of LogAvgExp parameterised by temperature as claimed in the main body of the text;
\item derivation of partial derivatives of LogAvgExp, with respect to $z_i$ and $t$.
\end{itemize}

% ====================================================================
\section{Impact of initial temperature on performance}
\label{sec:initial-temp}

We measured the impact of changing the initial LAE temperature on the final performance of the network, for the Imagenette and Imagewoof, 128 and 256 pixels resolution (5 epochs) benchmarks.
Experiments were performed on random 80/20 cross-validation folds of the training set, and trained for 6 epochs on the training subpartition.

\begin{figure}[htbp]
\centering
\begin{subfigure}[b]{0.48\textwidth}
    \centering
    \includegraphics[width=\textwidth]{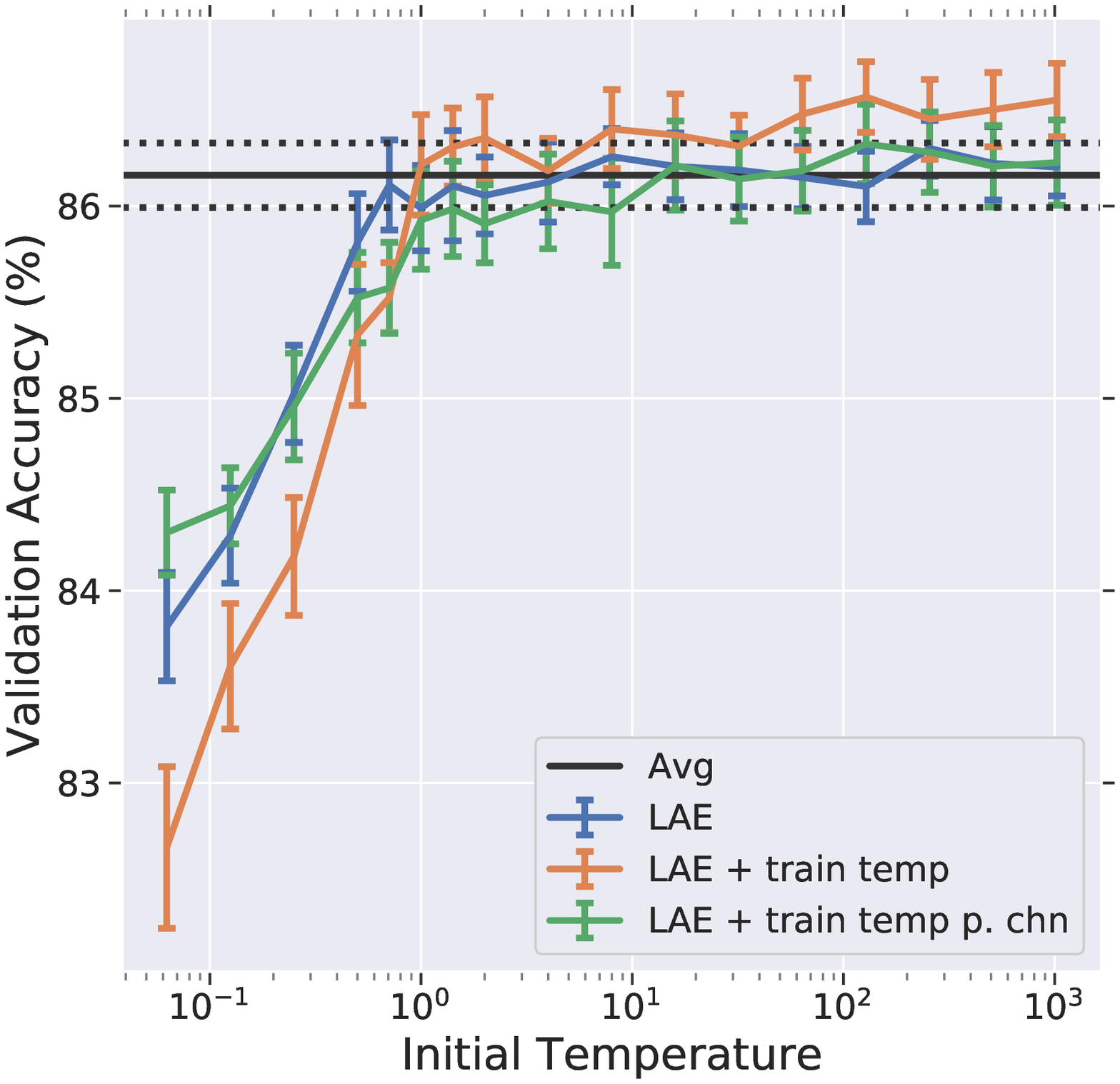}
    \caption{\label{fig:temp-val-ours-imagenette128}Imagenette 128px}
\end{subfigure}
~
\begin{subfigure}[b]{0.48\textwidth}
    \centering
    \includegraphics[width=\textwidth]{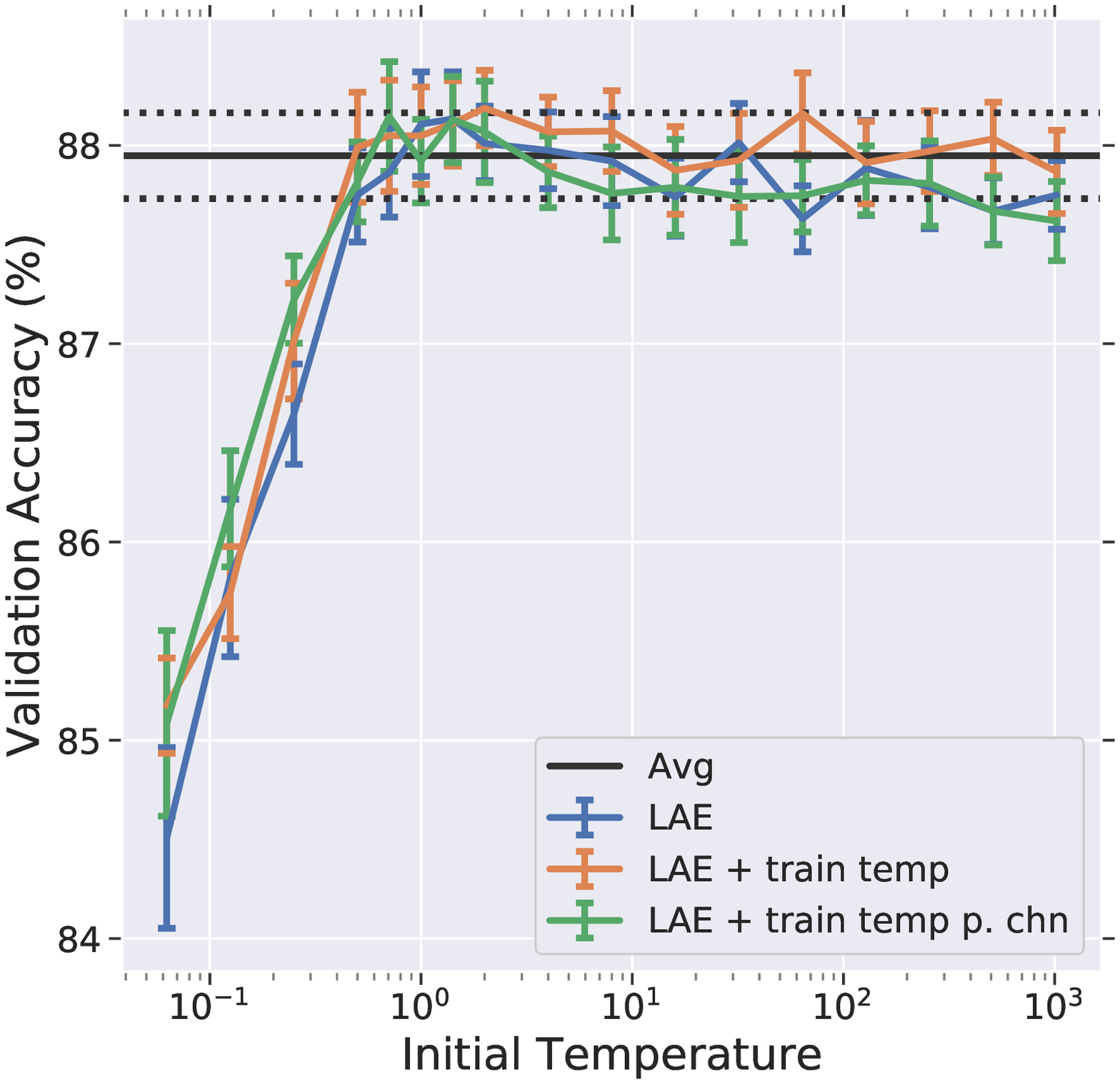}
    \caption{\label{fig:temp-val-ours-imagenette256}Imagenette 256px}
\end{subfigure}
\\
\begin{subfigure}[b]{0.48\textwidth}
    \centering
    \includegraphics[width=\textwidth]{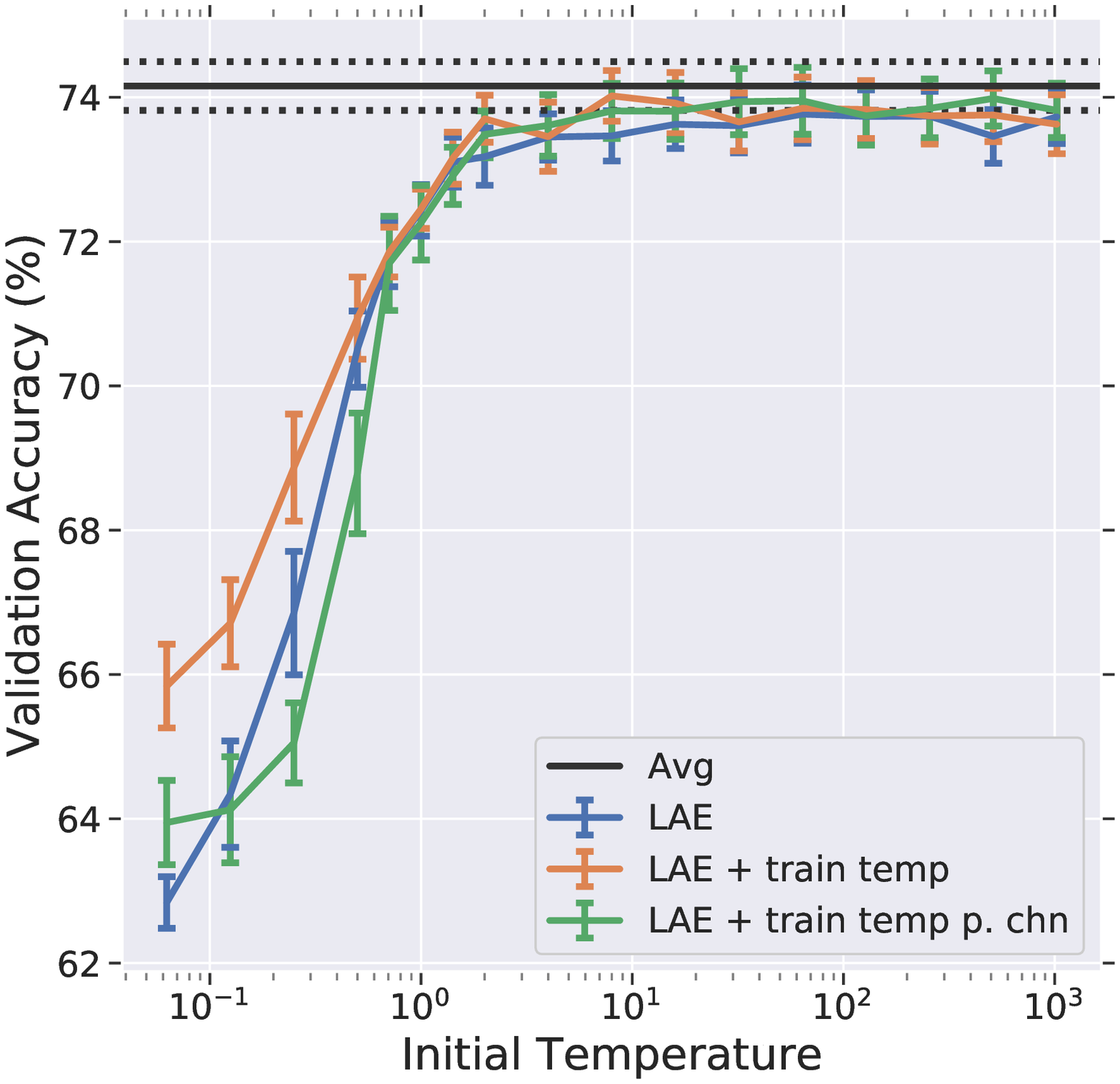}
    \caption{\label{fig:temp-val-ours-imagewoof128}Imagewoof 128px}
\end{subfigure}
~
\begin{subfigure}[b]{0.48\textwidth}
    \centering
    \includegraphics[width=\textwidth]{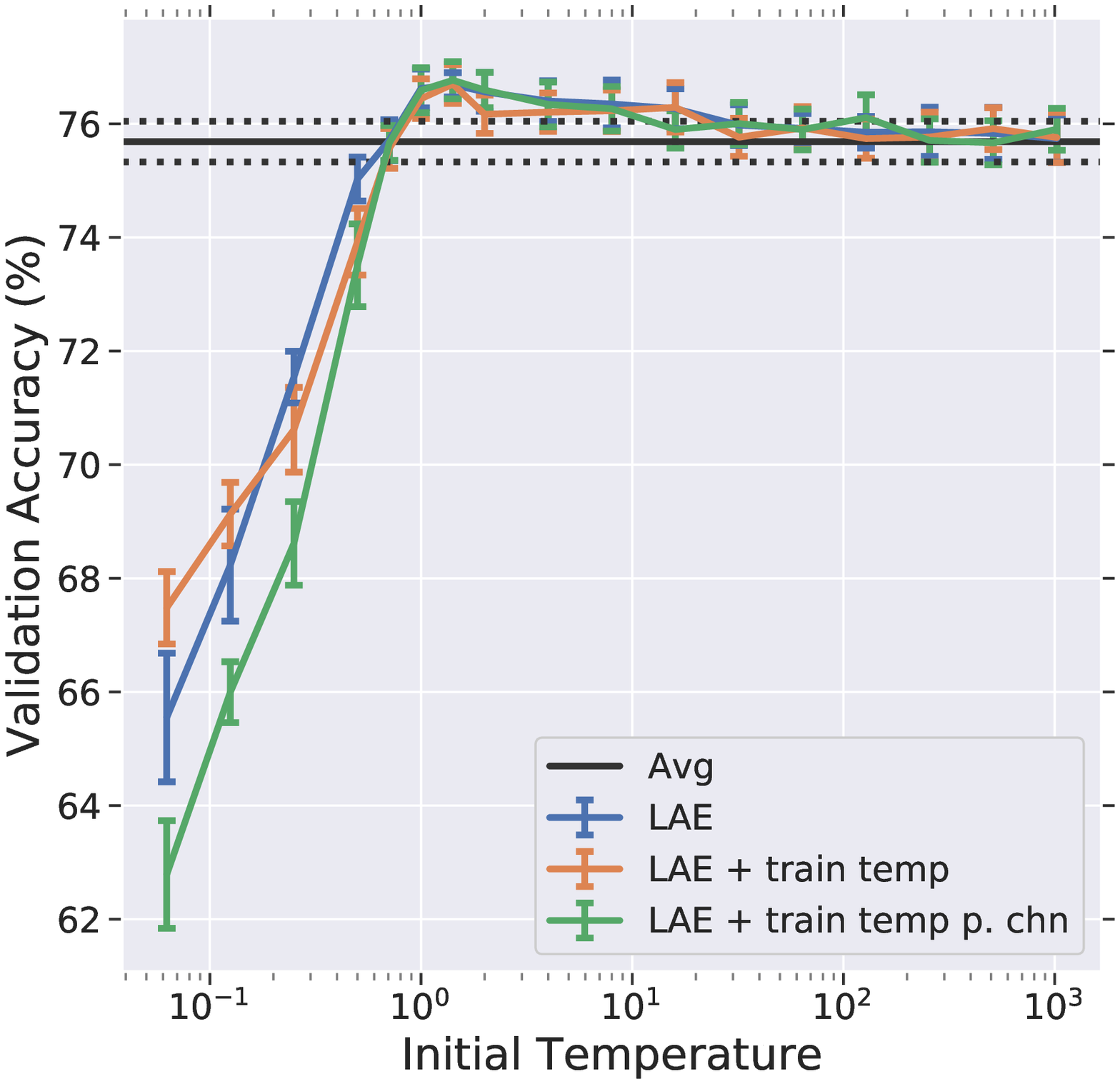}
    \caption{\label{fig:temp-val-ours-imagewoof256}Imagewoof 256px}
\end{subfigure}
\caption{
Initial temperatures versus final accuracy, using different hyperparameters for networks with each pooling operation (as discovered in hyperparameter search).
The LAE operations were performed using FP64 precision.
Mean and standard error (SEM) over $n=10$ random 80/20 cross-validation folds of the training set, trained for 6 epochs.}
\label{fig:temp-val-ours}
\end{figure}

As shown in \autoref{fig:temp-val-ours}, we find that the performance of the LAE pooling methods exceeds that of average pooling across the range $1 < t_0 < 32$ for imagenette 128, imagenette 256, and imagewoof 256.

These results do not necessarily converge to the same performance as the average pooling baseline (black), since the networks are trained using different hyperparameters.
To draw more exact comparisons across the models, we also ran the analysis using the same hyperparameters to train each model (the hyperparameters discovered for average pooling).
As shown in \autoref{fig:temp-val-avg}, the performance of the networks converge to match that of average pooling if the initial temperature is sufficiently large.

\begin{figure}[htbp]
\centering
\begin{subfigure}[b]{0.48\textwidth}
    \centering
    \includegraphics[width=\textwidth]{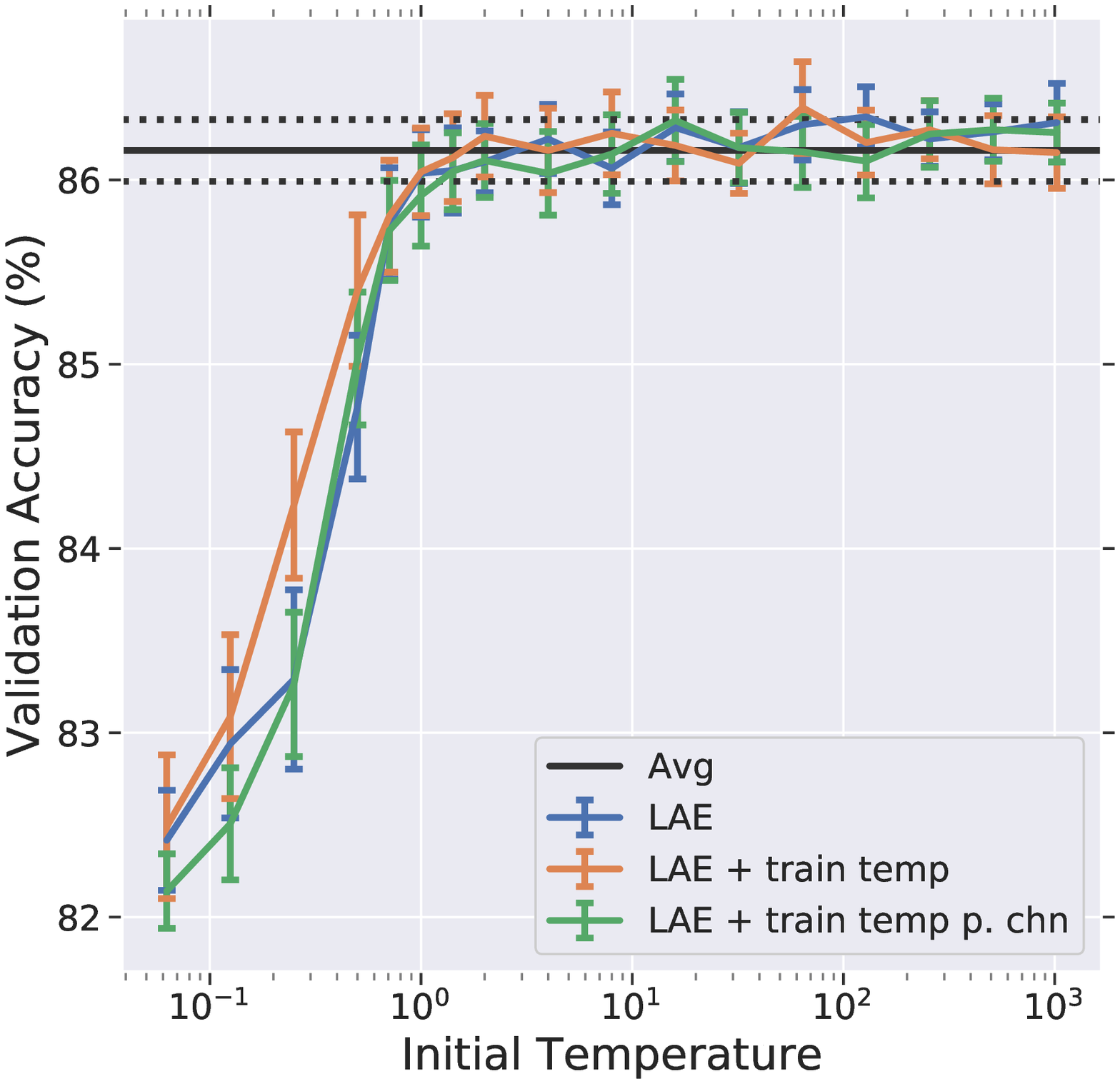}
    \caption{\label{fig:temp-val-avg-imagenette128}Imagenette 128px}
\end{subfigure}
~
\begin{subfigure}[b]{0.48\textwidth}
    \centering
    \includegraphics[width=\textwidth]{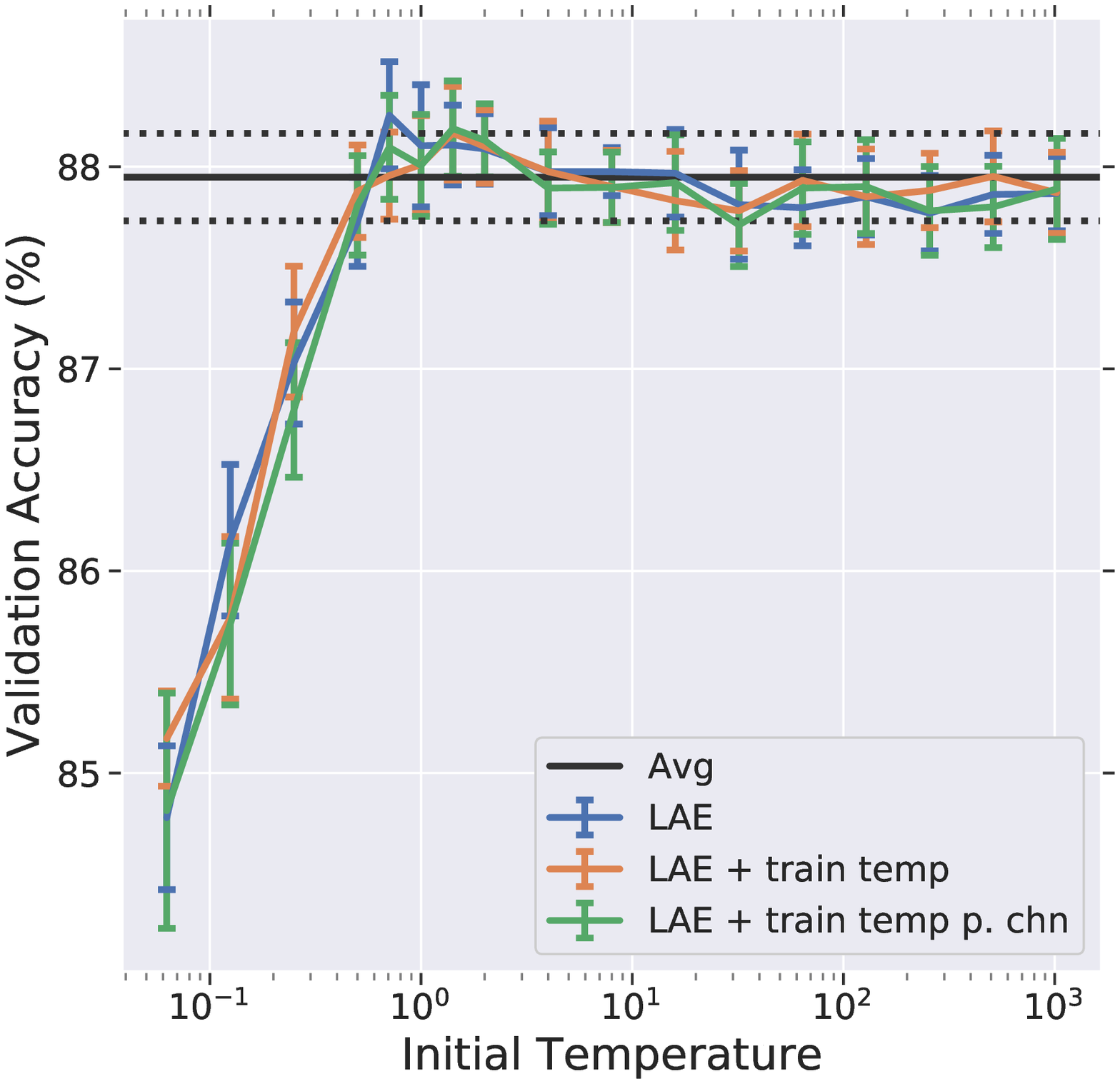}
    \caption{\label{fig:temp-val-avg-imagenette256}Imagenette 256px}
\end{subfigure}
\\
\begin{subfigure}[b]{0.48\textwidth}
    \centering
    \includegraphics[width=\textwidth]{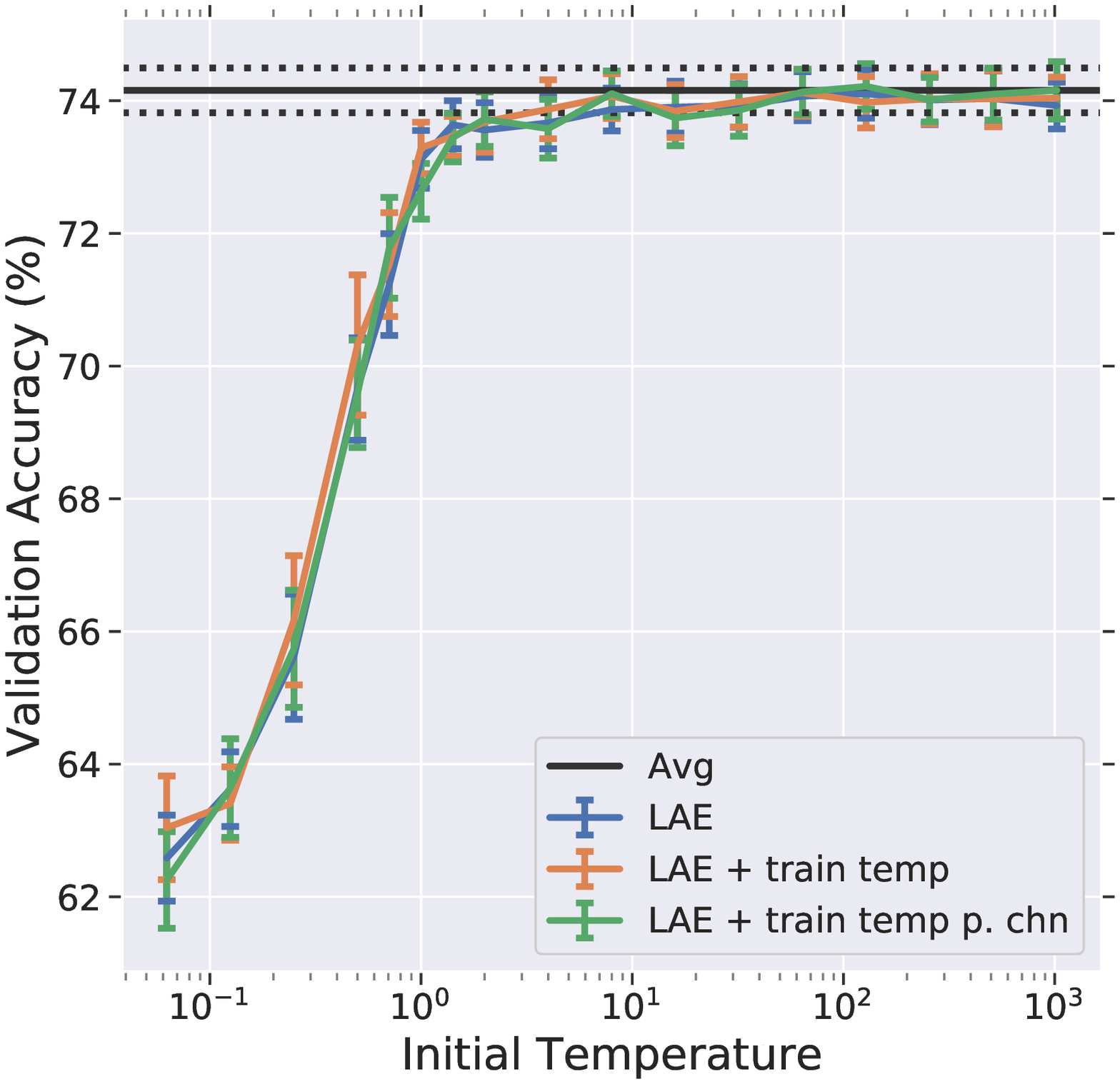}
    \caption{\label{fig:temp-val-avg-imagewoof128}Imagewoof 128px}
\end{subfigure}
~
\begin{subfigure}[b]{0.48\textwidth}
    \centering
    \includegraphics[width=\textwidth]{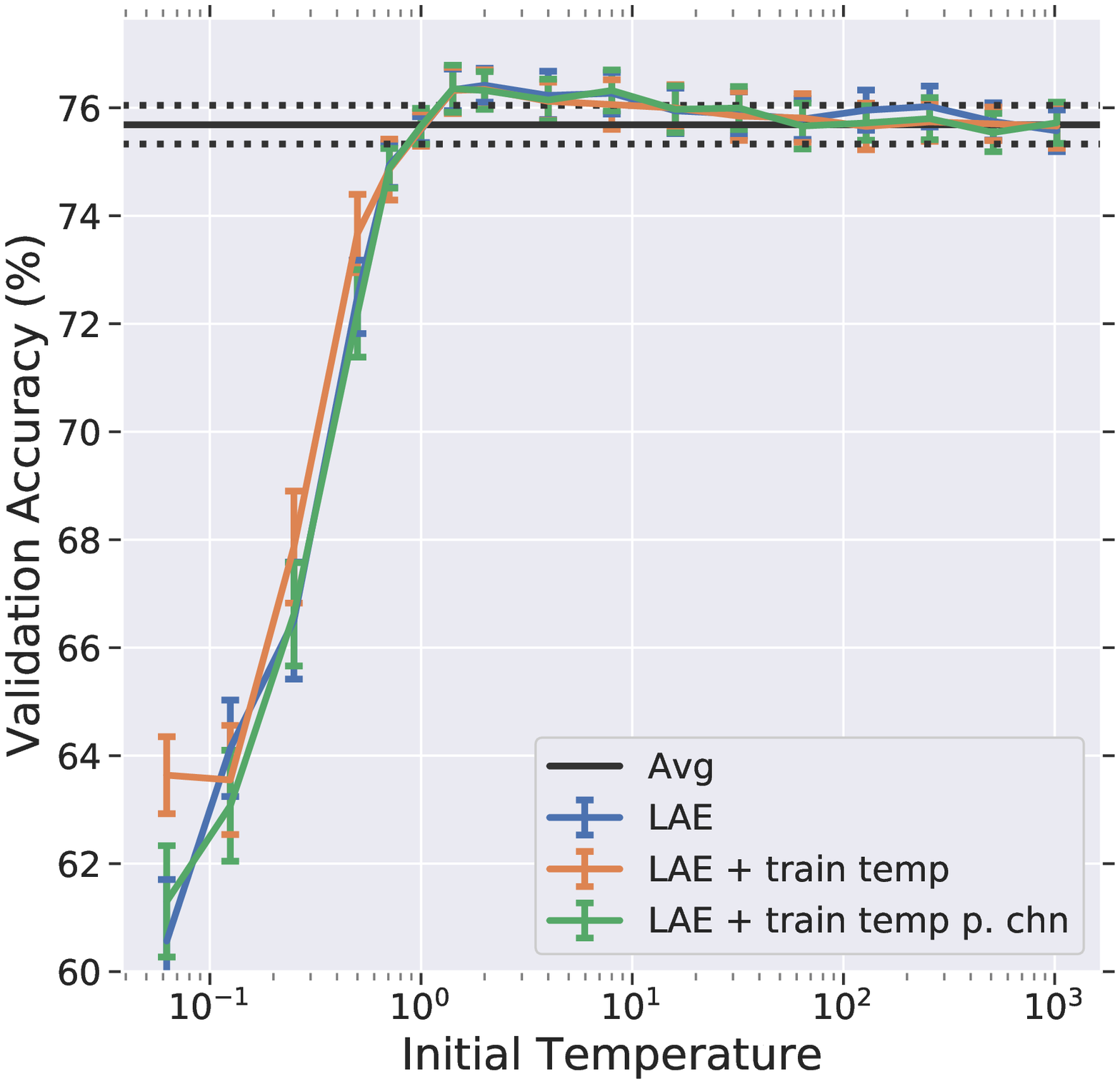}
    \caption{\label{fig:temp-val-avg-imagewoof256}Imagewoof 256px}
\end{subfigure}
\caption{
Initial temperatures versus final accuracy, using the same hyperparameters for networks with each pooling operation (as discovered in for average pooling).
The LAE operations were performed using FP64 precision.
Mean and SEM over $n=10$ random 80/20 cross-validation folds of the training set, trained for 6 epochs.
}
\label{fig:temp-val-avg}
\end{figure}

We note that the differences between LAE with fixed temperature, a single trainable temperature, and a trainable temperature per channel, are partially mitigated when we use the same hyperparameters for each network, as shown in \autoref{fig:temp-val-avg} compared against \autoref{fig:temp-val-ours}.

We also note that the performance of LAE pooling appears to be better on Imagewoof 128px when using the hyperparameters discovered for average pooling (\autoref{fig:temp-val-avg-imagewoof128} verus \autoref{fig:temp-val-ours-imagewoof128}).

% ====================================================================
\section{Impact of Floating Point precision on performance}

In this section, we consider the impact of the floating point precision used for the LogAvgExp operation.
We trained networks using LogAvgExp global pooling on Imagenette and Imagewoof using MXResNet, as per Section~4.2 of the main paper, using random 80/20 cross-validation folds of the training set.

\begin{figure}[htbp]
\centering
\begin{subfigure}[b]{0.48\textwidth}
    \centering
    \includegraphics[width=\textwidth]{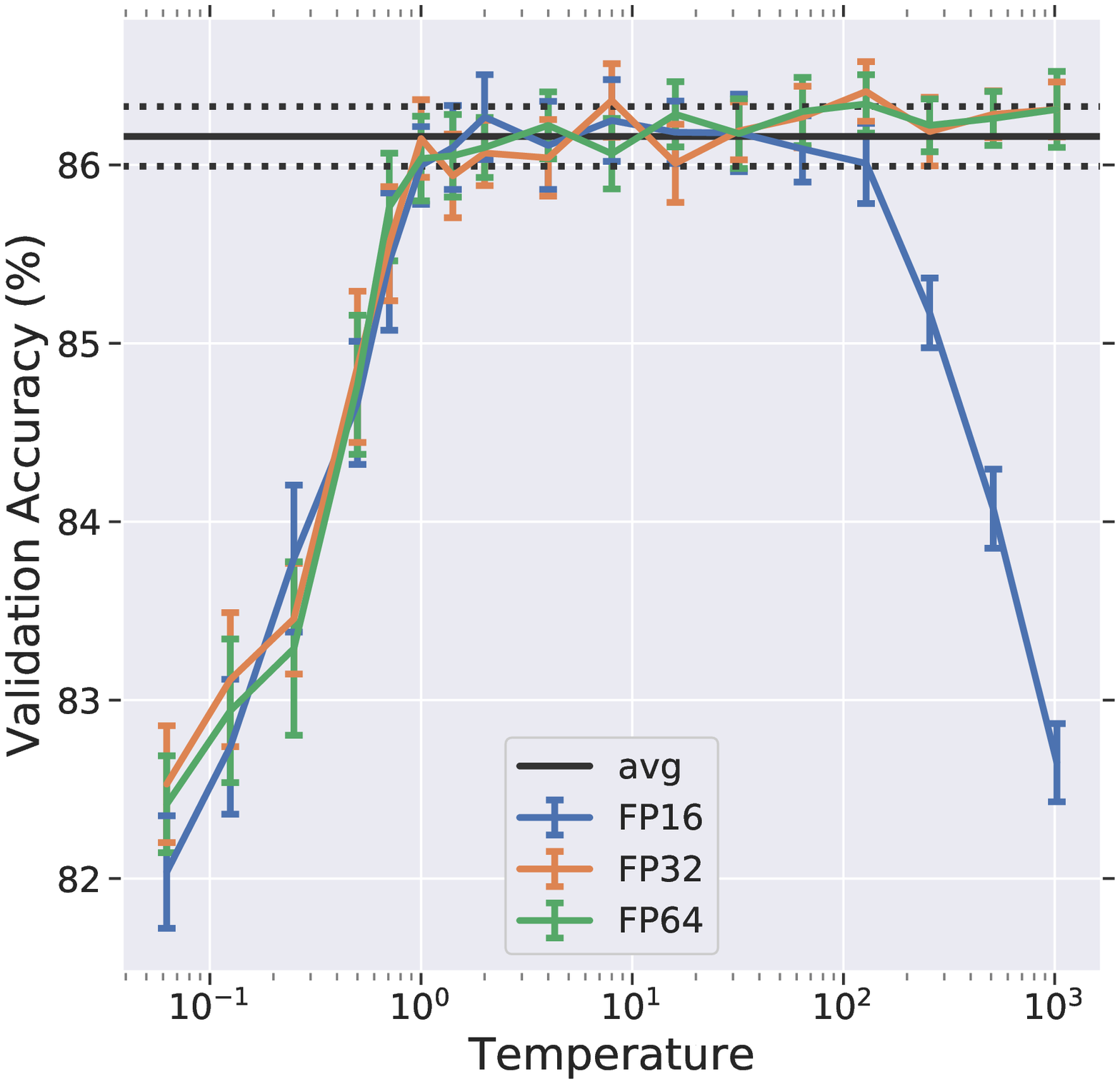}
    \caption{\label{fig:fp-cmp-imagenette128}Imagenette 128px}
\end{subfigure}
~
\begin{subfigure}[b]{0.48\textwidth}
    \centering
    \includegraphics[width=\textwidth]{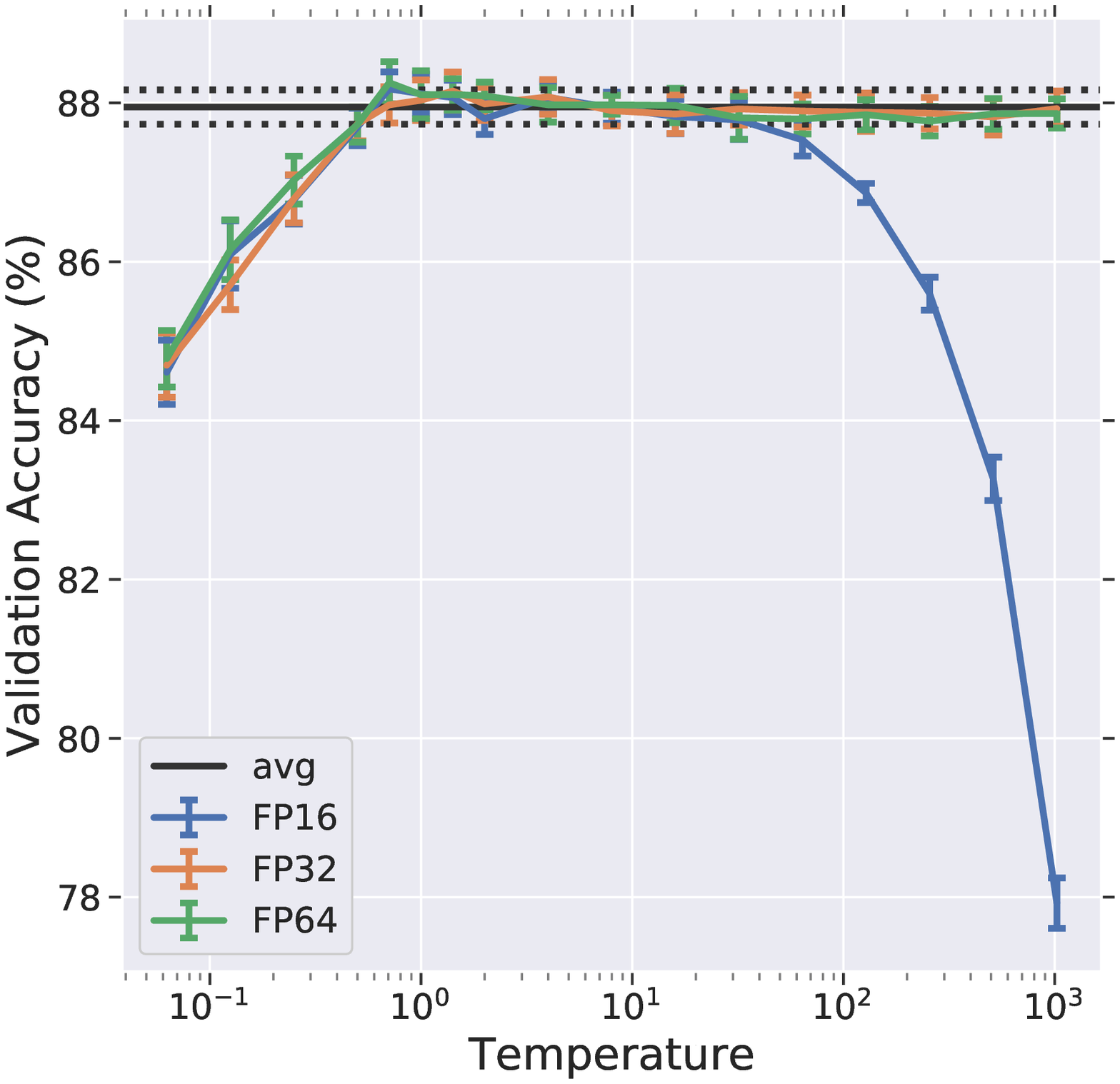}
    \caption{\label{fig:fp-cmp-imagenette256}Imagenette 256px}
\end{subfigure}
\\
\begin{subfigure}[b]{0.48\textwidth}
    \centering
    \includegraphics[width=\textwidth]{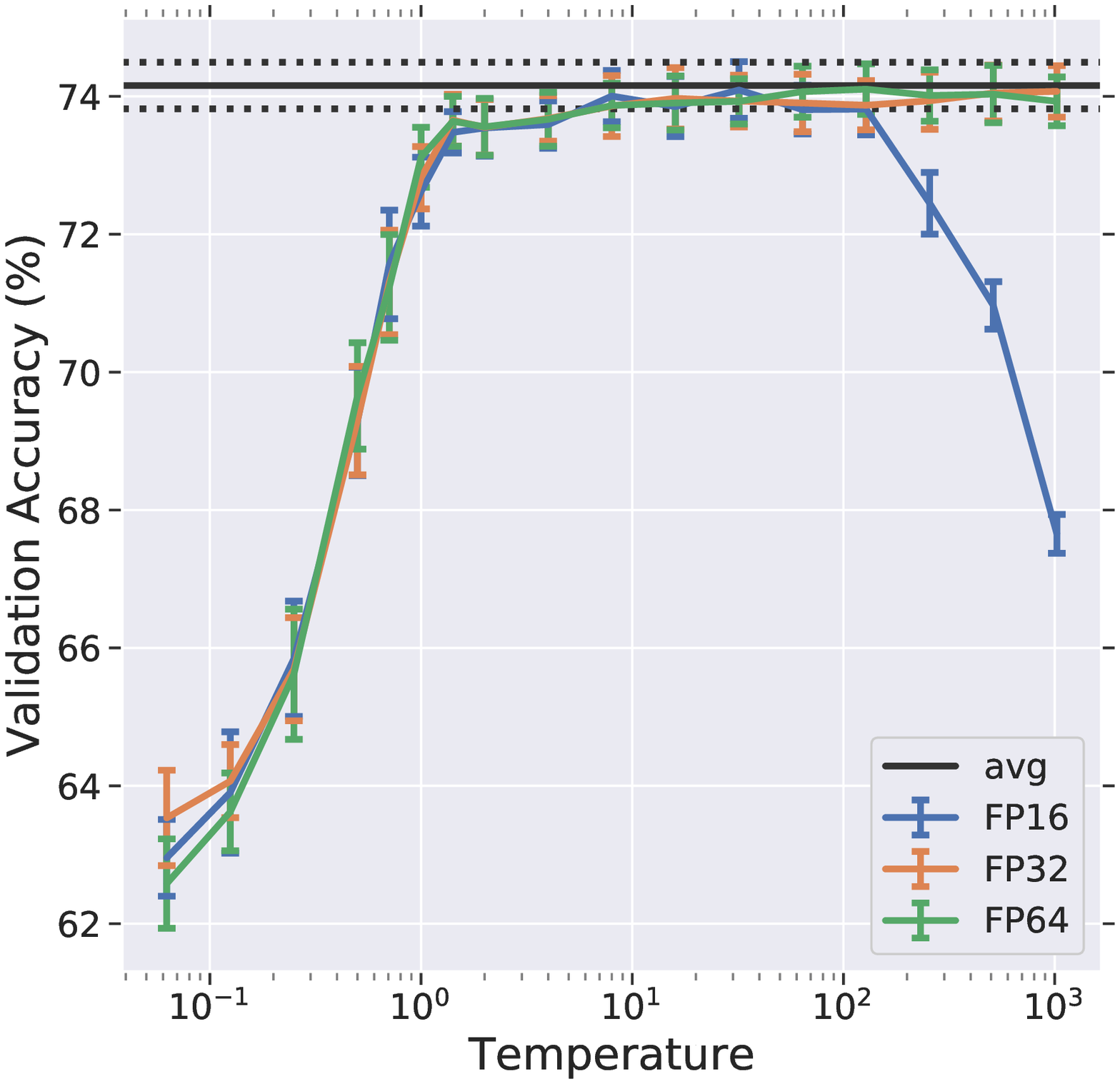}
    \caption{\label{fig:fp-cmp-imagewoof128}Imagewoof 128px}
\end{subfigure}
~
\begin{subfigure}[b]{0.48\textwidth}
    \centering
    \includegraphics[width=\textwidth]{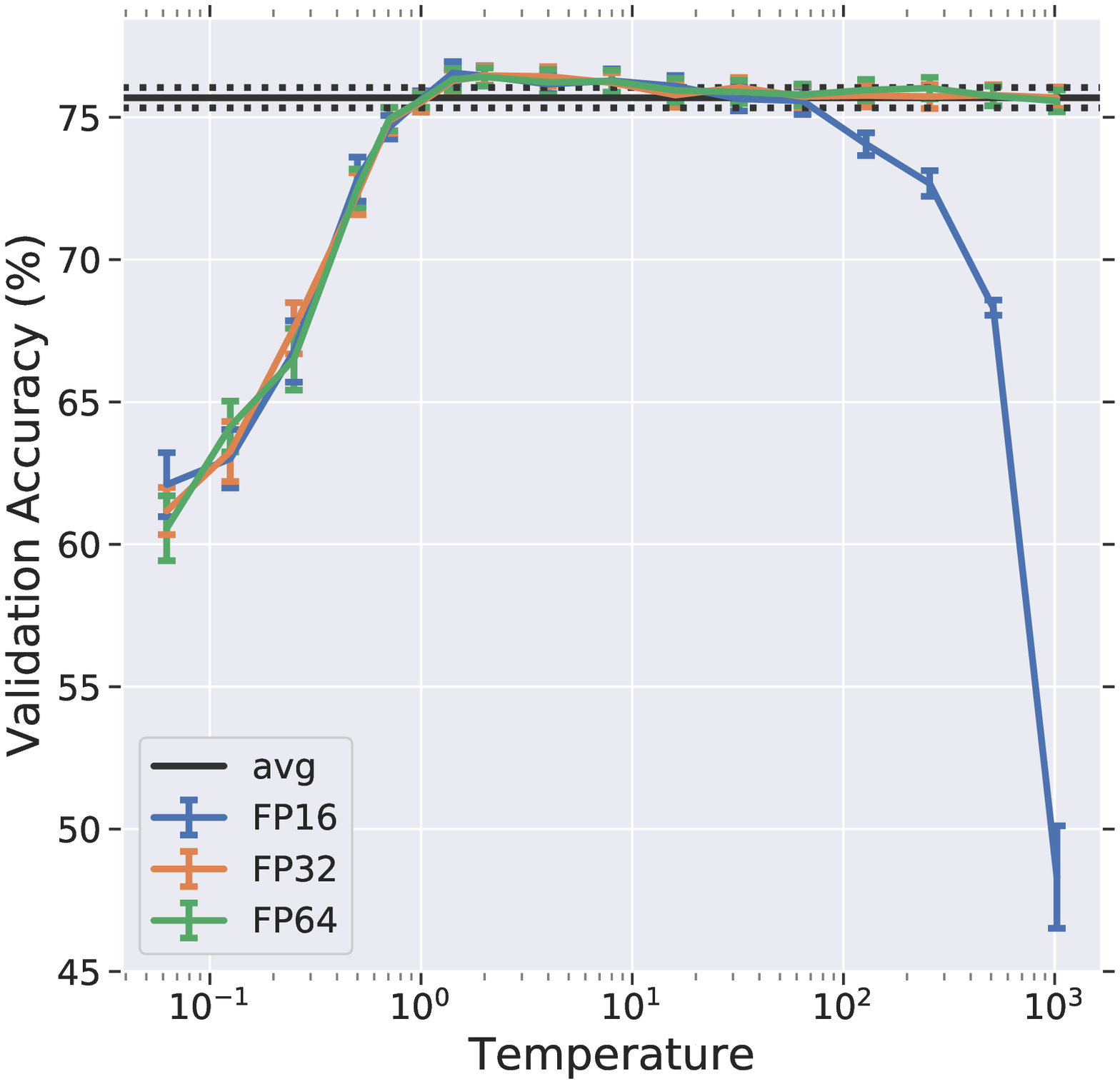}
    \caption{\label{fig:fp-cmp-imagewoof256}Imagewoof 256px}
\end{subfigure}
\caption{
Accuracy of networks with LAE global pooling, using 16-, 32-, or 64-bit floating point precision for the LAE operation.
Networks were trained for 6 epochs, with a range of different temperature values.
In all cases, the temperature was fixed and not a trainable parameter.
We used the optimizer parameters in each case, as discovered for average pooling in our hyperparameter search.
The performance of the network when using average pooling is indicated in black.
The mean and SEM over $n=10$ random 80/20 cross-validation folds of the training set are shown.
}
\label{fig:fp-cmp}
\end{figure}

We show results for LAE pool networks trained using the hyperparameters discovered for average pooling.
In accordance with \autoref{thm:lae-mean-conv}, the behaviour of the network with LAE pooling should converge toward the network with average pooling as $t\to\infty$.
This was observed to be true, provided the LAE pooling was performed with at least 32-bit floating point precision.

As shown in \autoref{fig:fp-cmp}, we found no discernable difference in the performance of the network when using 32-bit versus 64-bit floating point precision.
When using 16-bit precision, the performance of the network dropped considerably if the temperature exceeded $t=32$.
This is due to lost signal and increased noise in the gradient as it is propagated back through LAE, which is amplified due to the multiplicative nature of the temperature parameter.

We also found that networks with a trainable temperature parameter were better able to handle LAE using FP16 (not shown).
This is because the network learnt to use a lower temperature during training.
The effect was more apparent when using a single temperature, rather than a temperature parameter per channel, since in the former case all temperature updates are accumulated into the same parameter and it is hence able to adapt faster.

As there was no discernable difference in performance across the range of temperature parameters considered, even up to $t=1024$, we recommend using single precision (FP32) when using LAE pooling.

% ====================================================================
\section{Sensitivity to Input Resolution}

We explored the sensitivity of the trained network to the resolution of the input image, as it contracts or expands to be smaller or larger than the size of image on which the network was originally trained.
In the main paper, we only demonstrated results for networks trained on Imagewoof at 256px resolution.
Here we show the effect with Imagenette and Imagewoof at both 128px and 256px training resolution.
In all cases, the network is trained on the full training partition and evalutated on resized versions of images in the validation partition.

\begin{figure}[htbp]
\centering
\begin{subfigure}[b]{0.48\textwidth}
    \centering
    \includegraphics[width=\textwidth]{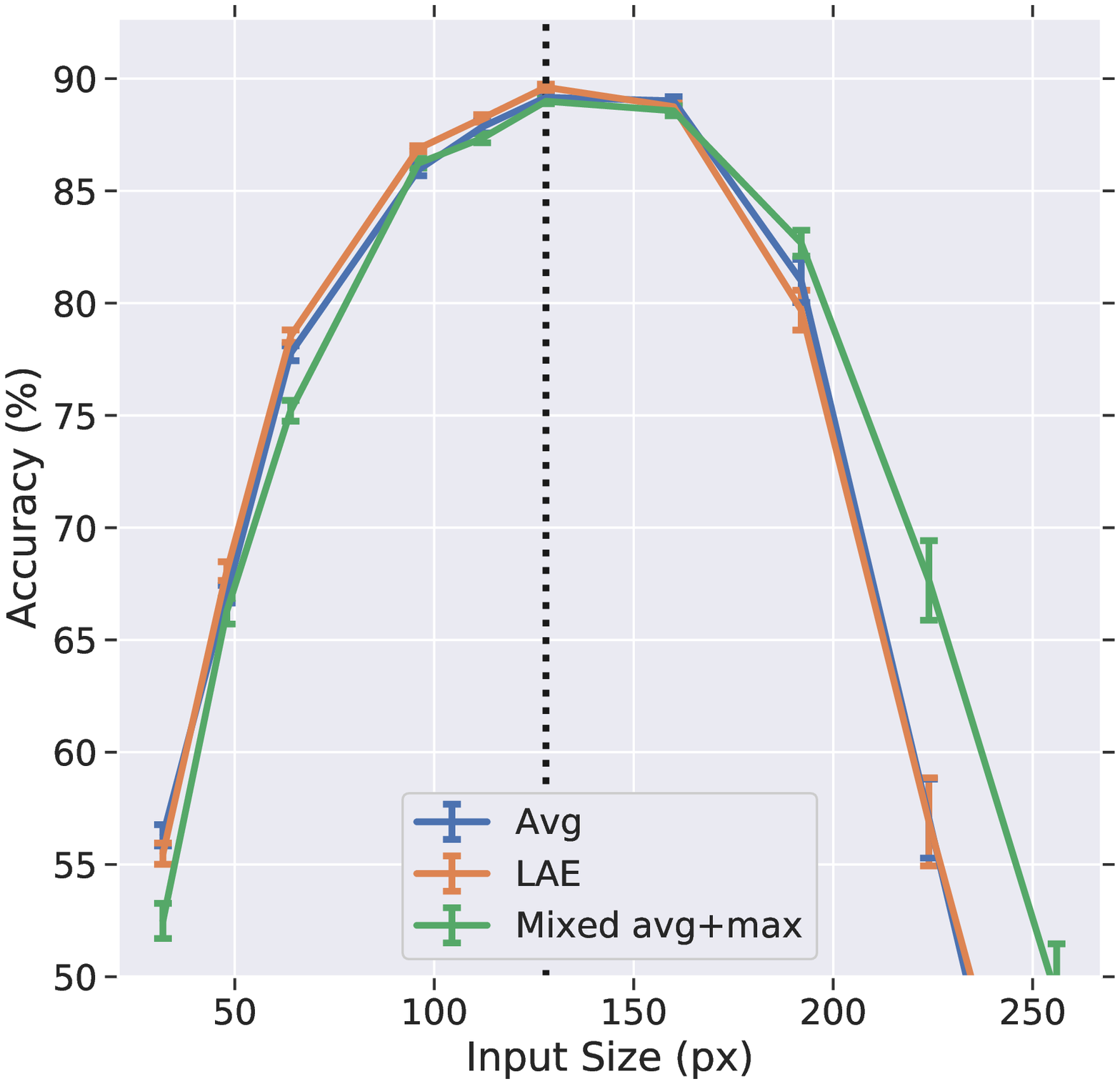}
    \caption{\label{fig:zoom-imagenette128}Imagenette 128px}
\end{subfigure}
~
\begin{subfigure}[b]{0.48\textwidth}
    \centering
    \includegraphics[width=\textwidth]{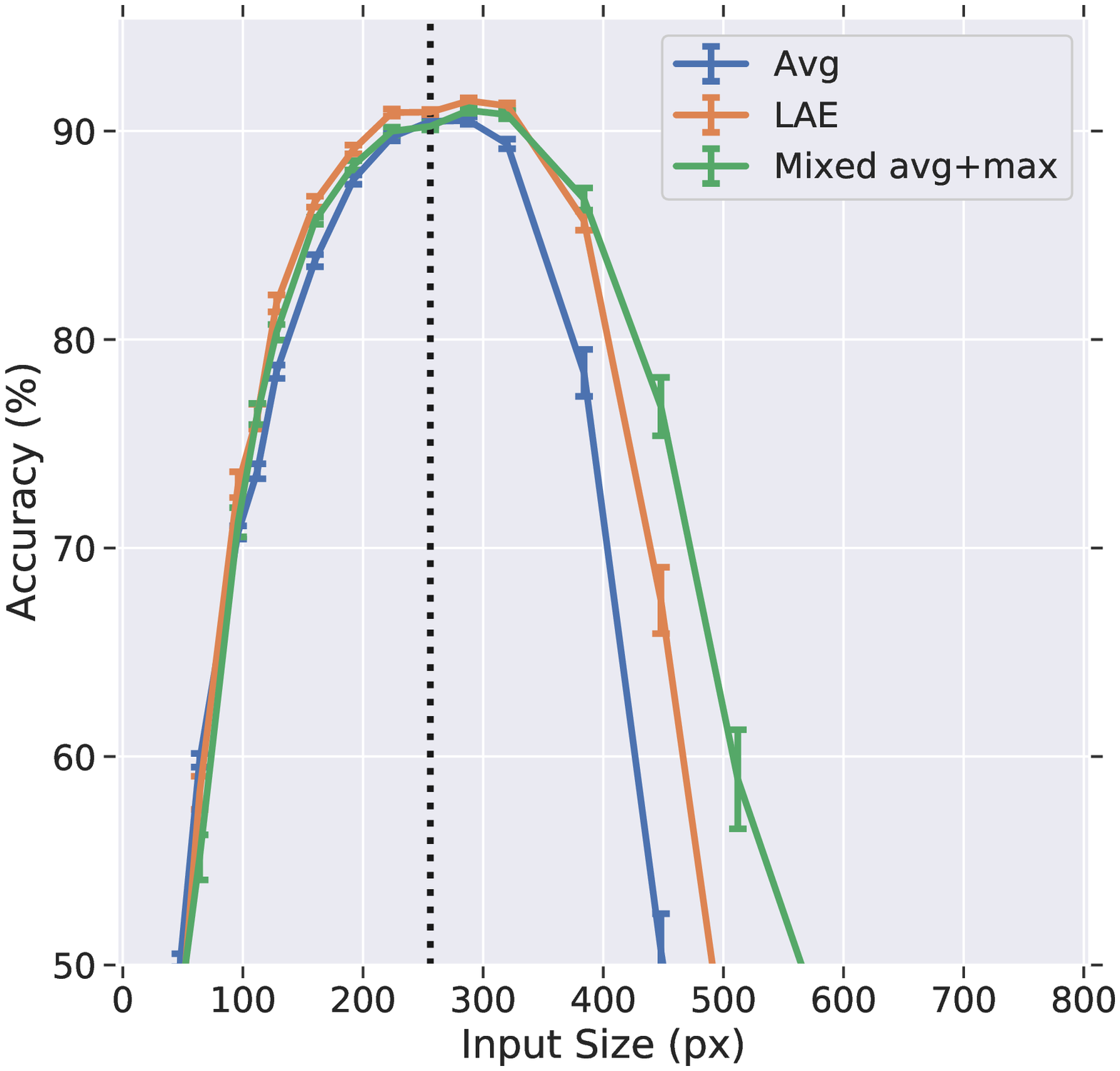}
    \caption{\label{fig:zoom-imagenette256}Imagenette 256px}
\end{subfigure}
\\
\begin{subfigure}[b]{0.48\textwidth}
    \centering
    \includegraphics[width=\textwidth]{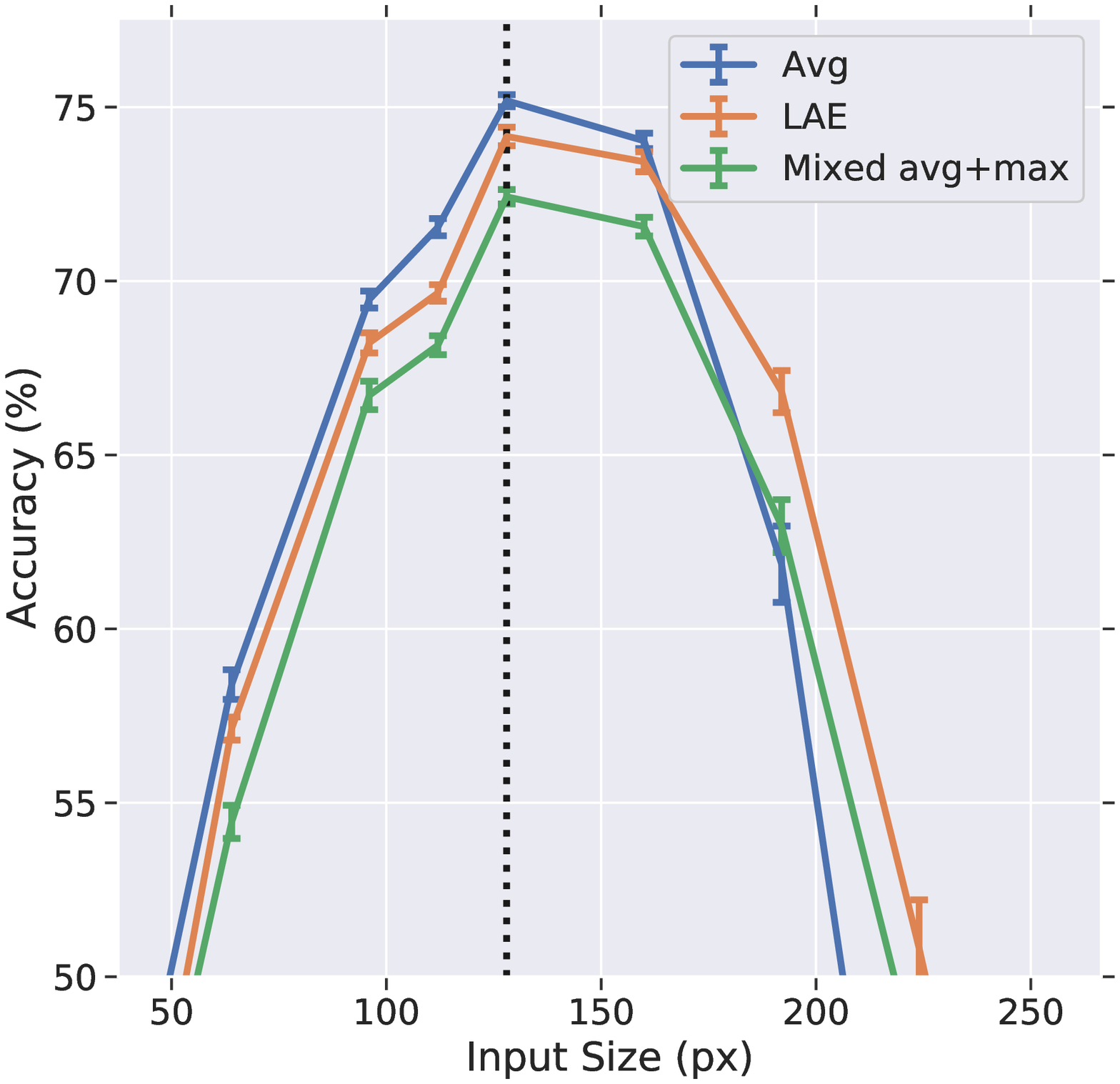}
    \caption{\label{fig:zoom-imagewoof128}Imagewoof 128px}
\end{subfigure}
~
\begin{subfigure}[b]{0.48\textwidth}
    \centering
    \includegraphics[width=\textwidth]{distort_figs/dis_zoom_results_sng_imagewoof_256_5ep.eps}
    \caption{\label{fig:zoom-imagewoof256}Imagewoof 256px}
\end{subfigure}
\caption{
The effect of changing the input resolution (by shrinking/stretching) on the validation performance of networks trained with different global pooling operations.
The mean and SEM over $18 \le n \le 28$ random seeds are shown.
}
\label{fig:zoom}
\end{figure}

\begin{figure}[htbp]
\centering
\begin{subfigure}[b]{0.48\textwidth}
    \centering
    \includegraphics[width=\textwidth]{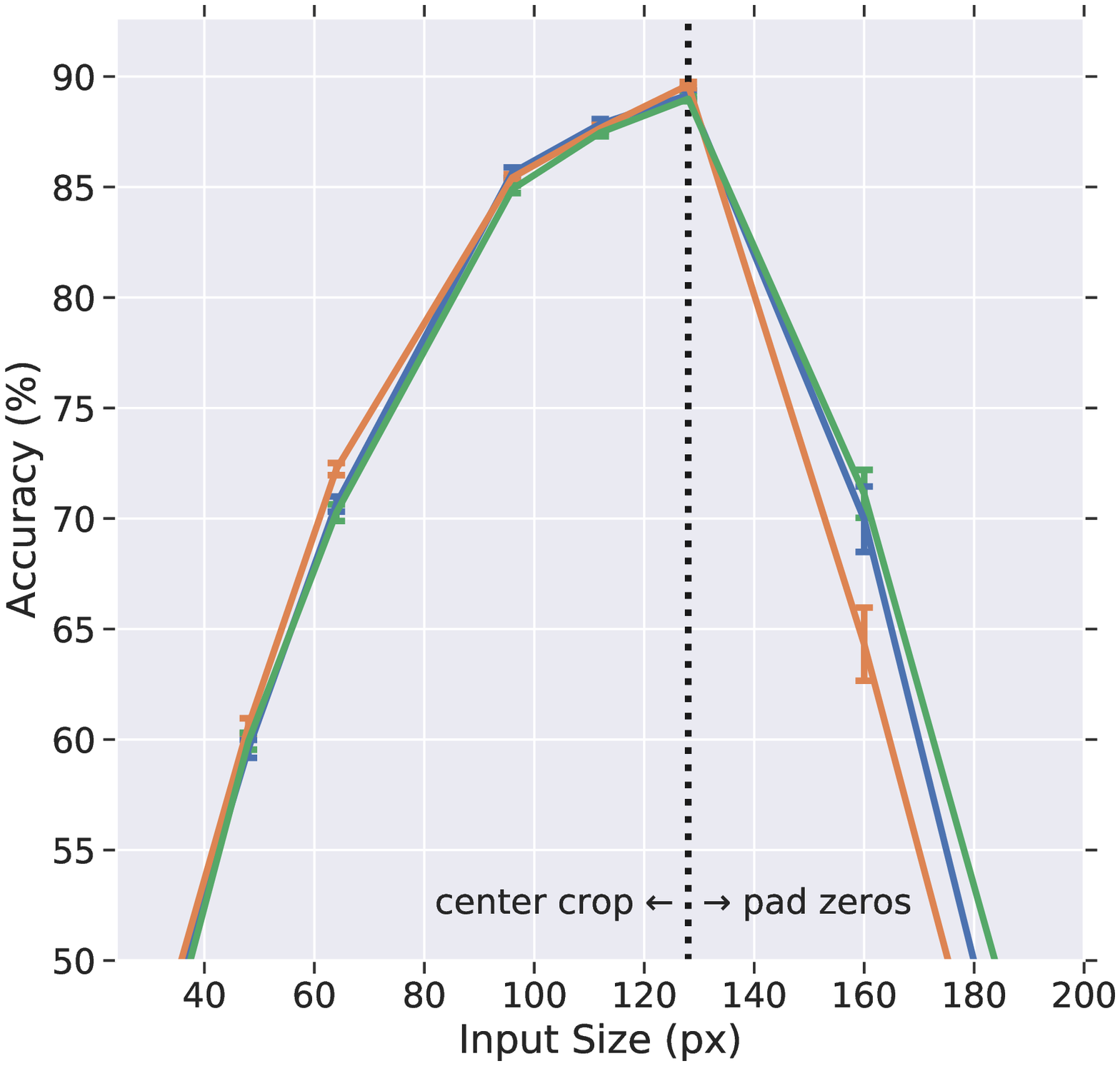}
    \caption{\label{fig:pad-zeros-imagenette128}Imagenette 128px}
\end{subfigure}
~
\begin{subfigure}[b]{0.48\textwidth}
    \centering
    \includegraphics[width=\textwidth]{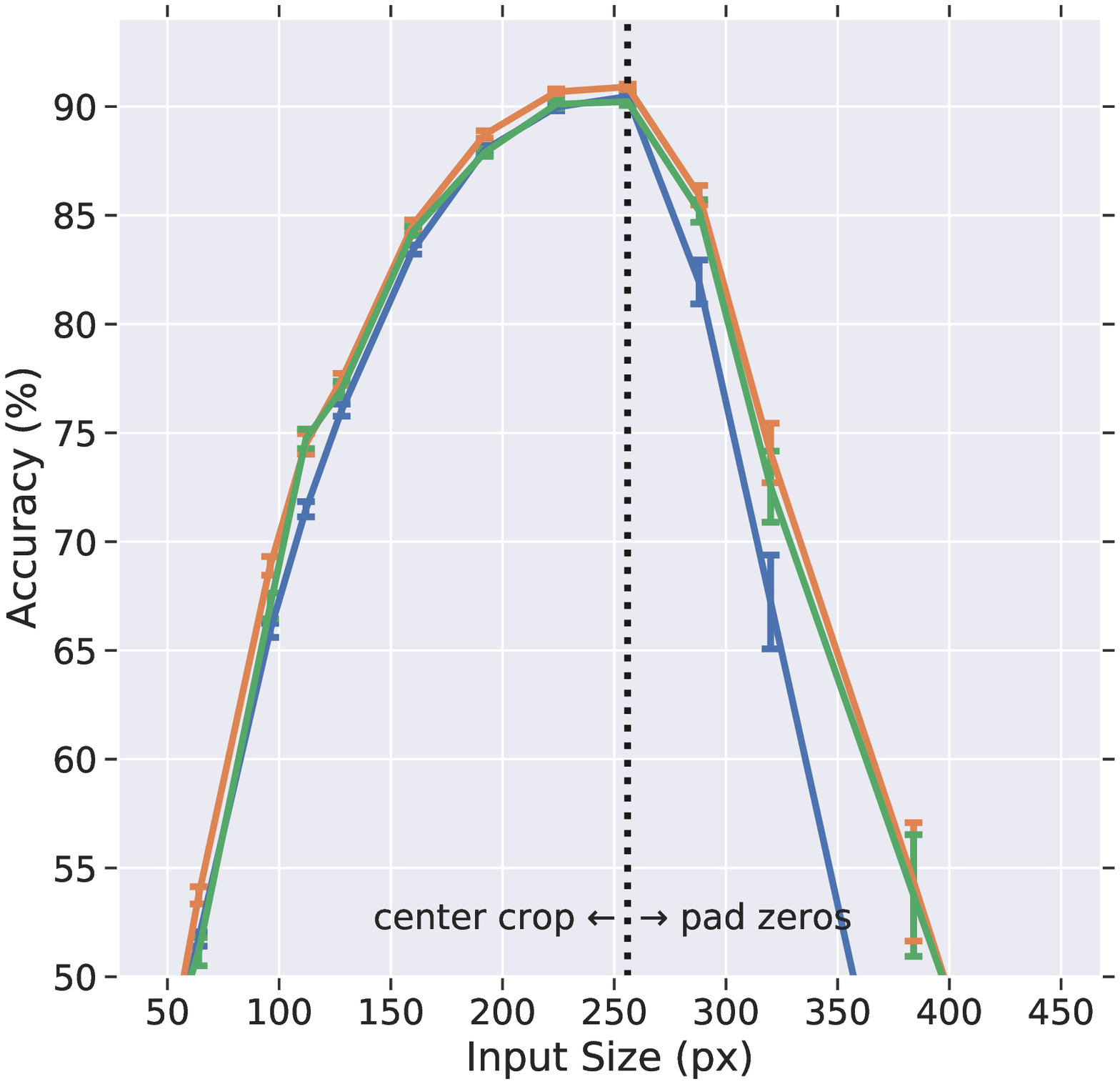}
    \caption{\label{fig:pad-zeros-imagenette256}Imagenette 256px}
\end{subfigure}
\\
\begin{subfigure}[b]{0.48\textwidth}
    \centering
    \includegraphics[width=\textwidth]{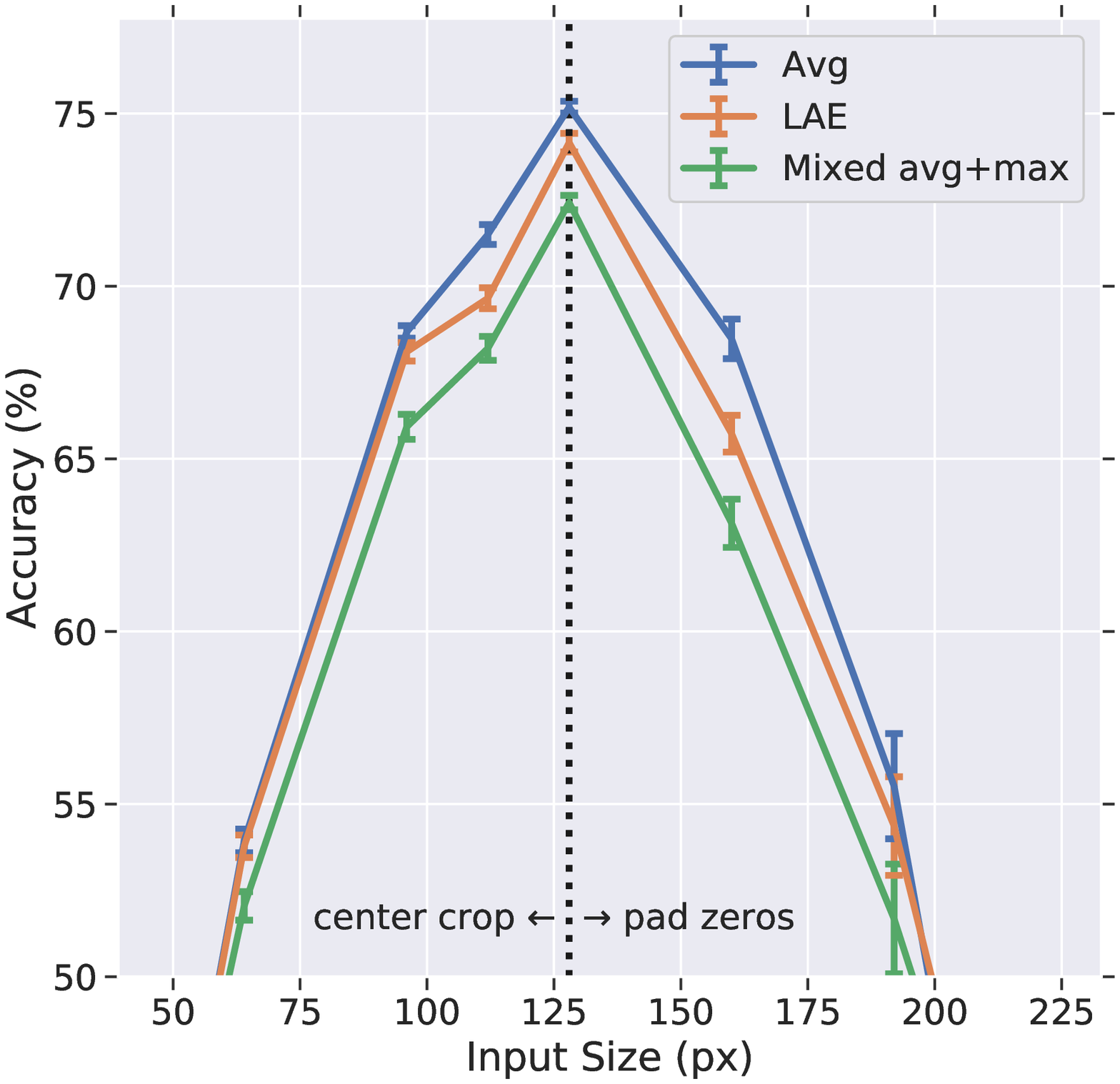}
    \caption{\label{fig:pad-zeros-imagewoof128}Imagewoof 128px}
\end{subfigure}
~
\begin{subfigure}[b]{0.48\textwidth}
    \centering
    \includegraphics[width=\textwidth]{distort_figs/dis_crop-pad-zeros_results_sng_imagewoof_256_5ep.eps}
    \caption{\label{fig:pad-zeros-imagewoof256}Imagewoof 256px}
\end{subfigure}
\caption{
The effect of changing the input resolution on the validation performance, by cropping the image or padding with zeros.
The mean and SEM over $18 \le n \le 28$ random seeds are shown.
Networks were trained for 5 epochs on the training set.
}
\label{fig:crop-pad-zeros}
\end{figure}

\begin{figure}[htbp]
\centering
\begin{subfigure}[b]{0.48\textwidth}
    \centering
    \includegraphics[width=\textwidth]{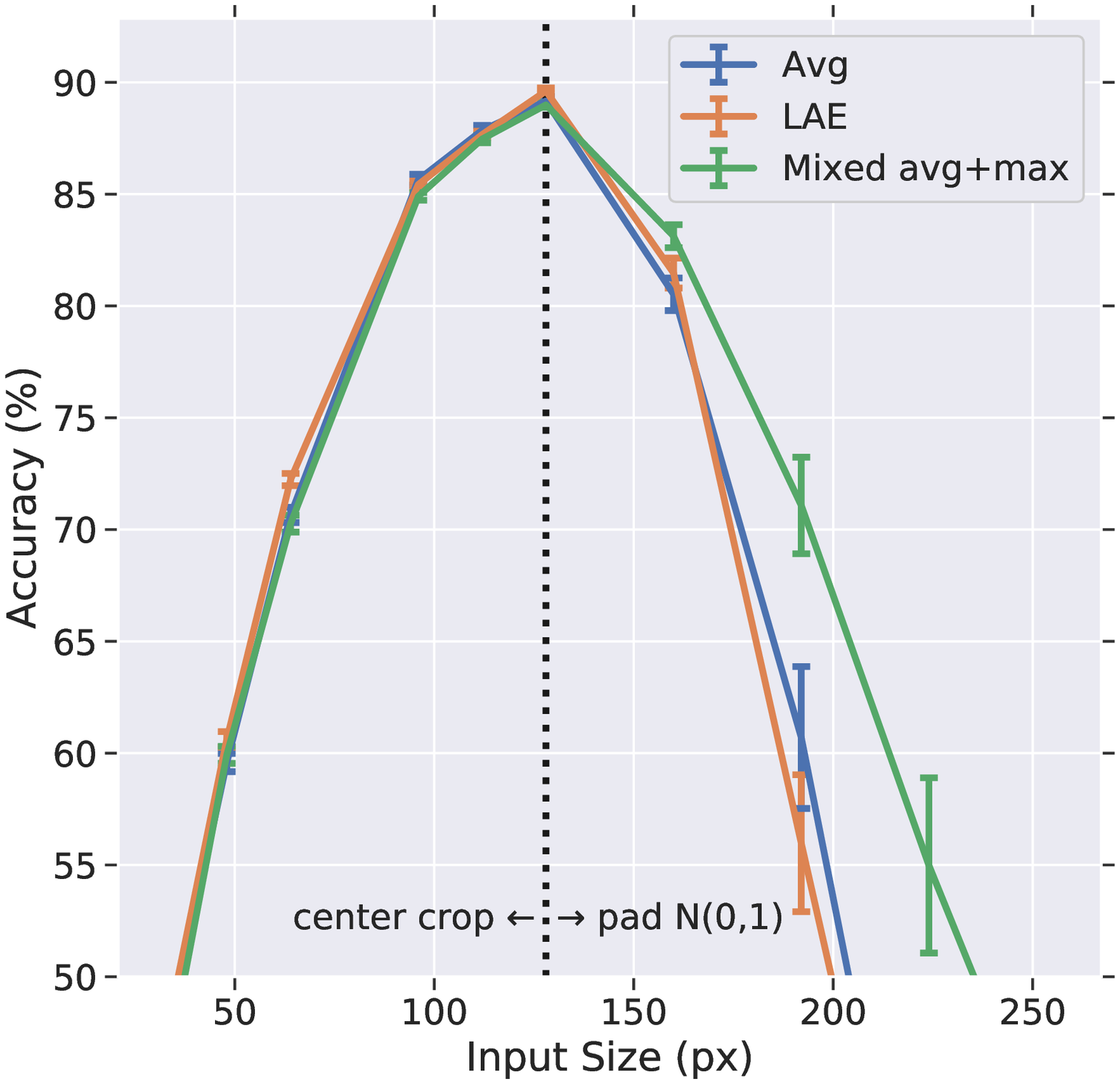}
    \caption{\label{fig:pad-norm-imagenette128}Imagenette 128px}
\end{subfigure}
~
\begin{subfigure}[b]{0.48\textwidth}
    \centering
    \includegraphics[width=\textwidth]{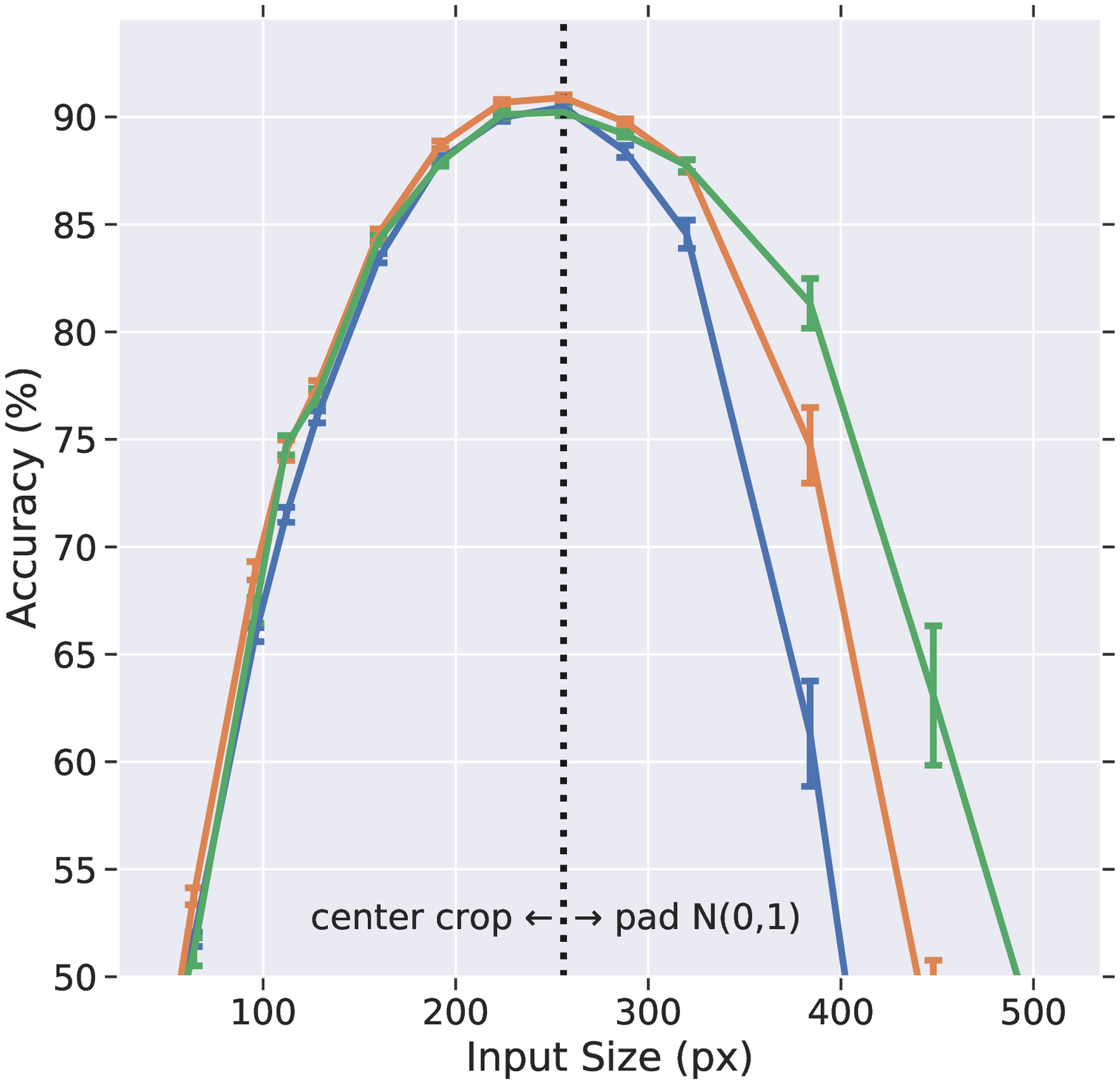}
    \caption{\label{fig:pad-norm-imagenette256}Imagenette 256px}
\end{subfigure}
\\
\begin{subfigure}[b]{0.48\textwidth}
    \centering
    \includegraphics[width=\textwidth]{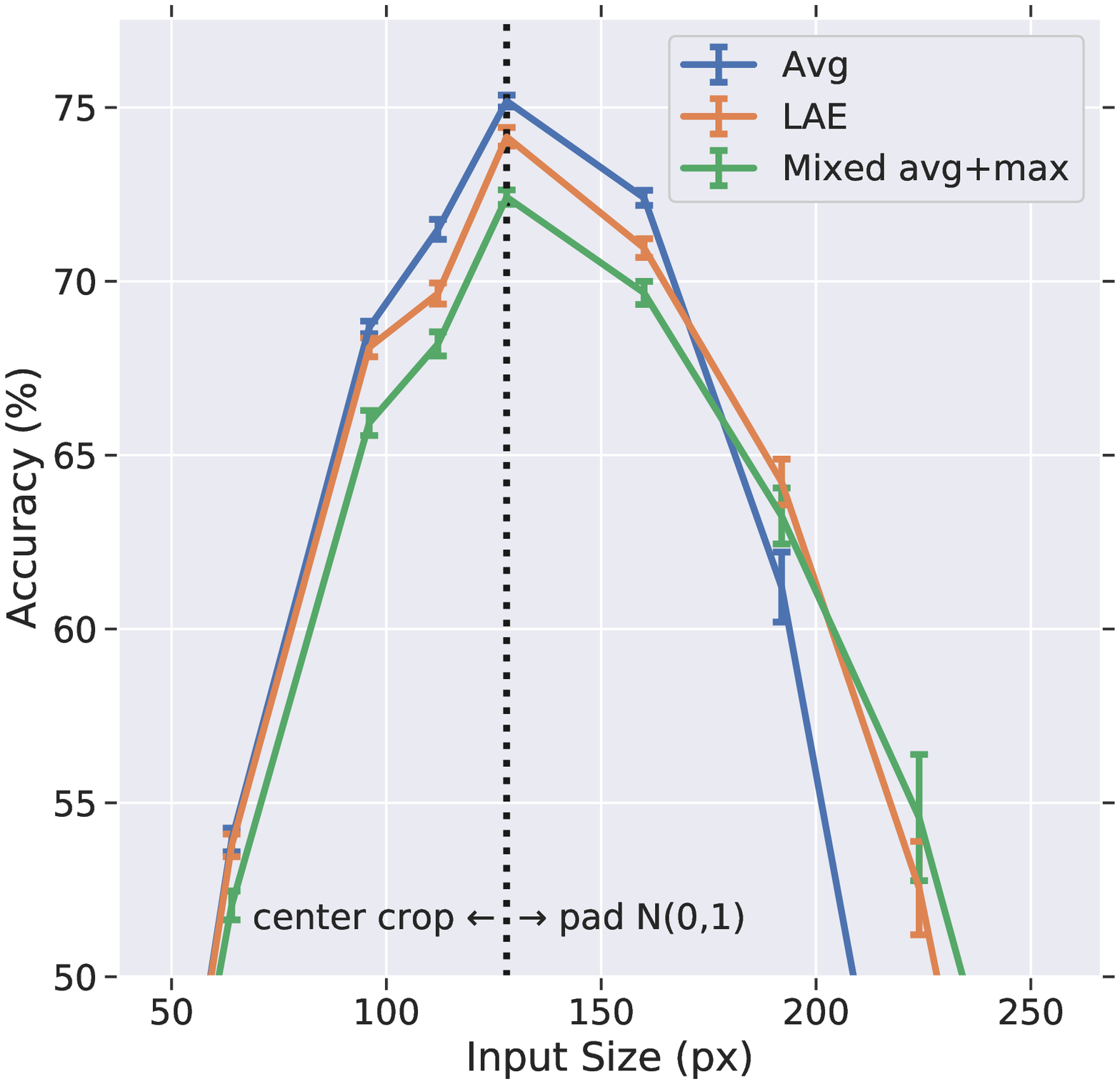}
    \caption{\label{fig:pad-norm-imagewoof128}Imagewoof 128px}
\end{subfigure}
~
\begin{subfigure}[b]{0.48\textwidth}
    \centering
    \includegraphics[width=\textwidth]{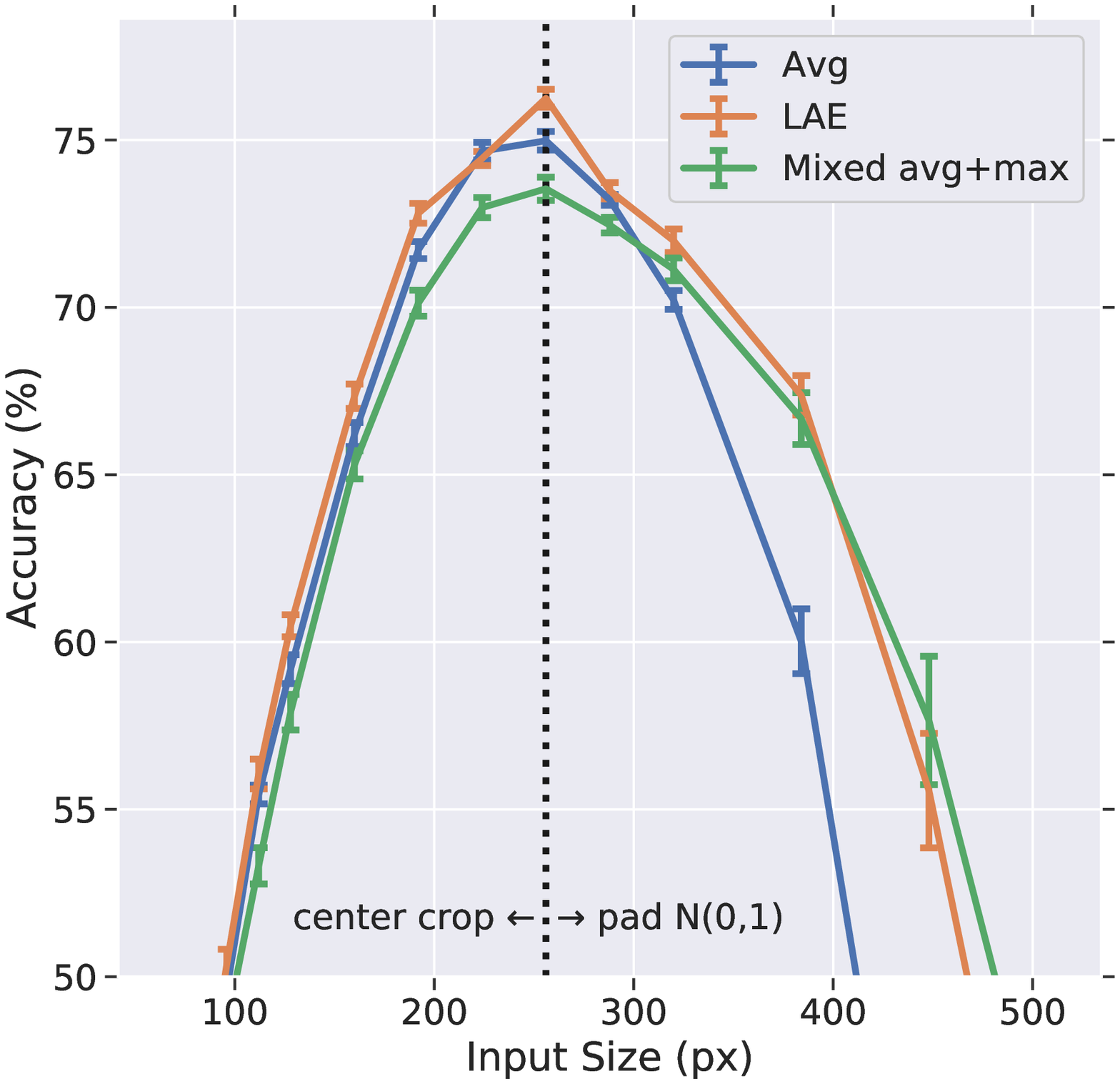}
    \caption{\label{fig:pad-norm-imagewoof256}Imagewoof 256px}
\end{subfigure}
\caption{
The effect of changing the input resolution on the validation performance, by cropping the image or padding with random values drawn independently from a standard normal distribution.
The mean and SEM over $18 \le n \le 28$ random seeds are shown.
Networks were trained for 5 epochs on the training set.
}
\label{fig:crop-pad-norm}
\end{figure}

For the datasets with larger image sizes, Imagnette 256px and Imagewoof 256px, we find that LAE pooling consistently provides high accuracy for a broader range of input sizes (Figures \ref{fig:zoom}, \ref{fig:crop-pad-zeros}, \ref{fig:crop-pad-norm}; right-hand panels).  % \autoref{fig:zoom}, \autoref{fig:crop-pad-zeros}, \autoref{fig:crop-pad-norm}
For Imagenette 256px, a trainable mixture of average and max pooling is performant for a broader range of inputs than LAE; for Imagewoof their relationship is reversed.
In both cases, each exceeds the performance of average pooling by a significant margin.

For Imagenette 128px (Figures \ref{fig:zoom-imagenette128}, \ref{fig:pad-zeros-imagenette128}, \ref{fig:pad-norm-imagenette128}) performance on untrained input sizes was similar for all three global pooling methods.

For Imagewoof 128px (Figures \ref{fig:zoom-imagewoof128}, \ref{fig:pad-zeros-imagewoof128}, \ref{fig:pad-norm-imagewoof128}), LAE with a fixed temperature of $t=1$ is less performant on the training image size than average pooling.
As noted in Table~2 of the main paper, using an initial temperature of $t_0=4$ is sufficient to mitigate differences between LAE and average pooling.
%And as noted in \autoref{sec:initial-temp}, the hyperparameters for average pooling are better than those discovered for LAE pooling.
Despite the performance at 128px being worse with LAE pooling, the network with LAE pooling outperforms average pooling when the size of the image differs from the training size by around a third.

% ====================================================================
\section{Discovered Hyperparameters}

In this section, we detail the final hyperparameters used for our analysis, as discovered during our hyperparameter optimization search.

\begin{table*}[htb]
\caption{
Discovered hyperparameters for PyramidNet+ShakeDrop on CIFAR-100.
}
\label{tab:hyper_shakepyramid}
\centering
\begin{tabular}{llrrrr}
\toprule
Global Pooling Op       &\, Dataset   &\, lr     &\, lr decay &\, wd      &\, mom     \\
\midrule
Average                 &\, CIFAR-100 &\, 0.200000 &\, 0.1316 &\, 0.00005 &\, 0.9 \\
LAE                     &\, CIFAR-100 &\, 0.021650 &\, 0.1316 &\, 0.00080 &\, 0.9 \\
LAE + train $t$ p. chn  &\, CIFAR-100 &\, 0.010825 &\, 0.1000 &\, 0.00160 &\, 0.9 \\ % 0.02165 & 0.1000 & 0.00080 & 0.9 \\
\bottomrule
\end{tabular}
\end{table*}

\begin{table*}[htb]
\caption{
Discovered hyperparameters for WRN-18-6 on CIFAR-10/100.
}
\label{tab:hyper_wrn-18-6}
\centering
\begin{tabular}{llrrr}
\toprule
Global Pooling Op                   &\, Dataset  &\, lr      &\, wd        &\, mom     \\
\midrule
Average                             &\, CIFAR-10  &\, 0.640 &\, 0.001024 &\, 0.3446 \\
                                    &\, CIFAR-100 &\, 6.711 &\, 0.000061 &\, 0.7440 \\
Mixed + train $\alpha$ p. chn       &\, CIFAR-10  &\, 0.250 &\, 0.000250 &\, 0.9375 \\
                                    &\, CIFAR-100 &\, 1.024 &\, 0.000640 &\, 0.7440 \\
Mixed + Gated $\alpha$ p. chn       &\, CIFAR-10  &\, 0.156 &\, 0.001638 &\, 0.7440 \\
                                    &\, CIFAR-100 &\, 0.400 &\, 0.001024 &\, 0.7440 \\
LAE                                 &\, CIFAR-10  &\, 0.640 &\, 0.001024 &\, 0.5904 \\
                                    &\, CIFAR-100 &\, 2.621 &\, 0.000640 &\, 0.3446 \\
LAE + train $t$                     &\, CIFAR-10  &\, 0.250 &\, 0.000640 &\, 0.8400 \\
                                    &\, CIFAR-100 &\, 0.156 &\, 0.002621 &\, 0.7440 \\
LAE + train $t$ p. chn              &\, CIFAR-10  &\, 0.640 &\, 0.000250 &\, 0.8400 \\
                                    &\, CIFAR-100 &\, 0.640 &\, 0.000400 &\, 0.8400 \\
\bottomrule
\end{tabular}
\end{table*}

\begin{table*}[htb]
\caption{
Discovered hyperparameters for WRN-18-6+SE on CIFAR-10/100.
}
\label{tab:hyper_wrn-18-6}
\centering
\begin{tabular}{lllrrr}
\toprule
Global Pooling Op       &\, SE Pooling Op             &\, Dataset   &\, lr       &\, wd        &\, mom     \\
\midrule
Average                 &\, Average                   &\, CIFAR-10  &\,  1.638 &\, 0.000038 &\, 0.9375 \\
                        &\,                           &\, CIFAR-100 &\,  0.640 &\, 0.001024 &\, 0.5904 \\
Average                 &\, LAE                       &\, CIFAR-10  &\,  0.640 &\, 0.000156 &\, 0.9000 \\
                        &\,                           &\, CIFAR-100 &\,  9.378 &\, 0.000156 &\, 0.0841 \\
Average                 &\, LAE + train $t$           &\, CIFAR-10  &\,  1.024 &\, 0.000156 &\, 0.9000 \\
                        &\,                           &\, CIFAR-100 &\,  6.711 &\, 0.000156 &\, 0.3446 \\
Average                 &\, LAE + train $t$ p. chn    &\, CIFAR-10  &\,  1.024 &\, 0.000098 &\, 0.8400 \\
                        &\,                           &\, CIFAR-100 &\, 12.855 &\, 0.000156 &\, 0.2154 \\
LAE                     &\, LAE                       &\, CIFAR-10  &\,  0.250 &\, 0.000400 &\, 0.9000 \\
                        &\,                           &\, CIFAR-100 &\,  2.621 &\, 0.000400 &\, 0.5904 \\
LAE + train $t$         &\, LAE + train $t$           &\, CIFAR-10  &\,  0.640 &\, 0.000400 &\, 0.7440 \\
                        &\,                           &\, CIFAR-100 &\,  0.640 &\, 0.001638 &\, 0.3446 \\
LAE + train $t$ p. chn  &\, LAE + train $t$ p. chn    &\, CIFAR-10  &\,  0.156 &\, 0.000250 &\, 0.9375 \\
                        &\,                           &\, CIFAR-100 &\,  2.621 &\, 0.000156 &\, 0.8400 \\
\bottomrule
\end{tabular}
\end{table*}

% \subsection{MXResNet on Imagenette and Imagewoof}

\begin{table*}[htb]
\caption{
Discovered hyperparameters for MXResNet on Imagenette and Imagewoof, for the 5 epochs 128 and 256 pixel resolution benchmarks.
}
\label{tab:hyper_imagenette}
\centering
\begin{tabular}{lllrrrrrr}
\toprule
Global Pooling Op            &\, Dataset& Size &\, lr      &\, ann     &\, wd      &\, mom     &\, alpha   &\, eps     \\
\midrule
Average                      &\, -nette & 128  &\, 0.00581 &\, 0.680   &\, 0.0086 &\, 0.892  &\, 0.9816  &\, 1.23e-07 \\
                             &\,        & 256  &\,{0.00718}&\,{0.636}  &\, 0.0052 &\, 0.931  &\, 0.9930  &\, 1.64e-07 \\
                             &\, -woof  & 128  &\, 0.00436 &\, 0.702   &\, 0.0104 &\, 0.906  &\, 0.9921  &\, 6.23e-07 \\
                             &\,        & 256  &\, 0.00551 &\, 0.678   &\, 0.0070 &\, 0.912  &\, 0.9929  &\, 7.10e-07 \\
Mixed+train $\alpha$ p. chn  &\, -nette & 128  &\, 0.00669 &\, 0.691   &\, 0.0264 &\, 0.949  &\, 0.9898  &\,{8.75e-08}\\
                             &\,        & 256  &\, 0.00386 &\, 0.679   &\, 0.0037 &\, 0.941  &\, 0.9787  &\, 9.06e-07 \\
                             &\, -woof  & 128  &\,{0.00282}&\, 0.661   &\, 0.0203 &\, 0.949  &\, 0.9856  &\, 1.44e-06 \\
                             &\,        & 256  &\, 0.00324 &\, 0.743   &\, 0.0116 &\, 0.931  &\, 0.9916  &\, 5.56e-07 \\
LAE                          &\, -nette & 128  &\, 0.00348 &\, 0.667   &\,{0.0026}&\,{0.869} &\, 0.9800  &\, 3.30e-06 \\
                             &\,        & 256  &\, 0.00660 &\, 0.665   &\, 0.0047 &\, 0.899  &\, 0.9908  &\, 2.04e-06 \\
                             &\, -woof  & 128  &\, 0.00365 &\, 0.741   &\,{0.0350}&\, 0.874  &\, 0.9939  &\, 2.84e-06 \\
                             &\,        & 256  &\, 0.00460 &\, 0.786   &\, 0.0062 &\, 0.929  &\, 0.9937  &\, 1.31e-06 \\
LAE + train $t$              &\, -nette & 128  &\, 0.00464 &\, 0.699   &\, 0.0119 &\, 0.897  &\,{0.9687} &\, 6.89e-07 \\
                             &\,        & 256  &\, 0.00486 &\, 0.664   &\, 0.0038 &\, 0.909  &\, 0.9782  &\, 2.92e-06 \\
                             &\, -woof  & 128  &\, 0.00501 &\, 0.689   &\, 0.0140 &\, 0.933  &\,{0.9957} &\, 6.19e-06 \\
                             &\,        & 256  &\, 0.00357 &\, 0.642   &\, 0.0049 &\, 0.919  &\, 0.9899  &\, 9.84e-07 \\
LAE + temp p. chn            &\, -nette & 128  &\, 0.00319 &\,{0.823}  &\, 0.0205 &\,{0.952} &\, 0.9708  &\, 1.38e-07 \\
                             &\,        & 256  &\,{0.00718}&\, 0.676   &\, 0.0143 &\, 0.909  &\, 0.9955  &\,{1.29e-06}\\
                             &\, -woof  & 128  &\, 0.00357 &\, 0.720   &\, 0.0212 &\, 0.910  &\, 0.9848  &\, 8.50e-07 \\
                             &\,        & 256  &\, 0.00467 &\, 0.750   &\, 0.0053 &\, 0.905  &\, 0.9940  &\, 1.84e-06 \\
\bottomrule
\end{tabular}
\end{table*}

% ====================================================================
\FloatBarrier
\section{Proofs of limits of temperature-mediated LogAvgExp}

In the main body of the paper, we assert that
\begin{align}
\lim_{t \to 0^+} \opn{LogAvgExp}(\VEC{z}; t) =& \max(\VEC{z}) = \max_{\forall i} z_i \\
\lim_{t \to +\infty} \opn{LogAvgExp}(\VEC{z}; t) =& \opn{mean}(\VEC{z}) = \frac{1}{n} \sum_{i=1}^n z_i
.\end{align}
Here, we provide proofs for each of these statements.

\begin{theorem}
$\opn{LogAvgExp}(\VEC{z}; t)$, where $t > 0$, $\VEC{z} \in \R^n$, is bounded above by $\max(\VEC{z})$.
\end{theorem}

\begin{proof}
Recall that
\begin{align}
\opn{LogAvgExp}(\VEC{z}; t)
=& \, t \cdot \opn{LogAvgExp}\left(\frac{\VEC{z}}{t}\right) \\
=& \, t \cdot \log \left( \frac{1}{n} \sum_{i=1}^n \exp\left(\frac{z_i}{t}\right) \right) \\
=& \, t \left( \log \left( \sum_{i=1}^n \exp\left(\frac{z_i}{t}\right) \right) - \log(n) \right) \label{eq:logavgexpt-a}
.\end{align}

Let $z_* := \max_i z_i$.
Noting that $t>0$, we can use the ``LogSumExp-trick'' as follows
\begin{align}
\opn{LogAvgExp}(\VEC{z}; t)
=& \, t \left( \log \left( \sum_{i=1}^n \exp\left(\frac{z_i}{t}\right) \right) - \log(n) \right) \\
=& \, t \left( \log \left( \sum_{i=1}^n \exp\left(\frac{z_i}{t} - \frac{z_*}{t} + \frac{z_*}{t} \right) \right) - \log(n) \right) \\
=& \, t \left( \log \left( \exp\left(\frac{z_*}{t}\right) \sum_{i=1}^n \exp\left(\frac{z_i - z_*}{t} \right) \right) - \log(n) \right) \\
=& \, t \left( \frac{z_*}{t} + \log \left( \sum_{i=1}^n \exp\left(\frac{z_i - z_*}{t}\right) \right) - \log(n) \right) \\
=& z_* + \, t \left( \log \left( \sum_{i=1}^n \exp\left(\frac{z_i - z_*}{t}\right) \right) - \log(n) \right) \label{eq:lse-trick1}
.\end{align}

Note that as $z_i \leq z_* \,\forall\, i$, and $t > 0$, so
\begin{align}
\implies \frac{z_i - z_*}{t} \leq& 0 \quad\quad \forall\,i
.\end{align}
Using the fact that the exponential function is monotonic and strictly increasing,
\begin{align}
\iff \exp\left(\frac{z_i - z_*}{t}\right) \leq& 1 \quad\quad \forall\,i \\
\implies \sum_{i=1}^n \exp\left(\frac{z_i - z_*}{t}\right) \leq& n
.\end{align}
Using the fact that the log function is monotonic and strictly increasing,
\begin{align}
\iff \log \left( \sum_{i=1}^n \exp\left(\frac{z_i - z_*}{t}\right) \right) \leq& \log(n) \\
\iff \log \left( \sum_{i=1}^n \exp\left(\frac{z_i - z_*}{t}\right) \right) - \log(n) \leq& 0 \label{eq:bounded-logn}
.\end{align}
Since $t > 0$,
\begin{align}
\implies t \left( \log \left( \sum_{i=1}^n \exp\left(\frac{z_i - z_*}{t}\right) \right) - \log(n) \right) \leq& 0 \\
\iff z_* + \, t \left( \log \left( \sum_{i=1}^n \exp\left(\frac{z_i - z_*}{t}\right) \right) - \log(n) \right) \leq& z_* \\
\iff \opn{LogAvgExp}(\VEC{z}; t) \leq& z_*
,\end{align}
where we have made use of \autoref{eq:lse-trick1} in the final step.

\end{proof}

\begin{corollary}
$\opn{LogAvgExp}(\VEC{z}; t)$ converges to its upper bound of $\max(\VEC{z})$ in the limiting case of $t \to 0^+$.
\end{corollary}

\begin{proof}
Using the fact that $\exp(x) > 0 \,\forall\, x$, and \autoref{eq:bounded-logn},
\begin{equation}
0 < \log \left( \sum_{i=1}^n \exp\left(\frac{z_i - z_*}{t}\right) \right) \leq \log(n)
,\end{equation}
we observe the middle term is bounded above and below, and hence must be finite $\forall t$.

Thus from \autoref{eq:lse-trick1},
\begin{align}
\lim_{t \to 0^+} \opn{LogAvgExp}(\VEC{z}; t)
&= \lim_{t \to 0^+} z_* + \, t \left( \log \left( \sum_{i=1}^n \exp\left(\frac{z_i - z_*}{t}\right) \right) - \log(n) \right) \\
&= z_*
.\end{align}
\end{proof}

\begin{theorem}
\label{thm:lae-mean-conv}
$\opn{LogAvgExp}(\VEC{z}; t)$, where $t > 0$, $\VEC{z} \in \R^n$, converges to $\opn{mean}(\VEC{z})$ in the limiting case of $t \to +\infty$.
\end{theorem}

\begin{proof}
Recall that
\begin{align}
\opn{LogAvgExp}(\VEC{z}; t)
    =& \, t \cdot \opn{LogAvgExp}\left(\frac{\VEC{z}}{t}\right) \\
    =& \, t \cdot \log \left( \frac{1}{n} \sum_{i=1}^n \exp\left(\frac{z_i}{t}\right) \right), \label{eq:logavgexpt}
\end{align}
by definition of $\opn{LogAvgExp}$.

Also, recall that the Taylor series expansion for $\exp(x)$ is given by
\begin{align}
\exp(x) &= \sum_{k=0}^\infty \frac{x^k}{k!} \label{eq:exp-taylor} \\
        &= 1 + x + \frac{x^2}{2} + \frac{x^3}{3!} + \cdots
.\end{align}

Substituting \autoref{eq:exp-taylor} into \autoref{eq:logavgexpt},
\begin{align}
\opn{LogAvgExp}(\VEC{z}; t)
&= \, t \cdot \log \left( \frac{1}{n} \sum_{i=1}^n \sum_{k=0}^\infty \frac{z_i^k}{k! \, t^k} \right) \\
&= \, t \cdot \log \left( \frac{1}{n} \sum_{i=1}^n \left( 1 + \sum_{k=1}^\infty \frac{z_i^k}{k! \, t^k} \right) \right) \\
&= \, t \cdot \log \left( 1 + \frac{1}{n} \sum_{i=1}^n \sum_{k=1}^\infty \frac{z_i^k}{k! \, t^k} \right)
\label{eq:lae-taylor}
.\end{align}

Recall that the Taylor series expansion for $\log(1+x)$ is
\begin{align}
\log(1+x) &= \sum_{k=1}^\infty \frac{(-1)^{k+1}\,x^k}{k} \label{eq:log-taylor} \\
          &= x - \frac{x^2}{2} + \frac{x^3}{3} - \frac{x^4}{4} + \cdots
,\end{align}
which converges for $-1 < x \leq 1$.

Since $\exp(x) > 0 \, \forall \, x \in \R$,
\begin{align}
\frac{1}{n} \sum_{i=1}^n \exp\left(\frac{z_i}{t}\right) &> 0 \\
\implies \frac{1}{n} \sum_{i=1}^n \sum_{k=1}^\infty \frac{z_i^k}{k! \, t^k} &> -1
,\end{align}
where we have substituted in \autoref{eq:exp-taylor} again.
This demonstrates we satisfy the lower bound for convergence of \autoref{eq:log-taylor}.

If we let $z_* := \max_i z_i$, we can write
\begin{align}
\frac{1}{n} \sum_{i=1}^n \sum_{k=1}^\infty \frac{z_i^k}{k! \, t^k}
&= -1 + \frac{1}{n} \sum_{i=1}^n \exp\left(\frac{z_i}{t}\right) \\
&\leq -1 + \frac{1}{n} \sum_{i=1}^n \exp\left(\frac{z_*}{t}\right) \\
&= -1 + \exp\left(\frac{z_*}{t}\right)
,\end{align}
where we have made use of the fact that the exponential function is strictly monotonically increasing.

Let us choose some $t>0$ sufficiently large such that
\begin{align}
\frac{z_*}{\log(2)} & \leq t  \\
\implies \frac{z_*}{t} & \leq \log(2) \\
\implies \exp\left(\frac{z_*}{t}\right) & \leq 2 \\
\implies -1 + \exp\left(\frac{z_*}{t}\right) &\leq 1 \\
\implies \frac{1}{n} \sum_{i=1}^n \sum_{k=1}^\infty \frac{z_i^k}{k! \, t^k}  &\leq 1
.\end{align}
This proves the upper bound of the requirement for convergence of \autoref{eq:log-taylor}.

As $ -1 < \frac{1}{n} \sum_{i=1}^n \sum_{k=1}^\infty \frac{z_i^k}{k! \, t^k} \leq 1$ for sufficiently large values of $t$, we can apply the Taylor expansion in \autoref{eq:log-taylor} to \autoref{eq:lae-taylor}.
\begin{align}
\opn{LogAvgExp}(\VEC{z}; t)
=& t \, \sum_{m=1}^\infty \frac{(-1)^{m+1}}{m} \left( \frac{1}{n} \sum_{i=1}^n \sum_{k=1}^\infty \frac{z_i^k}{k! \, t^k} \right)^m \\
=& t \left( \frac{1}{n} \sum_{i=1}^n \frac{z_i}{t} + \mathcal{O}\left(\nicefrac{1}{t^2}\right) \right) \\
=& \frac{1}{n} \sum_{i=1}^n z_i + \mathcal{O}\left(\nicefrac{1}{t}\right) \\
\xrightarrow{t\to\infty}& \frac{1}{n} \sum_{i=1}^n z_i
\end{align}
\end{proof}

\section{Derivatives of LogAvgExp}

\begin{theorem}
The derivative of $\opn{LogAvgExp}$ with respect to each input element, $z_i$, is given by
\begin{equation}
\pd{z_i} \opn{LogAvgExp}(\VEC{z}; t) = \frac{\exp(\frac{z_i}{t})}{\sum_{j=1}^n \exp(\frac{z_j}{t})}
.\end{equation}
\end{theorem}

\begin{proof}
Recall that
\begin{align}
\opn{LogAvgExp}(\VEC{z}; t)
&= \, t \cdot \opn{LogAvgExp}\left(\frac{\VEC{z}}{t}\right) \\
&= \, t \, \left( \opn{LSE}\left(\frac{\VEC{z}}{t}\right) - \log(n) \right)
\end{align}
by definition, where
\begin{equation}
\opn{LSE}\left(\VEC{z}\right)
= \opn{LogSumExp}\left(\VEC{z}\right)
= \log \left( \sum_{i=1}^n \exp\left(z_i\right) \right)
\label{eq:lse}
.\end{equation}

\begin{align}
\pd{z_i}\opn{LogAvgExp}(\VEC{z}; t)
&= \pd{z_i} \, t \, \left( \opn{LSE}\left(\frac{\VEC{z}}{t}\right) - \log(n) \right) \\
&= t \, \pd{z_i} \opn{LSE}\left(\frac{\VEC{z}}{t}\right) \\
&= t \, \pd{z_i} \log\left(\sum_{j=1}^n \exp(\frac{z_j}{t})\right) \\
&= t \, \frac{1}{\sum_{j=1}^n \exp(\frac{z_j}{t})} \pd{z_i} \sum_{j=1}^n \exp(\frac{z_j}{t}) \\
&= \frac{t}{\sum_{j=1}^n \exp(\frac{z_j}{t})} \exp(\frac{z_i}{t}) \pd{z_i} \frac{z_i}{t} \\
&= \frac{\exp(\frac{z_i}{t})}{\sum_{j=1}^n \exp(\frac{z_j}{t})}
\end{align}
\end{proof}

%\subsubsection{Derivation of derivative of LogAvgExp with respect to temperature}

\begin{theorem}
The derivative of $\opn{LogAvgExp}$ with respect to temperature, $t$, is given by
\begin{equation}
\pd{t} \opn{LogAvgExp}(\VEC{z}; t) = \frac{1}{t} \left( \opn{LogAvgExp}\left(\VEC{z}; t\right) - \frac{\sum_{i=1}^n {z_i} \exp\left(\frac{z_i}{t}\right)}{\sum_{i=1}^n \exp\left(\frac{z_i}{t}\right)} \right)
.\end{equation}
\end{theorem}

\begin{proof}
Recall that
\begin{align}
\opn{LogAvgExp}(\VEC{z}; t)
&= \, t \cdot \opn{LogAvgExp}\left(\frac{\VEC{z}}{t}\right) \\
&= \, t \, \left( \opn{LSE}\left(\frac{\VEC{z}}{t}\right) - \log(n) \right) \label{eq:lae-temp}
\end{align}
by definition.

\begin{align}
\pd{t} \opn{LogAvgExp}(\VEC{z}; t)
&= \pd{t} \, t \left( \opn{LSE}\left(\frac{\VEC{z}}{t}\right) - \log(n) \right) \\
&= \opn{LSE}\left(\frac{\VEC{z}}{t}\right) - \log(n) + t \pd{t} \opn{LSE}\left(\frac{\VEC{z}}{t}\right)
\end{align}

\begin{align}
\opn{LSE}\left(\frac{\VEC{z}}{t}\right) &= \log\left(\sum_{i=1}^n \exp\left(\frac{z_i}{t}\right) \right) \\
\implies \pd{t} \opn{LSE}\left(\frac{\VEC{z}}{t}\right) &= \pd{t} \log\left(\sum_{i=1}^n \exp\left(\frac{z_i}{t}\right) \right) \\
&= \frac{1}{\exp\left(\opn{LSE}\left(\nicefrac{\VEC{z}}{t}\right)\right)} \pd{t} \sum_{i=1}^n \exp\left(\frac{z_i}{t}\right) \\
&= \frac{1}{\exp\left(\opn{LSE}\left(\nicefrac{\VEC{z}}{t}\right)\right)} \sum_{i=1}^n \pd{t} \exp\left(\frac{z_i}{t}\right) \\
&= \frac{1}{\exp\left(\opn{LSE}\left(\nicefrac{\VEC{z}}{t}\right)\right)} \sum_{i=1}^n -\frac{z_i}{t^2} \exp\left(\frac{z_i}{t}\right) \\
&= \frac{-1}{t^2} \frac{\sum_{i=1}^n z_i \exp\left(\frac{z_i}{t}\right)}{\exp\left(\opn{LSE}\left(\nicefrac{\VEC{z}}{t}\right)\right)} \\
&= \frac{-1}{t^2} \frac{\sum_{i=1}^n z_i \exp\left(\frac{z_i}{t}\right)}{\sum_{i=1}^n \exp\left(\frac{z_i}{t}\right)}
\end{align}

Consequently, we conclude that
\begin{align}
\pd{t} \opn{LogAvgExp}(\VEC{z}; t)
&= \opn{LSE}\left(\frac{\VEC{z}}{t}\right) - \log(n) - \frac{\sum_{i=1}^n {z_i} \exp\left(\frac{z_i}{t}\right)}{t \exp\left(\opn{LSE}\left(\frac{\VEC{z}}{t}\right)\right)} \\
&= \opn{LSE}\left(\frac{\VEC{z}}{t}\right) - \log(n) - \frac{\sum_{i=1}^n {z_i} \exp\left(\frac{z_i}{t}\right)}{t \sum_{i=1}^n \exp\left(\frac{z_i}{t}\right)} \\
&= \frac{1}{t} \left( \opn{LogAvgExp}\left(\VEC{z}; t\right) - \frac{\sum_{i=1}^n {z_i} \exp\left(\frac{z_i}{t}\right)}{\sum_{i=1}^n \exp\left(\frac{z_i}{t}\right)} \right)
.\end{align}
\end{proof}

%% file: main.bbl
\begin{thebibliography}{10}
\providecommand{\url}[1]{\texttt{#1}}
\providecommand{\urlprefix}{URL }
\providecommand{\doi}[1]{https://doi.org/#1}

\bibitem{boureau2010theoretical}
Boureau, Y.L., Ponce, J., LeCun, Y.: A theoretical analysis of feature pooling
  in visual recognition. In: Proceedings of the 27th international conference
  on machine learning (ICML-10). pp. 111--118 (2010)

\bibitem{cohen2016deep}
Cohen, N., Sharir, O., Shashua, A.: Deep simnets. In: Proceedings of the IEEE
  Conference on Computer Vision and Pattern Recognition. pp. 4782--4791 (2016)

\bibitem{autoaugment}
Cubuk, E.D., Zoph, B., Man{\'{e}}, D., Vasudevan, V., Le, Q.V.: Autoaugment:
  Learning augmentation policies from data. CoRR  \textbf{abs/1805.09501}
  (2018), \url{http://arxiv.org/abs/1805.09501}

\bibitem{imagenet}
Deng, J., Dong, W., Socher, R., Li, L.J., Li, K., Fei-Fei, L.: {ImageNet: A
  Large-Scale Hierarchical Image Database}. In: CVPR09 (2009)

\bibitem{fractional-maxpool}
Graham, B.: Fractional max-pooling. CoRR  \textbf{abs/1412.6071} (2014),
  \url{http://arxiv.org/abs/1412.6071}

\bibitem{pyramidnet}
Han, D., Kim, J., Kim, J.: Deep pyramidal residual networks. CoRR
  \textbf{abs/1610.02915} (2016), \url{http://arxiv.org/abs/1610.02915}

\bibitem{resnet}
He, K., Zhang, X., Ren, S., Sun, J.: Deep residual learning for image
  recognition. CoRR  \textbf{abs/1512.03385} (2015),
  \url{http://arxiv.org/abs/1512.03385}

\bibitem{bag-of-tricks}
He, T., Zhang, Z., Zhang, H., Zhang, Z., Xie, J., Li, M.: Bag of tricks for
  image classification with convolutional neural networks. CoRR
  \textbf{abs/1812.01187} (2018), \url{http://arxiv.org/abs/1812.01187}

\bibitem{hoffer-code}
Hoffer, E.: Convolutional networks using pytorch.
  \url{https://github.com/eladhoffer/convNet.pytorch} (2019)

\bibitem{hoffer-fix}
Hoffer, E., Hubara, I., Soudry, D.: Fix your classifier: the marginal value of
  training the last weight layer. CoRR  \textbf{abs/1801.04540} (2018),
  \url{http://arxiv.org/abs/1801.04540}

\bibitem{mobilenetv1}
Howard, A.G., Zhu, M., Chen, B., Kalenichenko, D., Wang, W., Weyand, T.,
  Andreetto, M., Adam, H.: Mobilenets: Efficient convolutional neural networks
  for mobile vision applications. CoRR  \textbf{abs/1704.04861} (2017),
  \url{http://arxiv.org/abs/1704.04861}

\bibitem{imagenette}
Howard, J.: Imagenette. \url{https://github.com/fastai/imagenette} (2019)

\bibitem{xresnet}
Howard, J., {FastAI}: {XResNet}.
  \url{https://github.com/fastai/fastai/blob/master/fastai/vision/models/xresnet.py}
  (2019)

\bibitem{squeeze-and-excitation}
Hu, J., Shen, L., Sun, G.: Squeeze-and-excitation networks. CoRR
  \textbf{abs/1709.01507} (2017), \url{http://arxiv.org/abs/1709.01507}

\bibitem{densenet}
Huang, G., Liu, Z., Weinberger, K.Q.: Densely connected convolutional networks.
  CoRR  \textbf{abs/1608.06993} (2016), \url{http://arxiv.org/abs/1608.06993}

\bibitem{kolesnikov2016seed}
Kolesnikov, A., Lampert, C.H.: Seed, expand and constrain: Three principles for
  weakly-supervised image segmentation. In: European Conference on Computer
  Vision. pp. 695--711. Springer (2016)

\bibitem{cifar-report}
Krizhevsky, A.: Learning multiple layers of features from tiny images. Tech.
  rep. (2009)

\bibitem{alexnet}
Krizhevsky, A., Sutskever, I., Hinton, G.E.: Imagenet classification with deep
  convolutional neural networks. In: Advances in neural information processing
  systems. pp. 1097--1105 (2012)

\bibitem{cnn}
{Lecun}, Y., {Bottou}, L., {Bengio}, Y., {Haffner}, P.: Gradient-based learning
  applied to document recognition. Proceedings of the IEEE  \textbf{86}(11),
  2278--2324 (Nov 1998). \doi{10.1109/5.726791}

\bibitem{generalizing-pooling}
{Lee}, C., {Gallagher}, P., {Tu}, Z.: Generalizing pooling functions in {CNNs}:
  Mixed, gated, and tree. IEEE Transactions on Pattern Analysis and Machine
  Intelligence  \textbf{40}(4),  863--875 (April 2018).
  \doi{10.1109/TPAMI.2017.2703082}

\bibitem{RAdam}
Liu, L., Jiang, H., He, P., Chen, W., Liu, X., Gao, J., Han, J.: On the
  variance of the adaptive learning rate and beyond (2019)

\bibitem{mish}
Misra, D.: Mish: A self regularized non-monotonic neural activation function
  (2019)

\bibitem{howtotrainresnet}
Page, D.: How to train your {ResNet}.
  \url{https://myrtle.ai/how-to-train-your-resnet-5-hyperparameters/} (2018)

\bibitem{pytorch-paper}
Paszke, A., Gross, S., Chintala, S., Chanan, G., Yang, E., DeVito, Z., Lin, Z.,
  Desmaison, A., Antiga, L., Lerer, A.: Automatic differentiation in {PyTorch}.
  In: NIPS Autodiff Workshop (2017)

\bibitem{Pinheiro2015}
Pinheiro, P.O., Collobert, R.: From image-level to pixel-level labeling with
  convolutional networks. In: The IEEE Conference on Computer Vision and
  Pattern Recognition (CVPR) (June 2015)

\bibitem{radenovic2018fine}
Radenovi{\'c}, F., Tolias, G., Chum, O.: Fine-tuning cnn image retrieval with
  no human annotation. IEEE transactions on pattern analysis and machine
  intelligence  \textbf{41}(7),  1655--1668 (2018)

\bibitem{ramon2000multi}
Ramon, J., De~Raedt, L.: Multi instance neural networks. In: Proceedings of the
  ICML-2000 workshop on attribute-value and relational learning. pp. 53--60
  (2000)

\bibitem{saeedan2018detail}
Saeedan, F., Weber, N., Goesele, M., Roth, S.: Detail-preserving pooling in
  deep networks. In: Proceedings of the IEEE Conference on Computer Vision and
  Pattern Recognition. pp. 9108--9116 (2018)

\bibitem{mobilenetv2}
Sandler, M., Howard, A.G., Zhu, M., Zhmoginov, A., Chen, L.: Inverted residuals
  and linear bottlenecks: Mobile networks for classification, detection and
  segmentation. CoRR  \textbf{abs/1801.04381} (2018),
  \url{http://arxiv.org/abs/1801.04381}

\bibitem{vgg}
Simonyan, K., Zisserman, A.: Very deep convolutional networks for large-scale
  image recognition. arXiv preprint arXiv:1409.1556  (2014)

\bibitem{inception}
Szegedy, C., Liu, W., Jia, Y., Sermanet, P., Reed, S.E., Anguelov, D., Erhan,
  D., Vanhoucke, V., Rabinovich, A.: Going deeper with convolutions. CoRR
  \textbf{abs/1409.4842} (2014), \url{http://arxiv.org/abs/1409.4842}

\bibitem{Wang2017}
Wang, X., Peng, Y., Lu, L., Lu, Z., Bagheri, M., Summers, R.M.: Chestx-ray8:
  Hospital-scale chest x-ray database and benchmarks on weakly-supervised
  classification and localization of common thorax diseases. In: Proceedings of
  the IEEE conference on computer vision and pattern recognition. pp.
  2097--2106 (2017)

\bibitem{Wang2018}
Wang, X., Yan, Y., Tang, P., Bai, X., Liu, W.: Revisiting multiple instance
  neural networks. Pattern Recognition  \textbf{74},  15 -- 24 (2018).
  \doi{https://doi.org/10.1016/j.patcog.2017.08.026},
  \url{http://www.sciencedirect.com/science/article/pii/S0031320317303382}

\bibitem{rangermish}
Wright, L., Doria, S., Grankin, M., Lois, F., Oguiza, I.:
  Ranger-mish-imagewoof-5.
  \url{https://github.com/lessw2020/Ranger-Mish-ImageWoof-5} (2019)

\bibitem{resnext}
Xie, S., Girshick, R.B., Doll{\'{a}}r, P., Tu, Z., He, K.: Aggregated residual
  transformations for deep neural networks. CoRR  \textbf{abs/1611.05431}
  (2016), \url{http://arxiv.org/abs/1611.05431}

\bibitem{shakedrop}
Yamada, Y., Iwamura, M., Kise, K.: Shakedrop regularization. CoRR
  \textbf{abs/1802.02375} (2018), \url{http://arxiv.org/abs/1802.02375}

\bibitem{wrn}
Zagoruyko, S., Komodakis, N.: Wide residual networks. CoRR
  \textbf{abs/1605.07146} (2016), \url{http://arxiv.org/abs/1605.07146}

\bibitem{selfattention}
Zhang, H., Goodfellow, I., Metaxas, D., Odena, A.: Self-attention generative
  adversarial networks (2018)

\bibitem{lookahead-opt}
Zhang, M.R., Lucas, J., Hinton, G.E., Ba, J.: Lookahead optimizer: k steps
  forward, 1 step back. CoRR  \textbf{abs/1907.08610} (2019),
  \url{http://arxiv.org/abs/1907.08610}

\end{thebibliography}
